\documentclass[]{fairmeta}

\usepackage{wrapfig}
\usepackage{tabularx}
\usepackage{textcomp}
\usepackage{stfloats}
\usepackage{url}
\usepackage{verbatim}
\usepackage{titlesec}
\usepackage{tocloft}
\usepackage{adjustbox}
\usepackage{multirow}
\usepackage{pifont}
\usepackage{tikz}
\usepackage{comment}
\usepackage{amsmath,amssymb} 
\usepackage{colortbl}  
\usepackage{color}
\usepackage{booktabs} 
\usepackage{hyperref}
\usepackage{graphicx}    
\usepackage{subcaption} 
\usepackage{multirow} 
\usepackage{booktabs} 
\usepackage{subcaption} 
\RequirePackage{xspace}
\makeatletter
\DeclareRobustCommand\onedot{\futurelet\@let@token\@onedot}
\def\@onedot{\ifx\@let@token.\else.\null\fi\xspace}

\makeatother

\definecolor{adptorange}{RGB}{248, 205, 172}
\definecolor{cmpblue}{RGB}{189, 215, 238}
\definecolor{cmpblue}{RGB}{189, 215, 238}

\definecolor{our_red}{RGB}{232,157,160}
\definecolor{our_blue}{RGB}{136,206,230}
\definecolor{our_orange}{RGB}{246,200,168}
\definecolor{our_green}{RGB}{178,211,164}

\definecolor{attn_code0}{RGB}{247,215,200}
\definecolor{attn_code1}{RGB}{238,169,139}
\definecolor{mlp_code0}{RGB}{204,201,221}
\definecolor{mlp_code1}{RGB}{102,95,153}

\definecolor{token_blue}{RGB}{84, 120, 140}

\usepackage{pifont}       
\usepackage{bbding}       
\usepackage{fontawesome}
\usepackage{xspace}

\usepackage{float}

\newlength\savewidth

\newcolumntype{x}[1]{>{\centering\arraybackslash}p{#1pt}}
\newcolumntype{y}[1]{>{\raggedright\arraybackslash}p{#1pt}}
\newcolumntype{z}[1]{>{\raggedleft\arraybackslash}p{#1pt}}

\renewcommand{\paragraph}[1]{\vspace{1mm}\noindent\textbf{#1}}
\usepackage{colortbl}
\usepackage{xcolor}
\usepackage{wrapfig}

\renewcommand{\paragraph}[1]{\vspace{1.25mm}\noindent\textbf{#1}}

\usepackage{algorithm}
\usepackage{listings}

\definecolor{codeblue}{rgb}{0.25, 0.5, 0.5}
\definecolor{codekw}{rgb}{0.35, 0.35, 0.75}
\lstdefinestyle{Pytorch}{
    language = Python,
    backgroundcolor = \color{white},
    basicstyle = \fontsize{9pt}{8pt}\selectfont\ttfamily\bfseries,
    columns = fullflexible,
    aboveskip=1pt,
    belowskip=1pt,
    breaklines = true,
    captionpos = b,
    commentstyle = \color{codeblue},
    keywordstyle = \color{codekw},
}

\definecolor{green}{HTML}{009000}
\definecolor{red}{HTML}{ea4335}

\usepackage{longtable}
\usepackage{array}
\usepackage{amsthm}
\usepackage{capt-of}
\usepackage{algorithmic}
\newcommand{\suppref}[1]{\cref{#1}}
\newcommand{\Suppref}[1]{\Cref{#1}}
\newcommand{\appref}[1]{Appendix~\ref{#1}}
\newcommand{\ours}{CASE}
\newcommand{\vtheta}{\bm{\theta}}

\newcommand{\ind}{\bm{1}}
\DeclareMathOperator{\KL}{D_{\mathrm{KL}}}
\DeclareMathOperator{\sg}{sg}
\DeclareMathOperator{\TopK}{TopK}
\DeclareMathOperator{\clip}{clip}

\newcommand{\artifact}[1]{\path{#1}}

\newcommand{\ablateon}{\ensuremath{\bullet}}
\newcommand{\ablateoff}{\ensuremath{\circ}}
\newcolumntype{Y}{>{\raggedright\arraybackslash}X}
\newtheorem{proposition}{Proposition}
\theoremstyle{remark}
\newtheorem{remark}{Remark}
\newcommand{\contributorslist}{}
\newcommand{\contributors}[1]{\gdef\contributorslist{{\small #1}}}
\newcommand{\coreauthorlabel}{}
\title{From Given to Gathered Evidence:\\Agentic Learning for Longitudinal Medical Reasoning}

\author[1, 2]{Core Authors:\: Minye Shao}
\author[* 2, 3]{Chaohui Yu}
\author[2]{Yixuan Wu}
\author[2]{Fan Wang}
\author[4]{Ling Shao}
\author[* 1]{Yang Long}

\contributors{Contributors: Jing Wang\textsuperscript{2,3}, Hangjie Yuan\textsuperscript{2,3}, Shang Liu\textsuperscript{2,3}}

\affiliation[1]{Durham University\\}
\affiliation[2]{DAMO Academy, Alibaba Group\\}
\affiliation[3]{Hupan Laboratory\\}
\affiliation[4]{University of the Chinese Academy of Sciences}

\contribution[*]{Corresponding authors}

\abstract{
Foundation models can serve as clinical agents through tool-use harnesses. However, conventional medical benchmarks assess reasoning over preselected evidence rather than the ability to seek it across clinical records and longitudinal imaging.
We propose \textbf{CASE}: a series of role-specific \textbf{C}linical \textbf{A}gents for \textbf{S}eeking \textbf{E}vidence, together with a tool-use harness and an agentic post-training framework for compact vision--language policy models.
We further introduce a longitudinal multimodal benchmark built on UK Biobank, comprising 50,401 clinical questions derived from real-world ICD-10-coded diagnoses of 4,739 participants. Each question links to a patient-specific environment containing clinical context and multi-sequence MRI from baseline and follow-up visits, where agents autonomously select which visits, organs, modalities, slices, and specialist tools to inspect and compare.
Supervised fine-tuning transfers evidence-seeking workflows from 147,434 frontier-model interaction trajectories, followed by agentic reinforcement learning on the learner's own environment interactions. Privileged on-policy self-distillation and rubric-based LLM feedback refine evidence-to-conclusion reasoning without prescribing tool sequences.
Experiments show that CASE moves beyond \textbf{question--answer imitation} toward \textbf{transferable investigation policies}, strengthening evidence-grounded longitudinal reasoning. Under matched evaluation conditions, our Qwen3-VL-8B-Instruct based agent achieves over 16\% and 10\% relative improvements in answer accuracy over GPT-5.4 and Claude Opus 4.8.
}

\metadata[Code]{\url{https://github.com/VinyehShaw/CASE}}
\metadata[Hugging Face]{\url{https://huggingface.co/CASE-Bench}}
\metadata[Project Page]{\url{https://vinyehshaw.github.io/case-project/}}

\begin{document}
\thispagestyle{firstheader}
\maketitle
\pagestyle{empty}

\begin{figure}[t]
\centering
\includegraphics[width=\linewidth]{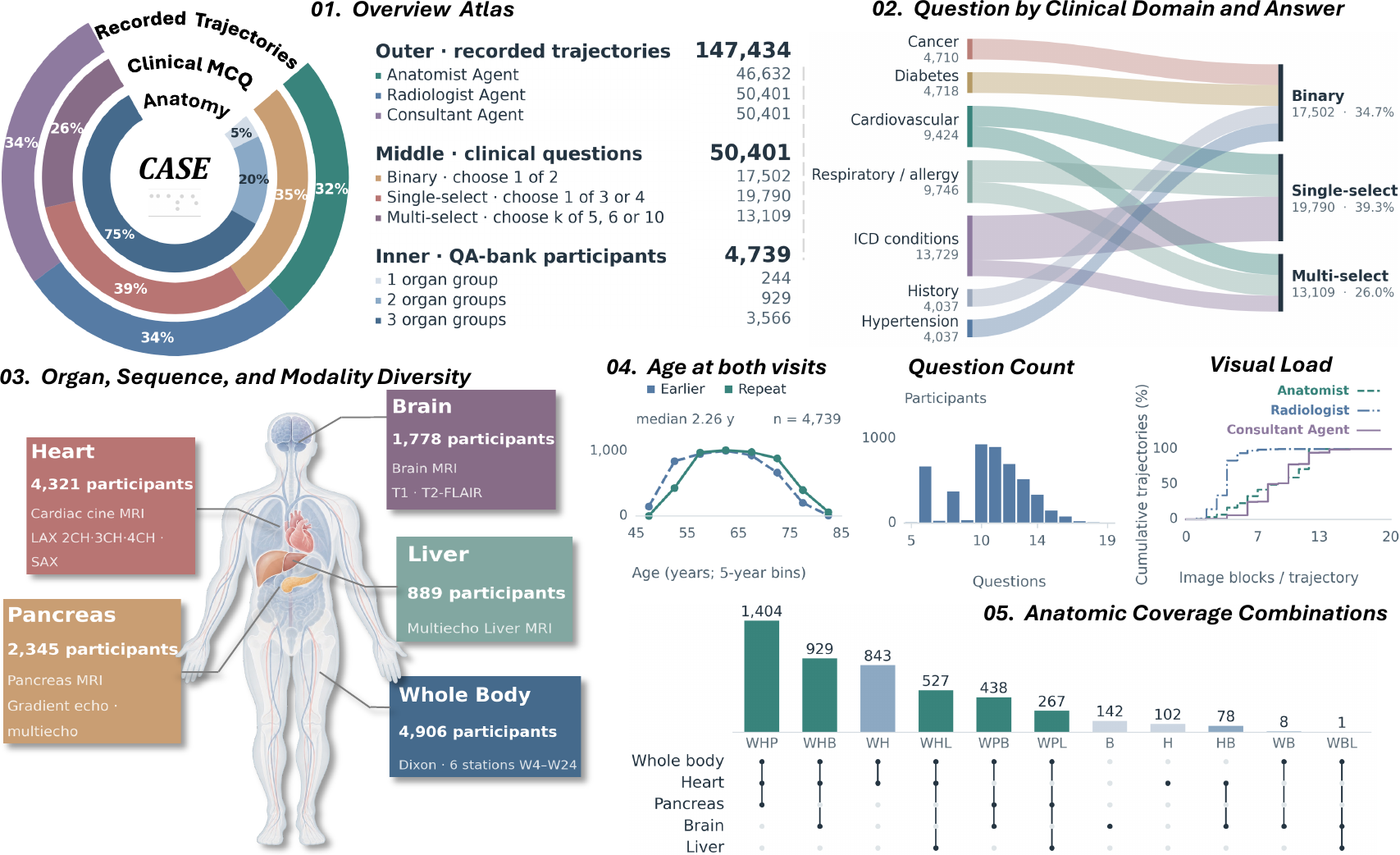}
\caption{\textbf{The CASE longitudinal evidence-seeking benchmark.} Fixed clinical targets are paired with heterogeneous patient environments. (1) Cohort, task, and raw-recording composition. (2) Clinical domains and answer formats. (3) Five organ-specific MRI protocols. (4) Paired-visit age, task density, and agent visual load. (5) Multi-organ coverage. Raw recordings precede training selection; W/H/P/B/L denote whole body, heart, pancreas, brain, and liver. Zoom in for details.}
\label{fig:benchmark_overview}
\end{figure}

\section{Introduction}
\label{sec:intro}

Foundation models paired with tool-use harnesses show promise as clinical agents~\citep{wang2026medagentpro,zhang2026radagents}. Yet longitudinal medical reasoning extends beyond medical knowledge: evidence is distributed across clinical records, imaging sequences, measurements, and follow-up visits, requiring agents to decide what to inspect before reaching a conclusion. Clinical reasoning therefore involves not only interpreting evidence, but also determining how to acquire it.

This distinction matters for evaluation. Fixed-context medical QA assesses reasoning over preselected evidence, while interactive benchmarks introduce information gathering and sequential decision-making~\citep{liu2025medchain,chiu2025vivabench}. Longitudinal multimodal reasoning further requires comparing alternative paths through the same patient environment, motivating standardized tasks and evidence access while leaving acquisition to the policy. A central question is whether compact models can bootstrap clinical investigation from frontier-model trajectories and move beyond imitation through autonomous interaction. Demonstrations cover only teacher-chosen paths, whereas autonomous interaction exposes the learner to different combinations of evidence and requires it to reason over what it actually gathers.

We investigate this through CASE, where a compact vision--language policy acts as role-specific clinical agents within a shared tool-use harness. On our longitudinal multimodal benchmark built on UK Biobank~\citep{sudlow2015ukbiobank,littlejohns2020ukbimaging} (\cref{fig:benchmark_overview}), supervised fine-tuning (SFT) initializes the investigation policy from recorded frontier-model trajectories, after which agentic reinforcement learning (RL) optimizes the policy on its own environment interactions. Privileged on-policy self-distillation, with its self-teacher continually refreshed as the policy evolves, together with rubric-based large language model (LLM) feedback, progressively refines evidence-to-conclusion reasoning without prescribing tool sequences, allowing the agent to learn both what evidence to seek and how to reason from what it finds. Together, these designs support the following contributions:
\begin{enumerate}

\item \textbf{Rethinking longitudinal medical reasoning as evidence seeking.}
By recasting longitudinal reasoning as an interactive evidence-seeking problem, we move beyond fixed-context evaluation and require agents to determine what patient evidence to inspect before reaching a conclusion.

\item \textbf{Revisiting clinical agents as evidence-seeking investigators.}
CASE operationalizes this view through role-specific \textbf{C}linical \textbf{A}gents for \textbf{S}eeking \textbf{E}vidence, a patient-level tool-use harness, and a \emph{longitudinal multimodal benchmark}, requiring agents to navigate heterogeneous evidence and invoke specialist tools when needed.

\item \textbf{Exploring learning beyond teacher imitation.}
Starting from action-level teacher trajectories, we optimize the compact policy on its own \emph{environment interactions without direct answer supervision}. Rubric-based LLM feedback and privileged self-distillation further strengthen evidence-to-conclusion synthesis, and the resulting agent surpasses frontier-model baselines under matched evaluation conditions; controlled ablations isolate the contribution of each component.

\end{enumerate}

\section{Related Work}
\label{sec:related}

\paragraph{Medical benchmarks and interactive evaluation.}
Static question--answer (QA) benchmarks, however complex, give diminishing signal about real task completion, especially in data-scarce clinical agent settings that demand a realistic environment. Medical QA supplies text or images~\citep{jin2020medqa,zuo2025medxpertqa}; agentic benchmarks add record operations, simulated encounters, and active information gathering~\citep{jiang2025medagentbench,schmidgall2024agentclinic,liu2025medchain,chiu2025vivabench} or widen specialty, longitudinal, and diagnostic-stage coverage~\citep{yan2025clinicallab,vasilev2025mtbbench,lu2026clinenv}. Ours grounds each question in recorded diagnoses and makes the agent gather the determining cross-visit magnetic resonance imaging (MRI), measurements, and reports (\appref{app:benchmark_landscape}, \suppref{tab:benchmark_paradigms}).

\paragraph{Medical agents and policy learning.}
Medical agents adapt collaboration to task complexity~\citep{kim2024mdagents}, coordinate tools and verification~\citep{fallahpour2025medrax,wang2026medagentpro,zhang2026radagents}, or learn from memory and feedback~\citep{almansoori2025medagentsim,liu2025medchain}, though MedAgentBoard finds such gains task-dependent~\citep{zhu2026medagentboard}; parameter learning spans supervised tool invocation~\citep{li2024mmedagent} to RL~\citep{jiang2026medvr,jiang2026ophiuchus,fan2026macro}. CASE instead trains one shared compact policy that learns acquisition and synthesis jointly rather than orchestrating fixed specialists (\appref{app:medical_agents}, \suppref{tab:medical_workflows,tab:medical_policy_learning}).

\paragraph{On-policy and privileged distillation.}
Prior distillation transfers teacher or privileged signals via guidance on student-generated sequences~\citep{agarwal2024gkd,zhang2026rlad,liu2026tgpo,lin2026opdvr}, privileged self-teachers~\citep{zhao2026opsd,hubotter2026sdpo}, selective reasoning supervision under Group Relative Policy Optimization (GRPO)~\citep{shao2024deepseekmath,wang2026trace,tan2026ssopd}, or turn-, token-, and step-level feedback~\citep{tian2026pbsd,zhang2026adrs,yang2026ocsd,wu2026sspo}. CASE instead turns the divergence between ordinary and answer-privileged views of the same frozen policy into a \emph{detached trajectory-level reward penalty} rather than a differentiable distillation loss; entering the return before group normalization, it reorders whole trajectories by synthesis quality, whereas OCSD, SSPO, PBSD, and ADRS adjust token-, step-, or turn-level weights without reordering whole trajectories, so synthesis feedback shapes acquisition without tool-specific credit assignment or an auxiliary loss (\appref{app:distillation_landscape}, \suppref{tab:distillation_landscape}).

\section{Preliminaries}
\label{sec:prelim}

\paragraph{Tool-interacting policies.}
Following the standard agentic-RL formulation~\citep{yang2026ocsd}, a task comprises a prompt $x$ (role, clinical question, options, initial patient context), a participant-specific environment $\mathcal E_p$, and a reference answer $y^\star$ withheld during interaction. At round $k$ the policy $\pi_{\vtheta}$ emits $a_k\sim\pi_{\vtheta}(\cdot\mid H_k)$ from history $H_k$ ($H_1=x$); a tool request yields $o_k=\mathcal T_p(a_k)=(e_k,X_k)$, where $\mathcal T_p$ returns text $e_k$ and images $X_k$; and $H_{k+1}=H_k\oplus(a_k,o_k)$ appends them in order. An answer-terminated episode of $M$ assistant turns is $\tau=(x,a_1,o_1,\ldots,a_{M-1},o_{M-1},a_M)$ with parsed answer $\hat y=\hat y(\tau)$.

\paragraph{Group-relative optimization.}
For each $x$, Group Relative Policy Optimization (GRPO)~\citep{shao2024deepseekmath} samples $G$ trajectories with scalar returns $R_j$; writing $\bar R=G^{-1}\sum_jR_j$ and $\operatorname{sd}(R_{1:G})$ for the within-prompt dispersion, the group-relative advantage is
\begin{equation}
 A_j=\sg\!\left[\frac{R_j-\bar R}{\operatorname{sd}(R_{1:G})+\epsilon}\right],
\label{eq:advantage}
\end{equation}
where $\sg$ denotes stop-gradient and $\epsilon>0$ stabilizes normalization. The same $A_j$ weights all assistant-generated tokens, while tool observations condition subsequent actions without carrying gradient targets and padding is masked; our return therefore assumes no tool-specific credit assignment.

\section{A Longitudinal Evidence-Seeking Benchmark}
\label{sec:benchmark}

CASE makes \textbf{evidence acquisition itself the task}. Each patient is asked several questions, each a multiple-choice question (MCQ), together spanning several clinical domains; given only the narrative and one such question, the agent interleaves tool calls with reasoning to decide which longitudinal, multi-sequence studies to inspect, with each returned observation conditioning the next decision and the final answer. Accuracy thus reflects \textbf{what the agent chooses to gather} from a scarce, heterogeneous imaging environment rather than evidence fixed in advance. \Cref{fig:benchmark_overview} overviews the imaging, the questions derived from International Classification of Diseases, 10th revision (ICD-10) codes, and the recorded trajectories.

\paragraph{Participants and question construction.}
Built on UK Biobank~\citep{sudlow2015ukbiobank}, CASE contains 50,401 questions from 4,739 participants, including 4,850 held-out questions from 455. Questions derive from doctor-diagnosed conditions and ICD-10 records, pair each recorded target with cohort-derived distractors that exclude the participant's own diagnoses, and span seven clinical families and ten task--format strata over binary, single-select, and multi-select items (panel~2 of \cref{fig:benchmark_overview}); \appref{app:data} details distractor sampling, task expansion, and label correction.

\paragraph{Patient evidence.}
Each participant's environment pairs \textbf{baseline and follow-up MRI of the same subject} from the UK Biobank imaging visits~\citep{littlejohns2020ukbimaging,ukbiobankimagingdata}, a median of 2.26 years apart, so within-subject interval comparison is available for every question; panel~4 of \cref{fig:benchmark_overview} adds the resulting workload of just over ten questions per participant and roughly eight images per resolved trajectory. The imaging is scarce and heterogeneous, covering five organ regions (heart, brain, pancreas, liver, and whole body) across seven sub-modalities, with most participants contributing two or three organ groups (panels~3 and~5). Tools expose MRI views, Vista3D-derived measurements and overlays~\citep{he2025vista3d}, and MedGemma-27B reports~\citep{sellergren2025medgemma}; the agent chooses which tool, organ, sequence, visit, and frame to retrieve and whether to compare across visits, so the evidence path is policy-controlled, while diagnosis records supply the answer labels (\appref{app:environment}).

\paragraph{Recording.}
\label{sec:recording}
We use Qwen3.8-Max~\citep{qwen2026qwen38max} to record role-conditioned trajectories across the three agents in LangGraph environments~\citep{langchain2024langgraph}, withholding reference answers until first-pass completion; each trajectory interleaves tool calls with roughly eight to ten paired-visit images, yielding 147,434 recordings, of which the screening rule of \cref{eq:rsft} retains approximately 140,000 demonstrations for initialization from training participants (\appref{app:training}).

\paragraph{Evaluation metrics and protocol.}
Participants are disjoint across splits. The primary metric is first-pass exact answer-set accuracy: $a_i\in\{0,1\}$ marks a valid first pass whose canonicalized answer set equals the reference exactly (all and only the correct options, single- and multi-select alike), with missing, failed, or unparseable episodes scoring $0$. Over $N$ questions in ten strata $\{\mathcal D_t\}_{t=1}^{10}$ with $N_t=|\mathcal D_t|$ and $A_t=N_t^{-1}\sum_{i\in\mathcal D_t}a_i$, we report
\begin{equation}
 A_{\mathrm{micro}}=\frac1N\sum_{i=1}^{N}a_i=\sum_{t=1}^{10}\frac{N_t}{N}\,A_t,
 \qquad
 A_{\mathrm{macro}}=\frac1{10}\sum_{t=1}^{10}A_t.
\label{eq:metrics}
\end{equation}
\Cref{tab:main_accuracy} reports per-stratum accuracy alongside the two aggregates: Micro weights questions equally, whereas Macro weights the ten strata equally and is thus sensitive to small strata. Each role is evaluated separately under matched evidence access, tool budgets, and a skill memory frozen across runs but live within a run, without test-answer feedback. A version-pinned judge held out from the training reward supplies a fixed-judge insight score that complements accuracy and is cross-checked across judges for rank stability (\appref{app:evaluation}).

\section{Learning to Acquire and Synthesize Evidence}
\label{sec:method}

\begin{figure}[t]
\centering
\includegraphics[width=\linewidth]{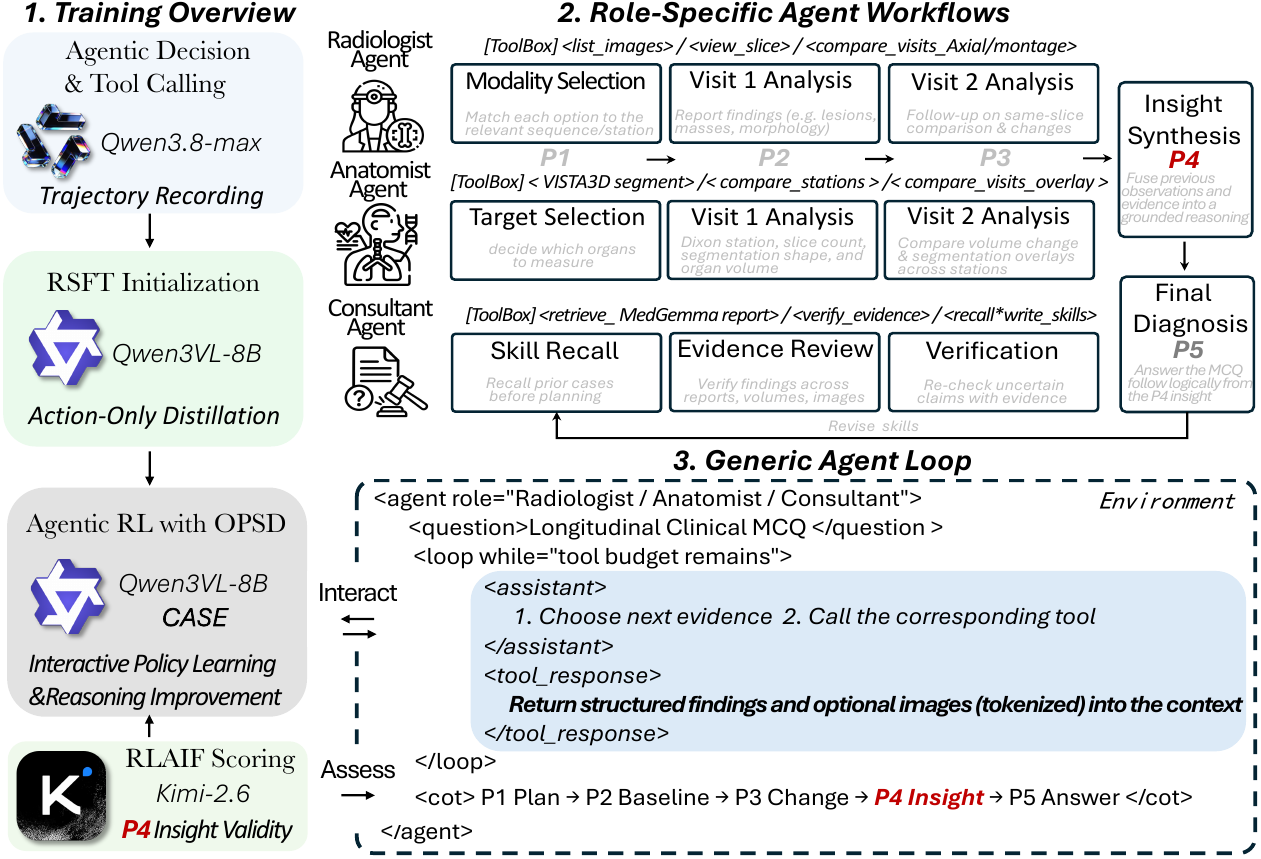}
\caption{\textbf{CASE training and execution.} (1) Two-stage training: RSFT distills screened teacher trajectories into Qwen3-VL-8B-Instruct, then agentic RL refines the policy on its own trajectories via outcome reward, OPSD, and RLAIF from an external non-ancestor judge. (2) One shared policy instantiates three evidence-seeking workflows through role prompts and tool registries. (3) A common contract alternates evidence decisions and tool observations until the P1--P5 synthesis. Privileged information scores completed training trajectories only and never determines the tool sequence.}

\label{fig:overview}
\end{figure}

\paragraph{Method overview.}
Panel~1 of \cref{fig:overview} outlines our two-stage pipeline. We first record role-conditioned teacher trajectories and distill a quality-filtered subset into the Qwen3-VL-8B-Instruct base~\citep{bai2025qwen3vl} by rejection-sampling fine-tuning (RSFT), which rapidly initializes \emph{how to act as an agent} (calling tools, alternating evidence decisions with observations, and following the synthesis contract). Agentic RL then optimizes the policy on its \emph{own} trajectories, combining outcome reward, privileged on-policy self-distillation (OPSD) of the synthesis step, and rubric feedback from an external non-ancestor judge (Kimi K2.6~\citep{moonshot2026kimik26}) that scores the logical core of each chain of thought, so the policy is rewarded for an answer justified by sound Phase-4 reasoning rather than correctness alone; both stages train the full vision--language model end to end, including the visual pathway. We initialize from the Instruct checkpoint rather than a thinking variant to avoid starting from a model already specialized for explicit chain-of-thought output; this does not imply that the Instruct checkpoint received no vendor-side reinforcement learning~\citep{bai2025qwen3vl}. Efficiency and kernel-stability considerations appear in \appref{app:training}. RSFT only imitates demonstrations and thus teaches the \emph{form} of acting, whereas \textbf{the genuine gains in reasoning and accuracy come from agentic RL}, in which the base model explores the harness on its own and can surpass its teacher. Because the reference answer is fixed by a participant's evidence and train/test participants are disjoint, \textbf{the answer is never a direct function of the question}, so the policy must seek the determining evidence rather than memorize question--answer pairs.

\subsection{CASE role-conditioned workflows}
\label{sec:interaction}

CASE instantiates three roles through one shared policy (\cref{fig:overview}, panel~2). The Radiologist selects acquisitions, inspects baseline and follow-up images, and compares matched views. The Anatomist selects anatomical targets and uses Vista3D-derived measurements and overlays to assess longitudinal change. The Consultant combines these sources with MedGemma reports and skill retrieval (a Consultant-only reusable-guidance memory, detailed in \appref{app:environment}), revisiting evidence to assess conflicting claims. Role prompts organize these workflows; the agents are not independently trained models or a communicating committee. Per-role tool registries and call budgets appear in \cref{tab:appendix_registry}; organ-name resolution, the rollout loop, and the evaluation harness are detailed in \appref{app:environment}; and a Consultant evaluation trajectory covering all callable tools is given in \appref{app:example_trajectory}.

\paragraph{Generic agent loop.}
The Anatomist, Radiologist, and Consultant run one common interaction loop (panel~3 of \cref{fig:overview}), differing only in role prompt and tool registry, where each assistant turn requests evidence or ends the investigation and returned text and images enter the context before the next decision. Recording, inference, and RL share this contract and a common five-phase terminal response (plan, multi-source review, cross-visit comparison, a \emph{key insight} distilled from the gathered evidence and preceding reasoning, and an answer derived \emph{strictly} from it); these phases structure the final reasoning, not a prescribed tool order. LangGraph supports recording and evaluation and an extended \textsc{verl} AgentLoop~\citep{sheng2025hybridflow} supports RL. The corpus pools all three roles, $\mathcal D_{\mathrm{rec}}=\bigcup_r\mathcal D_r$ for $r\in\{\mathrm{anat},\mathrm{rad},\mathrm{cons}\}$, so SFT mixes their retained trajectories and RL their task prompts while one shared policy is updated and every example retains its role instruction.

\subsection{Evidence-linkage rejection sampling for policy initialization}

To provide an efficient cold start for tool use and multi-step agentic reasoning, we apply rejection-sampling fine-tuning (RSFT). Executable demonstrations need not explain how evidence supports an answer, so we screen recorded trajectories that contain at least one tool call and a parseable Phase-4 insight. Kimi K2.6~\citep{moonshot2026kimik26} scores the question, options, reference and predicted answers, and extracted insight $I_d$ on specific evidence, logical progression, answer support, and completeness, giving $s_d^{\mathrm{off}}\in[0,1]$; writing $n_{\mathrm{tool}}(d)$ for the number of assistant tool calls, we retain
\begin{equation}
 \mathcal D_{\mathrm{RS}}=\left\{d\in\mathcal D_{\mathrm{rec}}:
 n_{\mathrm{tool}}(d)\geq1,\ s_d^{\mathrm{off}}\geq0.4\right\}.
\label{eq:rsft}
\end{equation}
This pass imposes no exact-answer gate and does not regenerate rejected trajectories; the text-only judge scores the stated argument, not the visual validity of its findings. Rubric, parsing, provenance, data-materialization, and hyperparameter details are in \appref{app:training} (\suppref{tab:appendix_training}).

Each retained trajectory is serialized as one multimodal sequence and trained with next-token cross-entropy masked at all non-assistant positions, so patient information and tool-returned text and images condition the model while only assistant reasoning, tool calls, and the final synthesis (including the answer) are prediction targets. With $w_{d,t}$ the token at position $t$, $g_{d,t}$ its preceding multimodal context, and $m_{d,t}=1$ only at assistant positions, mixed-role SFT minimizes
\begin{equation}
 \mathcal L_{\mathrm{SFT}}(\vtheta)=
 -\frac{\sum_{d\in\mathcal D_{\mathrm{RS}}}\sum_t m_{d,t}
 \log\pi_{\vtheta}(w_{d,t}\mid g_{d,t})}
 {\sum_{d\in\mathcal D_{\mathrm{RS}}}\sum_t m_{d,t}}.
\label{eq:sft}
\end{equation}
This initializes an executable policy on demonstrator-selected histories; RL subsequently trains on policy-selected histories.

\subsection{On-policy interaction with a replayable evidence environment}
\label{sec:onpolicy_environment}

Only two source artifacts are prepared once per participant, and they are the sole outputs that require model inference, namely Vista3D segmentation masks of the whole-body Dixon MRI alone, not the organ-specific sequences, across its six neck-to-ankle stations (4, 8, 12, 16, 20, 24; some participants have fewer)~\citep{he2025vista3d}, and MedGemma-27B report objects~\citep{sellergren2025medgemma}. Agentic RL never re-runs either model; its mock tools replay observations precomputed from these two sources. Every other observation follows from the policy's own interaction-time decisions over tool, organ (resolved to a segmentation label through the Vista3D organ-index map), sequence, sub-modality, visit, frame, and cross-visit pair. Recording and evaluation instead render each requested observation on demand; the RL replay is keyed by an identical (participant, tool, argument) tuple, which removes tool-induced waiting and GPU demand without changing what the agent may request or observe.

At update $n$, policy $\pi_{\bar{\vtheta}_n}$ generates $G=8$ fresh interactions per prompt from clinical text, without reference answers. A tool request $a_k$ retrieves an observation for participant $p$ through $o_k=\mathcal T_p(a_k)=\mathcal C_p(\kappa(a_k))$, where $\kappa$ indexes the tool and its requested arguments (organ, sequence, sub-modality, visit, frame) and $\mathcal C_p$ stores the corresponding observations. Only evidence is replayed; tool selection, argument choice, retrieval order, and stopping remain policy-generated.

Let $u_{j,1:N_j}$ serialize rollout $j$'s assistant messages and tool observations, with $\mathcal A_j$ indexing assistant tokens. The causal prefix $g_{j,t}$ contains the initial prompt $x$, preceding tokens $u_{j,<t}$, and previously retrieved images. Returned pixels pass through the policy's vision encoder and merger, replacing aligned image-placeholder embeddings, and these representations are recomputed during training even though the source evidence stays fixed. Only assistant positions contribute to the policy loss, whose gradients also update the visual modules through subsequent predictions; evidence preparation, mock-tool execution, and interaction budgets are detailed in \appref{app:environment}.

\subsection{On-policy privileged self-distillation for synthesis}
\label{sec:opsd}

In mid-to-late RL the policy stalls, prolonging Phase~4 by restating findings and repeating discourse connectives instead of committing to a conclusion; this repetitive mode coincides with a sudden reward drop and arises because a compact policy learns the surface form of an explanation before it reliably resolves the answer. We therefore apply privileged on-policy self-distillation (OPSD), whose training-only privileged information provides an answer-informed reference for synthesis on the policy's own evidence history while leaving tool decisions autonomous. Because the privileged view conditions on a resolved answer, at a fixed synthesis prefix it shifts probability mass away from further restatement and toward advancing or terminating, so aligning the ordinary view with it discourages repetitive continuations without dictating their content. \appref{app:repetition} gives a formal event bound and \cref{sec:experiments} tests the effect through an 8-gram repetition diagnostic (\suppref{eq:repetition_metric}) and the late-training completion behavior of the ablation variants in \cref{tab:learning_ablation}; top-128 truncation is justified by measured next-token coverage (\cref{fig:bucket128}) and the two-forward scoring alignment is tabulated in \suppref{tab:app-forwards}.

For every rollout $j=1,\ldots,G$, we score $[x;u_j]$ and $[x;z_j^\star;u_j]$ under the same fixed snapshot $\bar{\vtheta}_n=\sg(\vtheta_n)$, where the hint $z_j^\star$ supplies the reference answer, resolvable option text, and an optional recorded insight. Both inputs retain all realized tool calls, tool-result text and images, intermediate reasoning, and any P1--P5 synthesis; only the privileged input contains the hint and neither branch generates a replacement continuation, so every rollout is scored under two matched conditions irrespective of correctness or judge rank, batched into microbatches.

The two forwards align logits predicting the same continuation token $u_{j,t}$, accounting for the inserted hint. Let $g^+_{j,t}$ be $g_{j,t}$ with $z_j^\star$ inserted after $x$. Causal attention restricts both views to preceding tokens and observations, even though the complete trajectory is forwarded at once. Over vocabulary $\mathcal V$, the two forwards give full-softmax distributions $q_{n,j,t}(v)=\pi_{\bar{\vtheta}_n}(v\mid g_{j,t})$ (ordinary) and $p_{n,j,t}(v)=\pi_{\bar{\vtheta}_n}(v\mid g^+_{j,t})$ (answer-privileged), whose conceptual full reverse divergence is $d^{\mathrm{full}}_{n,j,t}=\KL(q_{n,j,t}\Vert p_{n,j,t})=\sum_{v\in\mathcal V}q_{n,j,t}(v)\log\big(q_{n,j,t}(v)/p_{n,j,t}(v)\big)$.
The implementation retains $\mathcal S_{n,j,t}=\TopK(q_{n,j,t},128)$ (the 128 highest-probability student tokens), gathers teacher probabilities at those indices, and aggregates the remaining mass into one residual event,
\begin{equation}
 q_{n,j,t}^{\perp}=1-\sum_{v\in\mathcal S_{n,j,t}}q_{n,j,t}(v),\qquad
 p_{n,j,t}^{\perp}=1-\sum_{v\in\mathcal S_{n,j,t}}p_{n,j,t}(v).
\label{eq:partition}
\end{equation}
Let $\mathcal I_j$ index the estimated Phase-4 region and $C_j=\max(1,|\mathcal I_j|)$. The implemented Causal-KL score is
\begin{equation}
 D^{\mathrm{CKL}}_{n,j}=\frac1{C_j}\sum_{t\in\mathcal I_j}\!\left[
 \sum_{v\in\mathcal S_{n,j,t}}q_{n,j,t}(v)
 \log\frac{q_{n,j,t}(v)}{p_{n,j,t}(v)}
 +q_{n,j,t}^{\perp}\log\frac{q_{n,j,t}^{\perp}}{p_{n,j,t}^{\perp}}
 \right].
\label{eq:bucket}
\end{equation}
Both normalizers span the full vocabulary and the top-128 probabilities are not renormalized, so only the divergence is coarsened to 129 events. The scoring region targets synthesis rather than the answer, but its character-based span estimate does not guarantee exact Phase-4 isolation (\appref{app:implementation}). The reverse KL divergence weights alternatives by the student's probabilities without prescribing a teacher rationale. Refreshing $\bar{\vtheta}_n$ after policy updates lets the privileged reference evolve with the learner; the relevant mathematical proofs are provided in the supplementary material.

\subsection{Group-relative rubric feedback and policy optimization}
\label{sec:reward}

During RL, Kimi K2.6 scores each Phase-4 insight against the question, options, reference answer, and Phase-5 prediction on a fixed rubric of evidence specificity, logical progression, answer support, and completeness, yielding $s_j\in[0,1]$; this text-only rubric-based AI feedback (RLAIF) signal follows the direct-feedback setting of \citet{lee2024rlaif} while using our task-specific continuous rubric to evaluate the stated argument, complementing OPSD's distributional comparison, with the four criteria and allocations listed in \suppref{tab:appendix_rubric}.

\emph{Our goal is not merely a correct answer but one reached through a sound Phase-4 key insight.} Outcome and rubric rewards alone cannot separate a guessed correct answer from a fluent insight that entails the wrong one, so the linkage reward $L_j$ closes this gap by rewarding a highly rated insight only when the answer is also correct and penalizing it otherwise, suppressing fluent-but-unfounded synthesis that would game the text-only judge. Within each question's $G=8$ rollouts the two highest-scoring eligible insights are flagged $b_j=1$ (safeguards in \appref{app:judge}); the rest keep $b_j=0$ and earn no linkage while retaining $s_j$ (set to $0$ only if unscorable), so selection adapts per group while the rubric stays fixed.

Two answer terms enter the return, both derived from a single exact-match test. The graded reward $r_j^{\mathrm{ans}}$, added directly to the base return $B_j$, takes $+1$ when the predicted answer set equals the reference exactly, $-1$ on a wrong answer, and $-2$ when no answer is parsed (\appref{app:implementation}). The binary flag $c_j=\ind[\hat y_j=y^\star]\in\{0,1\}$ is that test's pass/fail result, equal to $1$ exactly when $r_j^{\mathrm{ans}}$ is maximal; it is not a separate score and enters only the linkage $L_j$, where it sets the sign. Exactness follows \cref{sec:benchmark}, so after canonicalizing order and duplicates any out-of-set letter, extra selection, or missing parse gives $c_j=0$ (\appref{app:evaluation}). The format term $r_j^{\mathrm{fmt}}$ is maximal only when all five response phases are present (plus the Consultant's trailing skill-writing step), and the coefficients scale all terms to comparable magnitudes. The composite reward is
\begin{equation}
\thickmuskip=3mu\medmuskip=2mu\relax
 L_j=b_j\bigl[0.5c_j-0.3(1-c_j)\bigr],\;
 B_j=r_j^{\mathrm{ans}}+0.01r_j^{\mathrm{fmt}}+0.6s_j+L_j,\;
 R_{n,j}=\sg\!\left[B_j-2D^{\mathrm{CKL}}_{n,j}\right].
\label{eq:reward}
\end{equation}
Here $\sg$ denotes stop-gradient, so the return is constant during the policy update. Top-2 selection controls only $L_j$ and all rollouts enter GRPO; the trainer subtracts each rollout's Causal-KL penalty before the group-relative advantage weights the assistant-token likelihood gradients, which do not propagate through the two scoring forwards.

Policy optimization uses the standard GRPO clipped surrogate with the trajectory-level advantage $A_j$ (objective, clipping, and the reference-KL term in \appref{app:implementation}); because reasoning, tool calls, and synthesis share $A_j$, synthesis feedback trains acquisition without tool-specific credit assignment.

\section{Experiments}
\label{sec:experiments}

We report quantitative results (clinical accuracy and interaction efficiency), an ablation study that isolates each learning component and the dependence on acquired evidence, and a qualitative walkthrough of one complete agent trajectory.

\begin{table}[t]
\centering
\footnotesize
\setlength{\tabcolsep}{1pt}
\renewcommand{\arraystretch}{1.0}
\resizebox{\linewidth}{!}{%
\begin{tabular}{l*{12}{c}}
\toprule
\shortstack{Domain\\{}} & \multicolumn{1}{c}{\shortstack{Oncology\\{}}} & \multicolumn{1}{c}{\shortstack{Endocrine/\\metabolic}} & \multicolumn{3}{c}{\shortstack{Cardiovascular\\{}}} & \multicolumn{2}{c}{\shortstack{Respiratory/\\allergy}} & \multicolumn{2}{c}{\shortstack{Multisystem\\{}}} & \multicolumn{1}{c}{\shortstack{Medical\\history}} & \multicolumn{2}{c}{\shortstack{Overall (\%) $\uparrow$\\{}}} \\
\cmidrule(lr){1-1}\cmidrule(lr){2-2}\cmidrule(lr){3-3}\cmidrule(lr){4-6}\cmidrule(lr){7-8}\cmidrule(lr){9-10}\cmidrule(lr){11-11}\cmidrule(lr){12-13}
Method / Question & \shortstack{Cancer\\(B)} & \shortstack{Diabetes\\(B)} & \shortstack{Condition\\(S)} & \shortstack{Condition\\set (M)} & \shortstack{Hyper-\\tension\\(B)} & \shortstack{Condition\\(S)} & \shortstack{Condition\\set (M)} & \shortstack{ICD-10\\(S)} & \shortstack{Disease\\systems\\(M)} & \shortstack{History\\(B)} & Micro & Macro \\
\midrule
\multicolumn{13}{@{}l}{\textit{Base model}} \\
\cmidrule(lr){1-13}
Qwen3-VL-8B & 67.7 & 46.1 & 72.9 & 70.0 & 38.4 & 52.5 & 49.3 & 26.9 & 5.8 & 41.1 & 46.0 & 47.1 \\
\addlinespace[2pt]
\multicolumn{13}{@{}l}{\textit{Medical models}} \\
\cmidrule(lr){1-13}
Lingshu-32B & 85.0 & 70.4 & 55.0 & 54.6 & 48.8 & 58.9 & 51.6 & 35.2 & 2.9 & 40.5 & 49.9 & 50.3 \\
Lingshu-7B & 65.4 & 83.1 & 48.6 & 40.3 & 65.8 & 50.7 & 38.6 & 28.5 & 3.3 & 42.8 & 45.5 & 46.7 \\
MedGemma-27B & 42.2 & 34.5 & 18.1 & 17.6 & 30.5 & 25.7 & 24.2 & 29.1 & 3.9 & 48.9 & 27.8 & 27.5 \\
MedGemma-1.5-4B & 55.0 & 51.8 & 37.5 & 65.1 & 48.4 & 40.7 & 44.0 & 26.9 & 3.4 & 40.9 & 40.5 & 41.4 \\
\addlinespace[2pt]
\multicolumn{13}{@{}l}{\textit{General multimodal foundation models}} \\
\cmidrule(lr){1-13}
Kimi K2.6 & 48.4 & 41.5 & 41.8 & 46.5 & 41.9 & 47.8 & 42.9 & 44.2 & 4.6 & \textbf{60.1} & 42.8 & 42.0 \\
Kimi K3 & 86.8 & 91.9 & 85.4 & 91.3 & 64.3 & 91.0 & 66.9 & 59.3 & 6.4 & 51.3 & 70.2 & 69.5 \\
Qwen3.7-Plus & 87.0 & 87.5 & 89.6 & 92.2 & 72.7 & 81.2 & 67.2 & 45.1 & 6.1 & 48.2 & 67.0 & 67.7 \\
Qwen3.8-Max & 75.8 & 66.4 & 71.5 & 78.0 & 38.4 & 72.4 & 59.5 & 57.8 & 5.6 & 57.4 & 59.7 & 58.3 \\
Gemini-3.5-Flash & 71.7 & 77.7 & 65.9 & 73.3 & 53.7 & 79.7 & 53.8 & 42.2 & 3.0 & 42.2 & 56.3 & 56.3 \\
GPT-5.4-0305 & 73.0 & 89.5 & 73.4 & 81.2 & \textbf{76.8} & 84.3 & 63.3 & 49.8 & 5.4 & 59.8 & 65.3 & 65.7 \\
Claude Opus 4.8 & 73.4 & 80.2 & 77.7 & 82.2 & 57.5 & 85.9 & 68.1 & \textbf{76.3} & 8.0 & 49.0 & 68.6 & 65.8 \\
\addlinespace[2pt]
\textbf{CASE (8B, ours)} & \textbf{87.8} & \textbf{94.4} & \textbf{94.0} & \textbf{92.5} & 75.7 & \textbf{94.3} & \textbf{71.8} & 69.6 & \textbf{10.9} & 59.7 & \textbf{76.1} & \textbf{75.1} \\
\bottomrule
\end{tabular}}
\caption{Clinical accuracy (\%) averaged over the three roles across the benchmark's domains and question types, comparing medical specialist models, frontier general-purpose multimodal models, and CASE; ``base'' is the untrained Qwen3-VL-8B. Columns are question types (B, S, M = binary, single-select, multi-select) under six domains, with Micro and Macro at right.}
\label{tab:main_accuracy}
\end{table}

\subsection{Evaluation protocol}
\label{sec:protocol}
Anatomist, Radiologist, and Consultant form separate comparison panels, each with a fixed question manifest and tool registry, and methods within a panel share prompts, patient evidence, execution limits, and parsing rules. Cross-role differences are descriptive because evidence eligibility and interfaces differ, with runtime controls in \appref{app:evaluation}; the RL rollout loop and the evaluation harness share tool names, prompts, and caches despite different implementations (\appref{app:environment}).

\subsection{Quantitative results}
\label{sec:quantitative}
\paragraph{Clinical accuracy.}
\Cref{tab:main_accuracy} compares CASE, averaged over the three roles, against the Qwen3-VL-8B-Instruct base~\citep{bai2025qwen3vl}, medical models (Lingshu-32B and Lingshu-7B~\citep{xu2026lingshu}, MedGemma-27B~\citep{sellergren2025medgemma}, and MedGemma-1.5-4B~\citep{sellergren2026medgemma15}), and general models (Kimi K2.6~\citep{moonshot2026kimik26}, Kimi K3~\citep{moonshot2026kimik3}, Qwen3.7-Plus~\citep{qwen2026qwen37plus}, the Qwen3.8-Max teacher~\citep{qwen2026qwen38max}, Gemini-3.5-Flash~\citep{google2026gemini35}, GPT-5.4-0305~\citep{openai2026gpt54}, and Claude Opus 4.8~\citep{anthropic2026opus48}), with the per-role breakdown in \suppref{tab:perrole_accuracy}. All methods receive the same tools (\appref{app:evaluation}), so nominal tool availability does not explain the gap; different abilities to use the shared evidence-seeking harness remain a possible explanation. Ten task--format strata over six clinical domains (\suppref{tab:appendix_types}) complement overall Micro and Macro accuracy (\cref{eq:metrics}), and \suppref{tab:paired_comparisons} gives paired gains and participant-bootstrap intervals against the teacher, testing improvement beyond demonstration imitation.

\paragraph{Interaction efficiency.}
\Suppref{tab:interaction_cost} reports, for the tool-enabled Consultant, turns per episode, the fraction issuing at least one tool call, the answer-cutoff rate, and the invalid-JSON rate. CASE finishes the longitudinal panels in few turns with a low cutoff and negligible invalid rate, so its accuracy is not bought with a larger interaction budget or malformed outputs, matching the tool-cost term in \cref{eq:reward} and the format guard of \appref{app:implementation}.

\subsection{Ablation study}
\label{sec:ablation}
\paragraph{Learning components.}
\Cref{tab:learning_ablation} compares RSFT alone with GRPO using neither addition, rubric feedback alone, OPSD alone, or both; all RL variants retain correctness and format rewards and the same reference-policy regularizer, and rubric feedback adds the continuous insight score and top-2 linkage. Evaluation instead scores each insight with a fixed external judge without top-2 selection, reporting three-role means of accuracy and insight quality on questions scorable for every variant, with judge controls, coverage, and role-specific results in \appref{app:judge_evaluation} and \cref{tab:role_ablation}. The training traces make the complementarity explicit (\cref{fig:ablation_zoom}; full epoch in \cref{fig:ablation_curves}), where without OPSD, GRPO+RLAIF optimizes the rubric with no brake on discourse and restates earlier findings until the correctness reward collapses over the final third of the epoch (highlighted band), whereas CASE-8B and GRPO+OPSD avoid this degeneration and CASE-8B additionally converts the rubric signal into higher endpoint accuracy. Remarkably, even though the rubric is a black-box signal from an external judge, it still drives a steady rise in judged insight quality throughout training, consistently under three different judge models (\cref{fig:judge_sources}; \appref{app:judge_evaluation}).

\begin{figure}[t]\centering

\begin{minipage}[t]{0.49\textwidth}
\vspace{4mm}
\centering
\footnotesize
\resizebox{\linewidth}{!}{%
\setlength{\tabcolsep}{4pt}
\renewcommand{\arraystretch}{1.05}
\begin{tabular}{ccc ccc}
\toprule
\multicolumn{3}{c}{Components} & \multicolumn{3}{c}{Role-averaged} \\
\cmidrule(r){1-3}\cmidrule(l){4-6}
GRPO & RLAIF & OPSD & Micro & Macro & Insight \\
\midrule
\ablateoff & \ablateoff & \ablateoff & 65.7 & 65.0 & 0.43 \\
\ablateon  & \ablateoff & \ablateoff & 69.3 & 68.2 & 0.47 \\
\ablateon  & \ablateon  & \ablateoff & 70.6 & 69.5 & 0.63 \\
\ablateon  & \ablateoff & \ablateon  & 73.8 & 72.4 & 0.57 \\
\addlinespace[2pt]
\ablateon  & \ablateon  & \ablateon  & \textbf{76.1} & \textbf{75.1} & \textbf{0.68} \\
\bottomrule
\end{tabular}}
\captionof{table}{Post-training ablations from a shared RSFT initialization, from RSFT alone (first row) to full CASE-8B (last). Micro and Macro are role-averaged accuracy (\%); Insight is a $[0,1]$ score, from a fixed external judge, of the logical link between the Phase-4 insight and the final answer; RLAIF is defined in \appref{app:judge_evaluation}.}
\label{tab:learning_ablation}
\end{minipage}\hfill
\begin{minipage}[t]{0.49\textwidth}
\vspace{0pt}
\centering
\includegraphics[width=\linewidth]{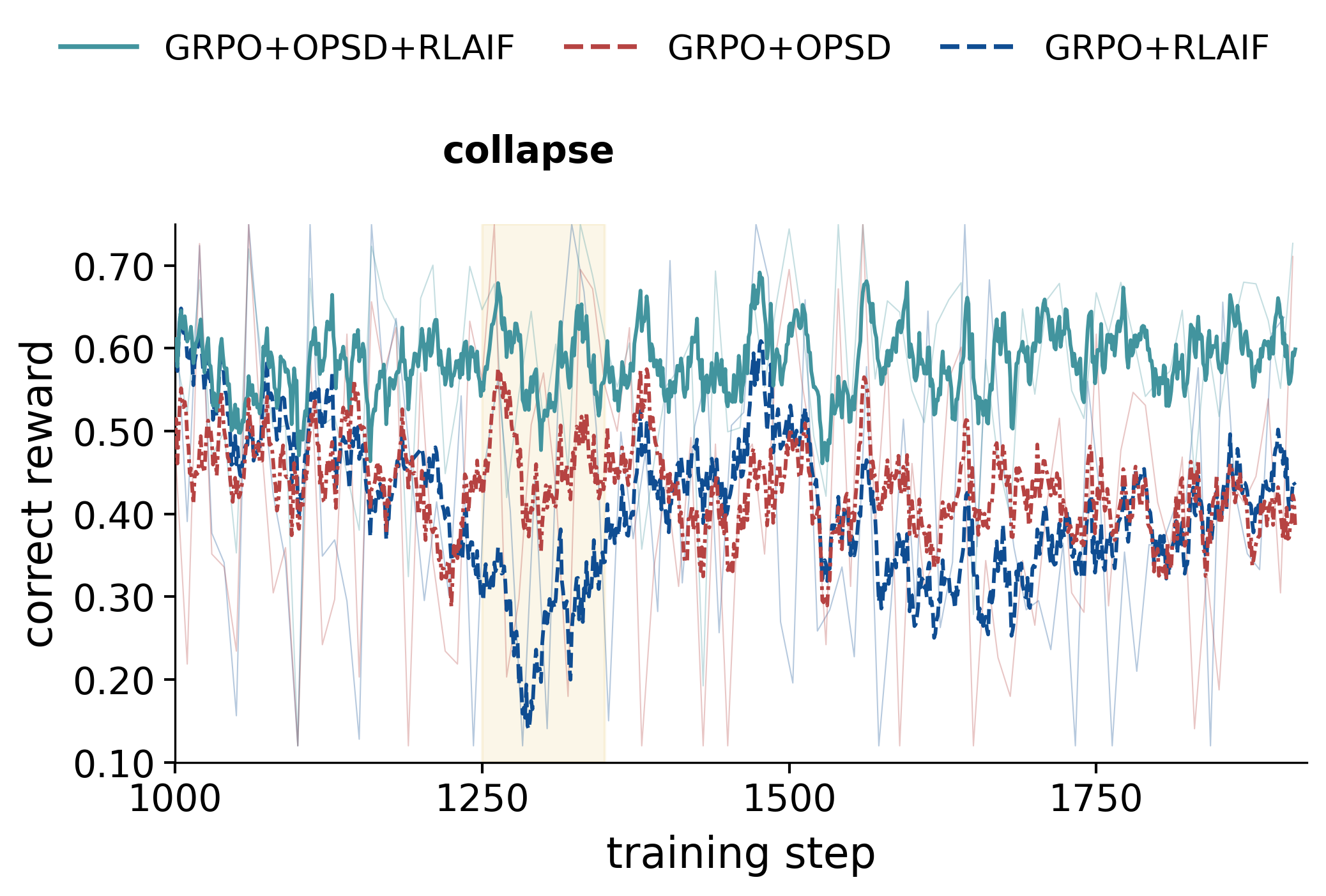}
\captionof{figure}{Correctness reward, steps 1{,}000--1{,}912, for CASE-8B (teal), GRPO+OPSD (red), and GRPO+RLAIF (blue). The band marks GRPO+RLAIF's late collapse without the OPSD brake; full epoch in \cref{fig:ablation_curves}.}
\label{fig:ablation_zoom}
\end{minipage}
\end{figure}

\paragraph{Per-role ablation.}
\Suppref{tab:role_ablation} repeats the component ablation per role, so each ingredient can be read where it acts, from RSFT over the base and GRPO over RSFT to OPSD and rubric feedback on top. Its lower block splits the insight score into the Phase-4 statement and the answer-anchored logic that OPSD and the linkage protect, and reports the fraction of Phase-4 evidence reused in the answer, locating the gain in answer-anchored reasoning and evidence reuse rather than plausibility.

\paragraph{Evidence dependence.}
To confirm the policy uses acquired evidence rather than priors, \suppref{tab:evidence_controls} varies what the agent may obtain (question-only, clinical-context-only, fixed-evidence) and removes each source (MRI, follow-up, reports) in turn (denominators in \appref{app:evaluation}). Accuracy drops whenever acquisition is withheld or a source removed, separating evidence-seeking from answer priors and from format or length artifacts; these ablate the \emph{evidence channel}, not the learning components. Historical pilots remain in \appref{app:historical}.

\subsection{Qualitative analysis}
\label{sec:qualitative}
\appref{app:example_trajectory} renders one complete Consultant rollout showing the behaviour the method targets. The agent recalls a written skill, cross-checks segmentation, cross-visit comparison, and draft reports, and, crucially, doubts an over-calling draft impression, requesting the missing arterial and washout evidence before rejecting it. The Phase-4 insight and final answer agree, tool calls stay within budget, and a post-answer review writes the conclusion back to skill memory, closing the acquisition--synthesis--memory loop of \cref{sec:method}; the trajectory thus makes the aggregate gains legible as doubt resolved against acquired evidence.

\section{Conclusion and Limitations}
\label{sec:conclusion}

\ours{} establishes a general framework for both evaluating and learning clinical investigation, making evidence acquisition measurable while keeping investigation policy-driven and evidence synthesis trainable. Although real clinical practice spans broader toolsets and specialty-specific diagnostic pathways, the same evaluation and training paradigm can be instantiated by specifying the available evidence, tools, and decision contract, providing a principled route from fixed-context assessment and workflow engineering toward learnable clinical investigation. Several limitations remain. Clinical labels may be unobservable from the available images, and generated evidence and judges remain fallible; establishing utility beyond this cohort, including generalization and efficiency, remains future work.

\clearpage
\section*{Ethics Statement}
This study used data from the UK Biobank Resource under Application Number 603483 as part of an approved research project. UK Biobank received ethical approval from the NHS North West Research Ethics Committee (11/NW/0382; 16/NW/0274), and no additional ethics review was required for this study. All analyses followed the applicable UK Biobank data-access and research-governance policies.

\section*{Reproducibility Statement}
We will release the implementation, recorded trajectories, benchmark construction and evaluation code, prompts, and training configurations needed to reproduce our experiments. UK Biobank participant-level data cannot be redistributed directly; access remains subject to UK Biobank approval and policy. Following agreement with UK Biobank, the CASE Benchmark will be returned to UK Biobank and made available to approved researchers through the UK Biobank Returned Datasets Catalogue and Research Analysis Platform (UKB-RAP). Publicly releasable code, benchmark specifications, and supporting artifacts will be hosted on Hugging Face and GitHub.

\clearpage

\bibliographystyle{assets/plainnat}
\bibliography{paper}

@article{jin2020medqa,
  title = {What Disease Does This Patient Have? A Large-scale Open Domain Question Answering Dataset from Medical Exams},
  author = {Jin, Di and Pan, Eileen and Oufattole, Nassim and Weng, Wei-Hung and Fang, Hanyi and Szolovits, Peter},
  journal = {arXiv preprint arXiv:2009.13081},
  year = {2020},
}

@inproceedings{pal2022medmcqa,
  title = {{MedMCQA}: A Large-scale Multi-Subject Multi-Choice Dataset for Medical Domain Question Answering},
  author = {Pal, Ankit and Umapathi, Logesh Kumar and Sankarasubbu, Malaikannan},
  booktitle = {Proceedings of the Conference on Health, Inference, and Learning},
  series = {Proceedings of Machine Learning Research},
  volume = {174},
  pages = {248--260},
  year = {2022},
}

@inproceedings{jin2019pubmedqa,
  title = {{PubMedQA}: A Dataset for Biomedical Research Question Answering},
  author = {Jin, Qiao and Dhingra, Bhuwan and Liu, Zhengping and Cohen, William and Lu, Xinghua},
  booktitle = {Proceedings of the 2019 Conference on Empirical Methods in Natural Language Processing and the 9th International Joint Conference on Natural Language Processing (EMNLP-IJCNLP)},
  pages = {2567--2577},
  year = {2019},
  doi = {10.18653/v1/D19-1259},
}

@inproceedings{zuo2025medxpertqa,
  title = {{MedXpertQA}: Benchmarking Expert-Level Medical Reasoning and Understanding},
  author = {Zuo, Yuxin and Qu, Shang and Li, Yifei and Chen, Zhang-Ren and Zhu, Xuekai and Hua, Ermo and Zhang, Kaiyan and Ding, Ning and Zhou, Bowen},
  booktitle = {Proceedings of the 42nd International Conference on Machine Learning},
  series = {Proceedings of Machine Learning Research},
  volume = {267},
  pages = {80961--80990},
  year = {2025},
}

@inproceedings{butsanets2026radimagenetvqa,
  title = {{RadImageNet-VQA}: A Large-Scale {CT} and {MRI} Dataset for Radiologic Visual Question Answering},
  author = {Butsanets, L{\'e}o and Corbi{\`e}re, Charles and Khlaut, Julien and Manceron, Pierre and Dancette, Corentin},
  booktitle = {Proceedings of the 9th International Conference on Medical Imaging with Deep Learning},
  series = {Proceedings of Machine Learning Research},
  volume = {315},
  pages = {3036--3068},
  year = {2026},
}

@article{bourigault2025ukbob,
  title = {{UKBOB}: One Billion {MRI} Labeled Masks for Generalizable {3D} Medical Image Segmentation},
  author = {Bourigault, Emmanuelle and Jamaludin, Amir and Hamdi, Abdullah},
  journal = {arXiv preprint arXiv:2504.06908},
  year = {2025},
}

@article{lau2018vqarad,
  title = {A Dataset of Clinically Generated Visual Questions and Answers about Radiology Images},
  author = {Lau, Jason J. and Gayen, Soumya and Ben Abacha, Asma and Demner-Fushman, Dina},
  journal = {Scientific Data},
  volume = {5},
  pages = {180251},
  year = {2018},
  doi = {10.1038/sdata.2018.251},
}

@inproceedings{liu2021slake,
  title = {{SLAKE}: A Semantically-Labeled Knowledge-Enhanced Dataset for Medical Visual Question Answering},
  author = {Liu, Bo and Zhan, Li-Ming and Xu, Li and Ma, Lin and Yang, Yan and Wu, Xiao-Ming},
  booktitle = {IEEE International Symposium on Biomedical Imaging (ISBI)},
  year = {2021},
}

@inproceedings{li2023llavamed,
  title = {{LLaVA-Med}: Training a Large Language-and-Vision Assistant for Biomedicine in One Day},
  author = {Li, Chunyuan and Wong, Cliff and Zhang, Sheng and Usuyama, Naoto and Liu, Haotian and Yang, Jianwei and Naumann, Tristan and Poon, Hoifung and Gao, Jianfeng},
  booktitle = {Advances in Neural Information Processing Systems},
  volume = {36},
  year = {2023},
}

@article{jiang2025medagentbench,
  title = {{MedAgentBench}: A Realistic Virtual {EHR} Environment to Benchmark Medical {LLM} Agents},
  author = {Jiang, Yixing and Black, Kameron C. and Geng, Gloria and Park, Danny and Zou, James and Ng, Andrew Y. and Chen, Jonathan H.},
  journal = {arXiv preprint arXiv:2501.14654},
  year = {2025},
}

@article{schmidgall2024agentclinic,
  title = {{AgentClinic}: A Multimodal Agent Benchmark to Evaluate {AI} in Simulated Clinical Environments},
  author = {Schmidgall, Samuel and Ziaei, Rojin and Harris, Carl and Reis, Eduardo and Jopling, Jeffrey and Moor, Michael},
  journal = {arXiv preprint arXiv:2405.07960},
  year = {2024},
}

@article{qiu2025medrbench,
  title = {Quantifying the Reasoning Abilities of {LLMs} on Real-world Clinical Cases},
  author = {Qiu, Pengcheng and Wu, Chaoyi and Liu, Shuyu and Zhao, Weike and Chen, Zhuoxia and Gu, Hongfei and Peng, Chuanjin and Zhang, Ya and Wang, Yanfeng and Xie, Weidi},
  journal = {arXiv preprint arXiv:2503.04691},
  year = {2025},
}

@inproceedings{vasilev2025mtbbench,
  title = {{MTBBench}: A Multimodal Sequential Clinical Decision-Making Benchmark in Oncology},
  author = {Vasilev, Kiril and Misrahi, Alexandre and Jain, Eeshaan and Cheng, Phil F. and Liakopoulos, Petros and Michielin, Olivier and Moor, Michael and Bunne, Charlotte},
  booktitle = {Advances in Neural Information Processing Systems},
  volume = {38},
  year = {2025},
}

@article{lu2026clinenv,
  title = {{ClinEnv}: An Interactive Multi-Stage Long Horizon {EHR} Environment for Agents},
  author = {Lu, Yuxing and Lin, Yushuhong and Shi, Wenqi and Tamo, J. Ben and Zhao, Xukai and Wang, Jinzhuo and Wang, May Dongmei},
  journal = {arXiv preprint arXiv:2606.02568},
  year = {2026},
}

@inproceedings{yao2023react,
  title = {{ReAct}: Synergizing Reasoning and Acting in Language Models},
  author = {Yao, Shunyu and Zhao, Jeffrey and Yu, Dian and Du, Nan and Shafran, Izhak and Narasimhan, Karthik and Cao, Yuan},
  booktitle = {International Conference on Learning Representations},
  year = {2023},
}

@inproceedings{schick2023toolformer,
  title = {{Toolformer}: Language Models Can Teach Themselves to Use Tools},
  author = {Schick, Timo and Dwivedi-Yu, Jane and Dess{\`i}, Roberto and Raileanu, Roberta and Lomeli, Maria and Hambro, Eric and Zettlemoyer, Luke and Cancedda, Nicola and Scialom, Thomas},
  booktitle = {Advances in Neural Information Processing Systems},
  volume = {36},
  year = {2023},
}

@inproceedings{agarwal2024gkd,
  title = {On-Policy Distillation of Language Models: Learning from Self-Generated Mistakes},
  author = {Agarwal, Rishabh and Vieillard, Nino and Zhou, Yongchao and Stanczyk, Piotr and Ramos, Sabela and Geist, Matthieu and Bachem, Olivier},
  booktitle = {International Conference on Learning Representations},
  year = {2024},
}

@inproceedings{gu2024minillm,
  title = {{MiniLLM}: Knowledge Distillation of Large Language Models},
  author = {Gu, Yuxian and Dong, Li and Wei, Furu and Huang, Minlie},
  booktitle = {International Conference on Learning Representations},
  year = {2024},
}

@article{zhao2026opsd,
  title = {Self-Distilled Reasoner: On-Policy Self-Distillation for Large Language Models},
  author = {Zhao, Siyan and Xie, Zhihui and Liu, Mengchen and Huang, Jing and Pang, Guan and Chen, Feiyu and Grover, Aditya},
  journal = {arXiv preprint arXiv:2601.18734},
  year = {2026},
}

@article{hubotter2026sdpo,
  title = {Reinforcement Learning via Self-Distillation},
  author = {H{\"u}botter, Jonas and L{\"u}beck, Frederike and Behric, Lejs and Baumann, Anton and Bagatella, Marco and Marta, Daniel and Hakimi, Ido and Shenfeld, Idan and Kleine Buening, Thomas and Guestrin, Carlos and Krause, Andreas},
  journal = {arXiv preprint arXiv:2601.20802},
  year = {2026},
}

@article{tian2026pbsd,
  title = {{PBSD}: Privileged Bayesian Self-Distillation for Long-Horizon Credit Assignment},
  author = {Tian, Yang and Wang, Rui and Wen, Xumeng and Li, Junjie and Sun, Shizhao and Song, Lei and Bian, Jiang and Zhao, Bo},
  journal = {arXiv preprint arXiv:2606.09348},
  year = {2026},
}

@article{zhang2026adrs,
  title = {Agentic Reinforcement Learning with Self-Distilled Reward Shaping},
  author = {Zhang, Ranxu and Chen, Guinan and {Chenshaodong} and Lin, Jinghao and Xu, Xiaozhou and {Sunzhe} and Zhang, Yanyong and Wang, Chao},
  journal = {arXiv preprint arXiv:2608.03223},
  year = {2026},
}

@article{zhu2026manyfaces,
  title = {The Many Faces of On-Policy Distillation: Pitfalls, Mechanisms, and Fixes},
  author = {Zhu, Siqi and Ye, Xuyan and Lu, Hongyu and Shi, Weiye and Liu, Ge},
  journal = {arXiv preprint arXiv:2605.11182},
  year = {2026},
}

@article{schulman2017ppo,
  title = {Proximal Policy Optimization Algorithms},
  author = {Schulman, John and Wolski, Filip and Dhariwal, Prafulla and Radford, Alec and Klimov, Oleg},
  journal = {arXiv preprint arXiv:1707.06347},
  year = {2017},
}

@article{shao2024deepseekmath,
  title = {{DeepSeekMath}: Pushing the Limits of Mathematical Reasoning in Open Language Models},
  author = {Shao, Zhihong and Wang, Peiyi and Zhu, Qihao and Xu, Runxin and Song, Junxiao and Bi, Xiao and Zhang, Haowei and Zhang, Mingchuan and Li, Y. K. and Wu, Y. and Guo, Daya},
  journal = {arXiv preprint arXiv:2402.03300},
  year = {2024},
}

@inproceedings{zheng2023judge,
  title = {Judging {LLM}-as-a-Judge with {MT-Bench} and Chatbot Arena},
  author = {Zheng, Lianmin and Chiang, Wei-Lin and Sheng, Ying and Zhuang, Siyuan and Wu, Zhanghao and Zhuang, Yonghao and Lin, Zi and Li, Zhuohan and Li, Dacheng and Xing, Eric P. and Zhang, Hao and Gonzalez, Joseph E. and Stoica, Ion},
  booktitle = {Advances in Neural Information Processing Systems},
  volume = {36},
  year = {2023},
}

@article{zhang2026rlad,
  title = {Reinforcement-aware Knowledge Distillation for {LLM} Reasoning},
  author = {Zhang, Zhaoyang and Jiang, Shuli and Shen, Yantao and Zhang, Yuting and Ram, Dhananjay and Yang, Shuo and Tu, Zhuowen and Xia, Wei and Soatto, Stefano},
  journal = {arXiv preprint arXiv:2602.22495},
  year = {2026},
}

@article{wang2026trace,
  title = {{TRACE}: Distilling Where It Matters via Token-Routed Self On-Policy Alignment},
  author = {Wang, Jiaxuan and Ouyang, Xuan and Chen, Zhiyu and Hu, Yulan and Pan, Zheng and Li, Xin and Guo, Lan-Zhe},
  journal = {arXiv preprint arXiv:2605.10194},
  year = {2026},
}

@article{tan2026ssopd,
  title = {Self-Supervised On-Policy Distillation for Reasoning Language Models},
  author = {Tan, Zhiquan and Hong, Yinrong},
  journal = {arXiv preprint arXiv:2605.17497},
  year = {2026},
}

@article{liu2026tgpo,
  title = {Teacher-Guided Policy Optimization for On-Policy Reasoning Distillation under Large Policy Divergence},
  author = {Liu, Xinyu and Jiao, Kechen and Xiao, Chunyang and Zhao, Runsong and Ruan, Junhao and Li, Bei and Liu, Jiahao and Wang, Qifan and Chen, Xin and Wang, Jingang and Wang, Chenglong and Xiao, Tong and Zhu, Jingbo},
  journal = {arXiv preprint arXiv:2605.13230},
  year = {2026},
}

@article{yang2026ocsd,
  title = {Agentic Reinforcement Learning with Observation-Calibrated Self-Distillation},
  author = {Yang, Yi and Qin, Cong and Liu, Xiaodan and Chen, Chishui and Dong, Qing and Zhang, Yan and Liu, Cao and Yang, Zhao and Pan, Lu and Lin, Jiaye and Feng, Yi},
  journal = {arXiv preprint arXiv:2608.04788},
  year = {2026},
}

@article{wu2026sspo,
  title = {Beyond Outcome Rewards: Step-Level Self-Distilled Policy Optimization for Deep Search Agents},
  author = {Wu, Haoze and Kuang, Chuqiao and Zhuang, Tianyi and Li, Xiaoguang},
  journal = {arXiv preprint arXiv:2608.12764},
  year = {2026},
}

@article{lin2026opdvr,
  title = {On-policy Distillation with Verifiable Reward},
  author = {Lin, Wenze and Zhao, Jiale and Jiang, Xitai and Rao, Songde and Li, Yining and Wang, Shenzhi and He, Bingxiang and Huang, Gao},
  journal = {arXiv preprint arXiv:2608.24696},
  year = {2026},
}

@inproceedings{kim2024mdagents,
  title = {{MDAgents}: An Adaptive Collaboration of {LLMs} for Medical Decision-Making},
  author = {Kim, Yubin and Park, Chanwoo and Jeong, Hyewon and Chan, Yik Siu and Xu, Xuhai and McDuff, Daniel and Lee, Hyeonhoon and Ghassemi, Marzyeh and Breazeal, Cynthia and Park, Hae Won},
  booktitle = {Advances in Neural Information Processing Systems},
  volume = {37},
  year = {2024},
}

@inproceedings{li2024mmedagent,
  title = {{MMedAgent}: Learning to Use Medical Tools with Multi-modal Agent},
  author = {Li, Binxu and Yan, Tiankai and Pan, Yuanting and Luo, Jie and Ji, Ruiyang and Ding, Jiayuan and Xu, Zhe and Liu, Shilong and Dong, Haoyu and Lin, Zihao and Wang, Yixin},
  booktitle = {Findings of the Association for Computational Linguistics: EMNLP 2024},
  pages = {8745--8760},
  year = {2024},
  doi = {10.18653/v1/2024.findings-emnlp.510},
}

@inproceedings{fallahpour2025medrax,
  title = {{MedRAX}: Medical Reasoning Agent for Chest X-ray},
  author = {Fallahpour, Adibvafa and Ma, Jun and Munim, Alif and Lyu, Hongwei and Wang, Bo},
  booktitle = {Proceedings of the 42nd International Conference on Machine Learning},
  series = {Proceedings of Machine Learning Research},
  volume = {267},
  pages = {15661--15676},
  year = {2025},
}

@inproceedings{almansoori2025medagentsim,
  title = {{MedAgentSim}: Self-Evolving Multi-Agent Simulations for Realistic Clinical Interactions},
  author = {Almansoori, Mohammad and Kumar, Komal and Cholakkal, Hisham},
  booktitle = {Medical Image Computing and Computer Assisted Intervention -- MICCAI 2025},
  series = {Lecture Notes in Computer Science},
  volume = {15968},
  pages = {362--372},
  publisher = {Springer Nature Switzerland},
  year = {2025},
  doi = {10.1007/978-3-032-05114-1_35},
}

@inproceedings{liu2025medchain,
  title = {{MedChain}: Bridging the Gap Between {LLM} Agents and Clinical Practice with Interactive Sequence},
  author = {Liu, Jie and Wang, Wenxuan and Ma, Zizhan and Huang, Guolin and Su, Yihang and Chang, Kao-Jung and Li, Haoliang and Shen, Linlin and Lyu, Michael R. and Chen, Wenting},
  booktitle = {Advances in Neural Information Processing Systems},
  volume = {38},
  year = {2025},
}

@inproceedings{chiu2025vivabench,
  title = {Simulating Viva Voce Examinations to Evaluate Clinical Reasoning in Large Language Models},
  author = {Chiu, Christopher and Pitis, Silviu and van der Schaar, Mihaela},
  booktitle = {Advances in Neural Information Processing Systems},
  volume = {38},
  year = {2025},
}

@inproceedings{zhu2026medagentboard,
  title = {{MedAgentBoard}: Benchmarking multi-agent collaboration with conventional methods for diverse medical tasks},
  author = {Zhu, Yinghao and He, Ziyi and Hu, Haoran and Zheng, Xiaochen and Zhang, Xichen and Wang, Jiyao and Gao, Junyi and Ma, Liantao and Yu, Lequan},
  booktitle = {Advances in Neural Information Processing Systems},
  volume = {38},
  year = {2025},
  doi = {10.52202/085713-4851},
}

@inproceedings{yan2025clinicallab,
  title = {{ClinicalLab}: Aligning Agents for Multi-Departmental Clinical Diagnostics in the Real World},
  author = {Yan, Weixiang and Liu, Haitian and Wu, Tengxiao and Chen, Qian and Wang, Wen and Chai, Haoyuan and Wang, Jiayi},
  booktitle = {Advances in Neural Information Processing Systems},
  volume = {38},
  year = {2025},
}

@inproceedings{wang2026medagentpro,
  title = {{MedAgent-Pro}: Towards Evidence-based Multi-modal Medical Diagnosis via Reasoning Agentic Workflow},
  author = {Wang, Ziyue and Wu, Junde and Cai, Linghan and Low, Chang Han and Yang, Xihong and Li, Qiaxuan and Jin, Yueming},
  booktitle = {The Fourteenth International Conference on Learning Representations},
  year = {2026},
}

@inproceedings{zhang2026radagents,
  title = {{RadAgents}: Multimodal Agentic Reasoning for Chest X-ray Interpretation with Radiologist-like Workflows},
  author = {Zhang, Kai and Barrett, Corey D. and Kim, Jangwon and Sun, Lichao and Taghavi, Tara and Kenthapadi, Krishnaram},
  booktitle = {Proceedings of the 9th International Conference on Medical Imaging with Deep Learning},
  series = {Proceedings of Machine Learning Research},
  volume = {315},
  pages = {3496--3519},
  year = {2026},
}

@inproceedings{jiang2026medvr,
  title = {{MedVR}: Annotation-Free Medical Visual Reasoning via Agentic Reinforcement Learning},
  author = {Jiang, Zheng and Guo, Heng and Fang, Chengyu and Xiao, Changchen and Hu, Xinyang and Sun, Lifeng and Xu, Minfeng},
  booktitle = {The Fourteenth International Conference on Learning Representations},
  year = {2026},
}

@article{jiang2026ophiuchus,
  title = {{Ophiuchus}: Incentivizing Tool-augmented ``Think with Images'' for Joint Medical Segmentation, Understanding and Reasoning},
  author = {Jiang, Yankai and Zhang, Yujie and Zhang, Peng and Li, Wenjie and Li, Yichen and Chen, Jintai and Shi, Xiaoming and Zhen, Shihui},
  journal = {arXiv preprint arXiv:2512.14157},
  year = {2026},
  note = {Version 2, July 2026; first submitted December 2025},
}

@article{fan2026macro,
  title = {Evolving Medical Imaging Agents via Experience-driven Self-skill Discovery},
  author = {Fan, Lin and Dai, Pengyu and Deng, Zhipeng and Wang, Haolin and Gong, Xun and Zheng, Yefeng and Ou, Yafei},
  journal = {arXiv preprint arXiv:2603.05860},
  year = {2026},
}

@article{sudlow2015ukbiobank,
  title = {{UK Biobank}: An Open Access Resource for Identifying the Causes of a Wide Range of Complex Diseases of Middle and Old Age},
  author = {Sudlow, Cathie and Gallacher, John and Allen, Naomi and Beral, Valerie and Burton, Paul and Danesh, John and Downey, Paul and Elliott, Paul and Green, Jane and Landray, Martin and Liu, Bette and Matthews, Paul and Ong, Giok and Pell, Jill and Silman, Alan and Young, Alan and Sprosen, Tim and Peakman, Tim and Collins, Rory},
  journal = {PLoS Medicine},
  volume = {12},
  number = {3},
  pages = {e1001779},
  year = {2015},
  doi = {10.1371/journal.pmed.1001779},
}

@misc{qwen2026qwen38max,
  title = {{Qwen3.8-Max}: A New Bar for Coding and Cowork},
  author = {{Qwen Team}},
  year = {2026},
  month = {August},
  url = {https://qwen.ai/blog?id=qwen3.8},
}

@misc{qwen2026qwen37plus,
  title = {{Qwen3.7-Plus}: Multimodal Agent Intelligence},
  author = {{Qwen Team}},
  year = {2026},
  month = {May},
  url = {https://qwen.ai/blog?id=qwen3.7-plus},
}

@article{deepseekai2026v4,
  title = {{DeepSeek-V4}: Towards Highly Efficient Million-Token Context Intelligence},
  author = {{DeepSeek-AI} and Xu, Anyi and Lin, Bangcai and Xue, Bing and Wang, Bingxuan and others},
  journal = {arXiv preprint arXiv:2606.19348},
  year = {2026},
  url = {https://arxiv.org/abs/2606.19348},
}

@misc{qwen2026qwen35,
  title = {{Qwen3.5}: Towards Native Multimodal Agents},
  author = {{Qwen Team}},
  year = {2026},
  month = {February},
  url = {https://qwen.ai/blog?id=qwen3.5},
}

@misc{moonshot2026kimik26,
  title = {{Kimi K2.6}: Advancing Open-Source Coding},
  author = {{Moonshot AI}},
  year = {2026},
  howpublished = {\href{https://www.kimi.ai/blog/kimi-k2-6}{Official technical blog}},
}

@misc{anthropic2026opus48,
  title = {Introducing {Claude Opus 4.8}},
  author = {{Anthropic}},
  year = {2026},
  howpublished = {\href{https://www.anthropic.com/news/claude-opus-4-8}{Official model release}},
}

@misc{openai2026gpt54,
  title = {{GPT-5.4} Model},
  author = {{OpenAI}},
  year = {2026},
  howpublished = {\href{https://developers.openai.com/api/docs/models/gpt-5.4}{Official model documentation}},
}

@misc{google2026gemini35,
  title = {{Gemini 3.5 Flash}},
  author = {{Google DeepMind}},
  year = {2026},
  howpublished = {\href{https://ai.google.dev/gemini-api/docs/models/gemini-3.5-flash}{Official model documentation}},
}

@misc{moonshot2026kimik3,
  title = {{Kimi K3}: Open Frontier Intelligence},
  author = {{Moonshot AI}},
  year = {2026},
  howpublished = {\href{https://www.kimi.ai/blog/kimi-k3}{Official technical blog}},
}

@misc{langchain2024langgraph,
  title = {{LangGraph}},
  author = {{LangChain Team}},
  year = {2024},
  howpublished = {\href{https://github.com/langchain-ai/langgraph}{Software framework}},
}

@article{bai2025qwen3vl,
  title = {{Qwen3-VL} Technical Report},
  author = {Bai, Shuai and Cai, Yuxuan and Chen, Ruizhe and Chen, Keqin and Chen, Xionghui and others},
  journal = {arXiv preprint arXiv:2511.21631},
  year = {2025},
  url = {https://arxiv.org/abs/2511.21631},
}

@article{xu2026lingshu,
  title={Lingshu: Generalist Foundation Model for Unified Multimodal Medical Understanding and Reasoning},
  author={Xu, Weiwen and Chan, Hou Pong and Li, Long and Aljunied, Mahani and Yuan, Ruifeng and Wang, Jianyu and Xiao, Chenghao and Chen, Guizhen and Liu, Chaoqun and Li, Zhaodonghui and others},
  journal={IEEE Transactions on Pattern Analysis and Machine Intelligence},
  year={2026},
  publisher={IEEE}
}

@inproceedings{welleck-etal-2020-consistency,
  title = {Consistency of a Recurrent Language Model With Respect to Incomplete Decoding},
  author = {Welleck, Sean and Kulikov, Ilia and Kim, Jaedeok and Pang, Richard Yuanzhe and Cho, Kyunghyun},
  booktitle = {Proceedings of the 2020 Conference on Empirical Methods in Natural Language Processing (EMNLP)},
  year = {2020},
  publisher = {Association for Computational Linguistics},
  pages = {5553--5568},
  doi = {10.18653/v1/2020.emnlp-main.448},
  url = {https://aclanthology.org/2020.emnlp-main.448/}
}

@inproceedings{he2025vista3d,
  title = {{VISTA3D}: A Unified Segmentation Foundation Model for {3D} Medical Imaging},
  author = {He, Yufan and Guo, Pengfei and Tang, Yucheng and Myronenko, Andriy and Nath, Vishwesh and Xu, Ziyue and Yang, Dong and Zhao, Can and Simon, Benjamin and Belue, Mason and Harmon, Stephanie and Turkbey, Baris and Xu, Daguang and Li, Wenqi},
  booktitle = {Proceedings of the IEEE/CVF Conference on Computer Vision and Pattern Recognition},
  pages = {20863--20873},
  year = {2025},
}

@article{sellergren2025medgemma,
  title = {{MedGemma} Technical Report},
  author = {Sellergren, Andrew and Kazemzadeh, Sahar and Jaroensri, Tiam and Kiraly, Atilla and Traverse, Madeleine and Kohlberger, Timo and Xu, Shawn and Jamil, Fayaz and Hughes, C{\'i}an and Lau, Charles and others},
  journal = {arXiv preprint arXiv:2507.05201},
  year = {2025},
}

@article{sellergren2026medgemma15,
  title = {{MedGemma 1.5} Technical Report},
  author = {Sellergren, Andrew and Gao, Chufan and Mahvar, Fereshteh and Kohlberger, Timo and Jamil, Fayaz and others},
  journal = {arXiv preprint arXiv:2604.05081},
  year = {2026},
  url = {https://arxiv.org/abs/2604.05081},
}

@article{littlejohns2020ukbimaging,
  title = {The {UK Biobank} Imaging Enhancement of 100,000 Participants: Rationale, Data Collection, Management and Future Directions},
  author = {Littlejohns, Thomas J. and Holliday, Jo and Gibson, Lorna M. and Garratt, Steve and Oesingmann, Niels and others},
  journal = {Nature Communications},
  volume = {11},
  pages = {2624},
  year = {2020},
  doi = {10.1038/s41467-020-15948-9},
}

@misc{ukbiobankimagingdata,
  title = {Imaging Data},
  author = {{UK Biobank}},
  howpublished = {\href{https://community.ukbiobank.ac.uk/hc/en-gb/articles/24618819821981-Imaging-Data}{UK Biobank documentation}},
  note = {Accessed 2026-09-28},
}

@inproceedings{lee2024rlaif,
  title = {{RLAIF} vs. {RLHF}: Scaling Reinforcement Learning from Human Feedback with {AI} Feedback},
  author = {Lee, Harrison and Phatale, Samrat and Mansoor, Hassan and Mesnard, Thomas and Ferret, Johan and Lu, Kellie Ren and Bishop, Colton and Hall, Ethan and Carbune, Victor and Rastogi, Abhinav and Prakash, Sushant},
  booktitle = {Proceedings of the 41st International Conference on Machine Learning},
  series = {Proceedings of Machine Learning Research},
  volume = {235},
  pages = {26874--26901},
  year = {2024},
  url = {https://proceedings.mlr.press/v235/lee24t.html},
}

@inproceedings{kwon2023vllm,
  title = {Efficient Memory Management for Large Language Model Serving with {PagedAttention}},
  author = {Kwon, Woosuk and Li, Zhuohan and Zhuang, Siyuan and Sheng, Ying and Zheng, Lianmin and Yu, Cody Hao and Gonzalez, Joseph E. and Zhang, Hao and Stoica, Ion},
  booktitle = {Proceedings of the 29th Symposium on Operating Systems Principles},
  pages = {611--626},
  year = {2023},
  doi = {10.1145/3600006.3613165},
}

@inproceedings{sheng2025hybridflow,
  title = {{HybridFlow}: A Flexible and Efficient {RLHF} Framework},
  author = {Sheng, Guangming and Zhang, Chi and Ye, Zilingfeng and Wu, Xibin and Zhang, Wang and Zhang, Ru and Peng, Yanghua and Lin, Haibin and Wu, Chuan},
  booktitle = {Proceedings of the Twentieth European Conference on Computer Systems},
  year = {2025},
  doi = {10.1145/3689031.3696075},
}

\newpage
\beginappendix
\section{Benchmark Construction and Information Boundaries}
\label{app:data}

\paragraph{Cohort and temporal scope.}
The preparation inventory contains 5,265 imaging participants; the corrected question bank contains 50,401 questions from 4,739 participants, including 4,850 held-out questions from 455 participants. These are distinct construction stages. Binary, other single-select, and multi-select questions account for 17,502, 19,790, and 13,109 items. Within the question-bank cohort, 3,566 participants have three organ groups, 929 have two, and 244 have one. Baseline and follow-up denote the earlier and later imaging visits (UK Biobank Instances 2 and 3)~\citep{ukbiobankimagingdata,littlejohns2020ukbimaging}, not recruitment and incident diagnosis. Participants are aged 45--85 at baseline, imaged from 2014 with a median inter-visit interval of 2.26 years, and average roughly ten questions each; organ-specific availability is recorded separately.

\paragraph{Diagnosis-derived question construction.}
The source loader joins participant identifiers to five clinical fields in the participant-level diagnosis table. Cancer status uses Instance 3; diabetes, vascular/heart diagnoses, and the mixed respiratory/allergy diagnosis list use Instance 2. ICD-based targets use the consolidated diagnosis-code list. The original question generator constructs five base tasks: two binary diagnosis questions and three diagnosis-set questions covering ICD-10 conditions, cardiovascular conditions, and respiratory/allergy conditions.

The source generator uses random seed 42. Diagnosis vocabularies collect distinct values across the loaded cohort. Distractors exclude each participant's reference conditions. ICD-10 items preferentially sample codes sharing the recorded codes' initial-letter chapter, supplementing from other chapters to reach four distractors. Cardiovascular and respiratory/allergy items sample $\max(2,4-n_+)$ distractors, where $n_+$ counts recorded positive conditions, and append ``None of the above''; that option is correct when no positive condition remains. Nonbinary options are shuffled. These are source-generation rules; final options are defined by the corrected bank.

Task expansion adds single-condition variants, ICD-derived disease-system selection, personal/family-history targets from Z codes, and essential hypertension from I10. Subsequent correction adds ``None of the above'' when a single-select reference diagnosis is absent from the options and removes cases with multiple valid selections. Together these stages produce the ten strata in \suppref{tab:appendix_types}, rather than a direct five-task export. Reproduction requires the expansion mappings, exact history-code predicate, correction counts, and source-to-final question identifiers.

The original binary rule maps values beginning with ``yes'' to Yes and all others to No; missing and nonresponse values therefore require explicit validation. The source vocabulary pools the full loaded cohort, not only training participants. Vocabulary provenance, missing-value treatment, and split membership must accompany the released bank.

\begin{table}[htbp]
\centering\small
\begin{tabularx}{\textwidth}{@{}p{.22\textwidth}Xlr@{}}
\toprule
Clinical area & Target & Form & Questions \\
\midrule
Oncology & Doctor-diagnosed cancer & B & 4,710 \\
Endocrine/metabolic & Doctor-diagnosed diabetes & B & 4,718 \\
Cardiovascular & Condition identification & S & 4,685 \\
Cardiovascular & Condition set & M & 4,739 \\
Cardiovascular & Essential hypertension (I10) & B & 4,037 \\
Respiratory/allergy & Condition identification & S & 5,013 \\
Respiratory/allergy & Condition set & M & 4,733 \\
Multisystem disease & ICD-10 condition identification & S & 10,092 \\
Multisystem disease & Affected disease systems & M & 3,637 \\
Medical history & Personal/family history or Z-code & B & 4,037 \\
\bottomrule
\end{tabularx}
\caption{Clinical organization of the ten task--format strata. B, S, and M denote binary, single-select, and multi-select; multi-select cardinality varies by question. Hypertension is grouped with cardiovascular medicine, while its source label remains a distinct target family.}
\label{tab:appendix_types}
\end{table}

\paragraph{Composition and difficulty.}
Panel~2 of \cref{fig:benchmark_overview} flows all 50{,}401 questions from seven clinical families to three answer formats. The families are Cancer (4{,}710), Diabetes (4{,}718), Cardiovascular (9{,}424), Respiratory/allergy (9{,}746), ICD (13{,}729), History (4{,}037), and Hypertension (4{,}037), aggregating the ten strata of \suppref{tab:appendix_types}; the formats are Binary (17{,}502, 34.7\%), Single-select (19{,}790, 39.3\%), and Multi-select (13{,}109, 26.0\%). Panel~5 partitions the 4{,}739 participants into eleven mutually exclusive organ/protocol combinations. Each code lists the imaging sequences a participant has, where W denotes whole-body Dixon, H cardiac cine, P pancreas GE, B brain T1--T2/FLAIR, and L liver IDEAL. The counts are WHP 1{,}404, WHB 929, WH 843, WHL 527, WPB 438, WPL 267, B 142, H 102, HB 78, WB 8, and WBL 1, which sum to the 3{,}566 / 929 / 244 participants holding three / two / one organ groups. \Cref{fig:benchmark_stats} then quantifies difficulty and interaction depth across the three roles.

\begin{figure}[htbp]
\centering
\includegraphics[width=\linewidth]{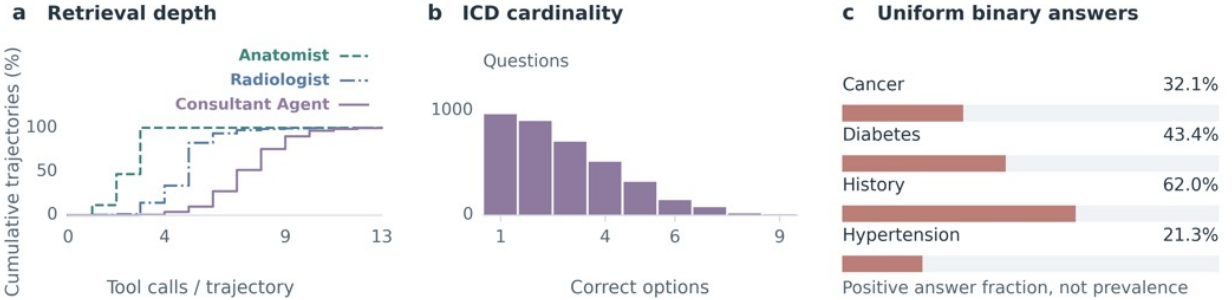}
\caption{\textbf{Benchmark difficulty and interaction depth.} (a) First-pass tool-call depth per role over 46{,}613 / 48{,}587 / 49{,}827 auditable Anatomist / Radiologist / Consultant trajectories; the median is 3 calls for Anatomist (almost all $\le4$), 5 for Radiologist (90\% $\le6$), and 7 for Consultant (90\% $\le9$), a hierarchy of evidence-acquisition complexity. (b) Answer cardinality of the 3{,}637 ICD multi-select items; the number of correct options $k$ spans 1--9 and concentrates at $k\le4$ (965 / 899 / 702 / 511 for $k{=}1$--$4$, ${\approx}85\%$), so exact-set scoring must recover a few correct options out of nine. (c) Positive-class proportion of the four binary families, Cancer 32.1\%, Diabetes 43.4\%, History 62.0\%, and Hypertension 21.3\% over 4{,}710 / 4{,}718 / 4{,}037 / 4{,}037 items; none is near 0 or 1, so always answering the majority class cannot inflate accuracy.}
\label{fig:benchmark_stats}
\end{figure}

\paragraph{Vista3D-derived anatomy.}
Vista3D processes whole-body Dixon MRI offline, producing station-level masks across the six principal neck-to-ankle stations (4, 8, 12, 16, 20, 24; some participants have fewer). For a station mask $S$, anatomical label $\ell$, and voxel spacing $(s_x,s_y,s_z)$ in millimetres read from the NIfTI header, the organ's voxel count $N_\ell=\sum_v\ind[S(v)=\ell]$ (a per-voxel equality test against the label) times the voxel volume gives the station volume in millilitres,
\begin{equation}
 V_{\ell}=10^{-3}s_xs_ys_z\,N_\ell=10^{-3}s_xs_ys_z\sum_v\ind[S(v)=\ell]\quad\text{mL}.
 \label{eq:appendix_volume}
\end{equation}
A visit-level volume sums the stations $\mathcal S_\ell$ whose mask contains the label, and the two-visit comparison takes the absolute change,
\begin{equation}
 V_\ell^{\mathrm{visit}}=\sum_{s\in\mathcal S_\ell}V_{\ell,s},\qquad
 \Delta V_\ell=V_\ell^{(2)}-V_\ell^{(1)},
 \label{eq:appendix_volume_visit}
\end{equation}
with relative change $100\,\Delta V_\ell/V_\ell^{(1)}$ percent recorded as $0$ when $V_\ell^{(1)}=0$. Per-slice counts additionally yield the slices containing the organ, their number, the maximal-voxel slice $z^\star$, and its cross-sectional area $s_xs_y\,N_\ell(z^\star)$ in mm$^2$. Volumes and percentages are rounded to one decimal place, and these rounded values are exactly what tool responses surface to the model. Visit-level totals sum per-station volumes without cross-station deduplication, i.e.\ we assume each station's mask covers a disjoint body segment, so an organ may span stations but no voxel is counted twice. Tools derive station and visit comparisons and render maximal-slice overlays. These measurements inherit segmentation and coverage error and do not define disease labels.

\paragraph{MedGemma report preparation.}
MedGemma-27B processes fixed MRI montages, orthogonal views, representative slices, and cardiac overviews offline. The preparation archive contains 42,717 findings, impression, and follow-up objects for 14,239 organ--participant pairs. A manifest indexes each report and its source images by participant, anatomy, sequence, and visit. Retrieval exposes at most 400 characters from the relevant summary field and masks ICD-like strings without rerunning the generator. Generation prompts, model and decoding versions, and manifest hashes specify provenance.

\paragraph{Participant partition.}
For participant $p$, define $H_s(p)=\operatorname{int}(\operatorname{MD5}(s\Vert p))\bmod100$. The preparation code reserves test participants with $H_{\mathrm{test}:s}(p)<100\rho$, $\rho=0.1$, then selects training participants using an independent $\mathrm{train}:s'$ salt. All derived questions, visits, roles, demonstrations, and memory sources must respect participant membership. 


\subsection{Benchmark landscape and evaluation paradigms}
\label{app:benchmark_landscape}

The distinction of interest is whether patient evidence is supplied as context or acquired through actions, and which aspects of the resulting interaction are evaluated. These are overlapping design choices, not a progression from non-agentic to uniformly superior benchmarks.

\begin{table}[htbp]
\centering\small
\begin{tabularx}{\textwidth}{@{}p{.17\textwidth}XXX@{}}
\toprule
Paradigm & Benchmark input & Model behavior & Evaluation target \\
\midrule
Medical QA & Question and task-provided text or images; external knowledge access where permitted. & Infer an answer from available context, optionally retrieving external knowledge. & Primarily answer correctness. \\
Interactive medical tasks & Task, patient state or records, and tools or simulated encounters. & Acquire information and perform task-dependent decisions or actions. & Task-specific success; some benchmarks also assess reasoning or information acquisition. \\
\ours{} & Clinical question and a participant-specific longitudinal evidence environment. & Select tools, accumulate text and images, synthesize findings, and answer. & Exact-set accuracy; logged retrieval and completion metrics. Independent grounding assessment is an additional validation protocol. \\
\bottomrule
\end{tabularx}
\caption{Illustrative input--interaction--evaluation distinctions. The categories are not mutually exclusive, and existing interactive benchmarks need not prescribe a tool sequence.}
\label{tab:benchmark_paradigms}
\end{table}

\paragraph{Question-answering resources.}
\citet{jin2020medqa} and \citet{pal2022medmcqa} organize medical examination questions in MedQA and MedMCQA; the former also study document retrieval, so medical QA should not be equated with universally fixed knowledge access. \citet{jin2019pubmedqa} condition PubMedQA answers on biomedical abstracts. \citet{lau2018vqarad}, \citet{liu2021slake}, and \citet{butsanets2026radimagenetvqa} emphasize image-conditioned questions in VQA-RAD, SLAKE, and RadImageNet-VQA, while \citet{zuo2025medxpertqa} include challenging clinical text and multimodal questions in MedXpertQA. Difficulty and clinical content alone do not specify whether the model must select observations from a patient-indexed environment.

\paragraph{Interactive and longitudinal tasks.}
\citet{jiang2025medagentbench} provide patient-specific FHIR tasks in MedAgentBench, and \citet{schmidgall2024agentclinic} support information gathering through AgentClinic encounters. \citet{qiu2025medrbench} evaluate clinical reasoning in MedR-Bench. \citet{vasilev2025mtbbench} combine longitudinal oncology decisions with multimodal tools in MTBBench, while \citet{lu2026clinenv} study information acquisition and decisions across inpatient stages in ClinEnv. These are substantive precedents for autonomous medical interaction. \ours{} studies a complementary setting: fixed clinical-label questions paired with follow-up imaging evidence whose measurements, views, and generated reports are separately accessible. Recording an evidence path enables its analysis; neither logging nor answer correctness establishes that the path contains a valid medical justification.

\paragraph{Cohort data and benchmark structure.}
\citet{bourigault2025ukbob} introduce UKBOB for large-scale MRI segmentation supervision and evaluation of segmentation generalization. \ours{} contributes a different task structure: participant cohort, longitudinal question instances, taxonomy, reference labels, partitions, evaluation rules, and an executable evidence interface. A data resource can also support benchmarks; the distinction is the task and evaluation protocol, not whether the underlying participant data can be redistributed. Baseline configurations and required measurements are specified in \appref{app:evaluation}; controlled baseline results remain to be supplied.

\subsection{Medical agents and forms of adaptation}
\label{app:medical_agents}

An agent can change its decisions through new observations, retrieved experience, or parameter updates. These mechanisms should not be conflated. In particular, inference-time orchestration can support autonomous actions and recovery without optimizing the underlying model, while supervised tool training already learns a policy even without RL. \Suppref{tab:medical_workflows,tab:medical_policy_learning} describe the adaptation mechanisms reported in the cited papers, not a binary judgment of whether a system is a genuine agent. Training an external segmentation model is also distinct from training the agent that chooses to invoke it.

\begin{table}[htbp]
\centering\small
\begin{tabularx}{\textwidth}{@{}p{.20\textwidth}XX@{}}
\toprule
System & Evidence use and interaction & Reported adaptation mechanism \\
\midrule
MDAgents~\citep{kim2024mdagents} & Medical problem solving with complexity-dependent individual or collaborative consultation. & Inference-time selection of roles and collaboration structure. \\
MedRAX~\citep{fallahpour2025medrax} & Iterative chest X-ray analysis using specialist tools. & Tool orchestration without additional agent training. \\
MedAgentSim~\citep{almansoori2025medagentsim} & Doctor--patient dialogue and requested tests or imaging results. & Retrieval of successful encounters and reflections from corrected failures, combined with multi-agent reasoning. \\
MedChain-Agent~\citep{liu2025medchain} & Information gathering and sequential referral, examination, diagnosis, and treatment. & Iterative agent feedback and MedCase-RAG with an expanding case database. \\
ClinicalAgent~\citep{yan2025clinicallab} & Department routing and staged consultation with laboratory and imaging reports. & Performance-informed clinician assignment and collaborative synthesis. \\
MedAgent-Pro~\citep{wang2026medagentpro} & Guideline-derived plans, patient-specific tool execution, and quantitative clinical indicators. & Hierarchical reasoning with stepwise evidence checks. \\
RadAgents~\citep{zhang2026radagents} & Radiologist-style inspection, measurements, and cross-tool verification. & Workflow guidance, local re-planning, and visual retrieval for conflict resolution. \\
\bottomrule
\end{tabularx}
\caption{Workflow and context-based adaptation in medical agents. These mechanisms can change actions within or across cases without an agent-policy gradient update. The classification concerns the described agent adaptation, not the training history of every component.}
\label{tab:medical_workflows}
\end{table}

\begin{table}[htbp]
\centering\small
\begin{tabularx}{\textwidth}{@{}p{.20\textwidth}XX@{}}
\toprule
System & Parameter-learning mechanism & Relation to the present study \\
\midrule
MMedAgent~\citep{li2024mmedagent} & Instruction tuning of tool calls and answers across medical tasks, including composed tool use. & Establishes medical tool-policy learning through supervised examples. \\
MedVR~\citep{jiang2026medvr} & RL with uncertainty-guided visual exploration and rollout-consensus supervision. & Learns image inspection; correctness gates its auxiliary tool reward. \\
Ophiuchus~\citep{jiang2026ophiuchus} & Tool-oriented SFT, reflection fine-tuning, and RL over interleaved visual interactions. & Establishes learning beyond demonstrations for medical visual tools. \\
MACRO~\citep{fan2026macro} & Supervised initialization and GRPO encouraging discovered composite-tool use. & Combines parameter learning with experience memory and tool-set expansion. \\
CASE & Demonstration initialization followed by learner-directed RL with an insight-focused, detached self-distillation penalty. & Studies acquisition across longitudinal patient sources and synthesis feedback combined with outcome and rubric rewards before group normalization. \\
\bottomrule
\end{tabularx}
\caption{Learned medical tool policies. The comparison identifies differences in supervision and task structure, not evidence of comparative performance on a shared benchmark.}
\label{tab:medical_policy_learning}
\end{table}

\paragraph{Interactive benchmark precedents.}
\citet{liu2025medchain} organize 12,163 MedChain cases into five sequential stages, including medical-image examination; MedChain is neither a static-QA benchmark nor a text-only predecessor. \citet{chiu2025vivabench} begin VivaBench with limited case information and let the candidate request history, examinations, and investigations before diagnosis. \citet{yan2025clinicallab} combine a real-case benchmark spanning 24 departments and 150 diseases with ClinicalAgent; its imaging stage uses textual radiology reports. \citet{zhu2026medagentboard} compare collaboration with single-model and conventional approaches across medical QA and EHR prediction, finding that collaboration does not consistently dominate. These results motivate evaluating each system component rather than treating additional agents as intrinsically beneficial.

\paragraph{Scope of the methodological distinction.}
Clinical evidence gathering, multimodal grounding, compact tool-using models, and medical agentic RL all have precedents. MedVR also conditions an auxiliary tool reward on answer correctness, so outcome-linked feedback alone is not a distinguishing claim. The present design evaluates the learner's post-retrieval synthesis under ordinary and training-only privileged contexts on the same realized evidence history. Its detached insight penalty enters the joint correctness and rubric reward before group normalization, rather than prescribing a retrieved case, a teacher tool sequence, or a consensus image region. This specifies a testable design difference, not a demonstrated advantage. Matched ablations are required to establish its value beyond existing RL and distillation mechanisms. Cross-system comparisons additionally require compatible observation interfaces and explicit memory-freezing rules; a shared backbone alone does not equalize evidence access.

\paragraph{Broader learning foundations.}
\citet{yao2023react} interleave reasoning with environment actions in ReAct, \citet{schick2023toolformer} study learned API use in Toolformer, \citet{li2023llavamed} adapt multimodal instruction following to biomedicine in LLaVA-Med, and \citet{gu2024minillm} study reverse-KL distillation into compact language models in MiniLLM. These foundations motivate the learning components without making their combination in an agent sufficient evidence of novelty.

\section{Tool Environment and Multimodal Interaction}
\label{app:environment}

\paragraph{Role configurations.}
During recording, Anatomist and Radiologist bind their respective three-tool registries, while Consultant binds all ten tools. The deployed mixed-role RL job initializes the ten-tool union registry for every rollout and relies on each role prompt to specify the intended workflow. Consultant directly requests the underlying evidence sources; it does not launch independently trained specialist agents or receive an ensemble vote. Skill retrieval is an auxiliary memory mechanism, rather than a fourth patient-imaging modality.

\begin{table}[htbp]
\centering\small
\begin{tabularx}{\textwidth}{@{}p{.18\textwidth}p{.35\textwidth}X>{\centering\arraybackslash}p{.09\textwidth}@{}}
\toprule
Configuration & Tools & Returned evidence & Call budget \\
\midrule
Anatomist & \artifact{vista3d_segment}; \artifact{vista3d_compare_stations}; \artifact{vista3d_compare_visits} & Organ/station measurements; longitudinal changes; segmentation overlays. & 4 \\
Radiologist & \artifact{list_available_images}; \artifact{view_image}; \artifact{compare_visits} & Acquisition availability and metadata; selected or paired MRI views. & 6 \\
Consultant & All six above; \artifact{list_available_reports}; \artifact{get_report}; \artifact{recall_skills}; \artifact{write_skill} & Combined patient evidence, generated report claims, and reusable textual guidance. & 12 \\
\bottomrule
\end{tabularx}
\caption{Executable registries. Anatomist and Radiologist each have three tools; Consultant has ten. Tool availability does not prescribe the policy's action sequence.}
\label{tab:appendix_registry}
\end{table}

\paragraph{Offline evidence preparation.}
Evidence preparation is completed before RL so that policy rollouts do not wait for segmentation or report-generation inference. First, Vista3D produces whole-body Dixon segmentation masks and MedGemma generates report objects from MRI renderings, as detailed in \appref{app:data}. Second, cache construction derives organ volumes from the masks using \suppref{eq:appendix_volume}, computes station-level quantities and absolute or relative changes across visits, and stores the resulting text with rendered overlays. The imaging cache stores acquisition metadata, MRI views, and paired-visit comparisons. Report text has its ICD-10 diagnostic codes removed before caching so it cannot leak the reference answer, and reports are indexed with their source montages for retrieval. This separates the GPU requirements of evidence-generating models from policy training and amortizes preparation over repeated patient interactions.

\paragraph{Mock-tool execution.}
The recording and evaluation specialists read the participant's MRI and segmentation sources and render each requested observation (volume, comparison, view, or frame) on demand. For instance, \emph{cardiac short-axis, frame 14} or \emph{liver, baseline versus follow-up} resolves against the stored sources into a rendered view, an organ volume, or a longitudinal comparison. During RL, mock tools resolve each policy-generated request against precomputed observations rather than rerunning source-model inference, keyed by the participant, tool, and requested arguments; missing entries return visible error or unavailable-data observations. Vista3D and imaging adapters read cached text and rendered images; the report adapter locates pre-generated JSON through a manifest and formats it on request. Loading, resizing, and assembling images remain online operations. These adapters serve observations for the policy's current requests; they do not supply the next teacher action. SFT subsequently consumes recorded conversations, whereas RL acquires observations through fresh policy-generated calls.

The action semantics are fixed by the tool schemas: the policy selects the tool, the organ (resolved to a segmentation label through the Vista3D organ-index map), the sequence and sub-modality, the visit, and the frame or slice to view or compare. In recording and evaluation, the specialists render each requested view, overlay, or comparison from the participant's source volumes on demand, so distinct frame requests yield distinct observations; the RL replay adapter serves the same observations from precomputed stores keyed by the request, for efficiency. Consultant's image and volume tools follow the same request semantics in the LangGraph runtime. Tool schemas, field formatting, and missing-entry behavior are part of the environment specification.

\paragraph{Skill memory (Consultant).}
Beyond the imaging and report tools, the Consultant alone has a \emph{skill-memory} tool pair that carries reusable textual guidance across cases. Offline, \artifact{build_skill_store} assembles a frozen store from seed skills together with skills harvested and clustered from recorded teacher trajectories, keeping one reusable skill per organ and indexing it by organ and keywords. At inference the Consultant calls \artifact{recall_skills} with the current organ and keywords, which deterministically returns the top-$k$ matching skills and injects their formatted guidance into the context before it commits to an answer, functionally a second, experience-informed reading of the same case. Formally, let the frozen store be $\mathcal{S}=\{s_m\}_{m=1}^{M}$, each skill tagged with an organ $o_m$, a keyword set $W_m$, and a failure-reflection flag $f_m\in\{0,1\}$. For a query organ $o$ and keyword set $W$ taken from the current case, \artifact{recall_skills} scores
\begin{equation}
\mathrm{score}(s_m)=2\,\ind[o_m{=}o]+0.5\,\ind[o_m{=}\mathrm{general}]+0.5\,|W\cap W_m|+0.3\,f_m,
\label{eq:skill_recall}
\end{equation}
and returns the top-$k$ ($k{=}3$) skills with $\mathrm{score}>0$, ordered by score with ties broken by store index. Their formatted guidance $g$ is prepended to the context, so the Consultant answers from $\pi_{\vtheta}(\cdot\mid x,e,g)$ rather than $\pi_{\vtheta}(\cdot\mid x,e)$ for the evidence $e$ gathered so far, and $\mathcal{S}$ is left unchanged. \artifact{write_skill} runs only after the first-pass answer, in a post-hoc review stage that elicits a short success summary when the answer was correct and a failure reflection when it was wrong (the latter tagged \artifact{failure_reflection} and rewarded in retrieval scoring, \cref{eq:skill_recall}); in the frozen setting it merely acknowledges the request. Skill memory is therefore an auxiliary retrieval store, not a fourth imaging modality and not a gradient-bearing component.

Skill memory is frozen during RL but run-scoped at evaluation. In the RL replay adapter (\artifact{LIVE_SKILL=0}), \artifact{recall_skills} deterministically queries a static store and \artifact{write_skill} acknowledges requests without persistence, keeping the environment fixed (\cref{sec:onpolicy_environment}). Held-out evaluation instead sets \artifact{LIVE_SKILL=1}: within one run \artifact{write_skill} appends to the in-memory store, so later questions can recall skills written earlier in the same run, but the store is never written back to disk (live within a run, frozen across runs), and its base contents exclude test participants.

\paragraph{Organ-name resolution and segmentation indexing.}
The vista3d tools accept free-text organ names and resolve them in two fixed steps: a hardcoded index-to-label map over 32 organs (e.g.\ 1=liver, 115=heart, 151=left\_ventricle) with its generated inverse, and an alias table that folds colloquial variants (hepatic, spine, kidneys, gall bladder) into canonical names after lowercasing and whitespace normalization. Lookup tries the alias table, then the canonical map; an unresolvable name returns a visible ``unknown organ'' observation rather than raising, so the policy can retry with another name. Because a segmentation is a per-voxel integer label array, the organ mask is the equality test against the resolved index: voxel count times per-voxel volume gives the organ volume in millilitres, per-station counts locate the dominant slice, and the same equality mask renders the overlay. The policy never sees these format details; the complete candidate organ list for a participant (up to 20, enumerated from that participant's segmentation) is supplied in the prompt at recording and evaluation. In RL the mock tool reuses the identical name-translation step and then reads a precomputed cache keyed by participant, tool, and canonical name, so recording, evaluation, and training agree on resolution behavior.

\paragraph{Observation serialization.}
In LangGraph recording, assistant tool requests produce text tool messages followed by an image-bearing user message encoded as data URIs. Saved conversations retain image references for subsequent loading. In RL, the patched tool loop constructs a tool message with one placeholder for each returned image, followed by its text. The processor receives images in the same order and produces pixel values and grid descriptors. The Qwen3-VL forward encodes these pixels and uses masked scatter to replace the corresponding placeholder embeddings; intermediate visual features also enter the decoder through its DeepStack pathway~\citep{bai2025qwen3vl}. Newly rendered observation positions are appended with zero action mask, while the cumulative image list is retained for later multimodal forwards. The next assistant turn therefore conditions on accumulated text and visual evidence. Only assistant-generated positions contribute to the policy-gradient action likelihood.

For a processed image grid $(t,h,w)$ with spatial merge factor $m$, the number of visual embedding positions is $N_{\mathrm{vis}}=thw/m^2$. A 512-pixel square image with 16-pixel patches and $m=2$ gives 256 positions. Actual counts depend on processor geometry; the placeholder token is a position marker, not a discrete encoding of image content. The evidence cache stores observations, not learned visual embeddings: vision-encoder and merger features are recomputed in model forwards from the retained pixel inputs. Both OPSD conditions receive the same images. During SFT and actor optimization, assistant-token losses propagate through the visual features to update the visual modules; OPSD scoring supplies the detached reward described in \cref{sec:opsd,sec:reward}.

The supplied Vista3D and MedPanel adapters load at most six cached images, stack them vertically, and resize the composite to one $512\times512$ PIL image. Their per-call output is thus zero or one image, potentially containing multiple views. The report adapter selects at most two montage/overview images and preserves aspect ratio while limiting the longest edge to 512. Its current return values are base64 strings rather than PIL objects; the complete installed schema and processor are needed to establish successful decoding. This distinction matters when comparing rollout and evaluation visual exposure.

\paragraph{Budgets and final synthesis.}
The Consultant LangGraph tool node enforces a 12-call budget on the first pass. The specialist prompts request at most four calls for Radiologist and one to three for Anatomist, but the supplied specialist nodes do not implement the same cumulative hard cap. The canonical RL wrapper allows ten assistant turns and a 6,144-token continuation; its local tool-loop patch additionally applies a per-turn parallel-call limit and checks continuation length. These are different resources.

During retrieval, assistant messages can contain brief intermediate reasoning. The final message is requested to contain five phases: selection/planning, evidence analysis, comparison or verification, insight synthesis, and answer. Consultant specifically uses multi-source review and doubt/verification as its middle phases. This structured final synthesis summarizes evidence acquired during interaction; it is not a fixed externally annotated sequence of retrieval actions.

\paragraph{AgentLoop execution.}
The asynchronous engine begins in \textsc{Pending}, where it renders the prompt and tool schemas. \textsc{Generating} obtains one model response, marks its tokens as actions, and parses function calls. A parsed call enters \textsc{ProcessingTools}; the engine executes at most the configured parallel-call allowance, serializes each returned observation, updates the accumulated images, and resumes \textsc{Generating}. Absence of a call, exhaustion of assistant turns, or the continuation limit enters \textsc{Terminated}. The generic \textsc{Interacting} state can append an external user response, but is not used by our evidence tools. The implementation adds JSON-parser fallback, Qwen3-VL multimodal-position handling, one placeholder per returned image, and empty-image handling for text-only microbatches. Scoring failure handling is detailed in \appref{app:implementation}.

\paragraph{Rollout loop and evaluation harness.}
RL interaction is driven entirely by the project's fork of the \textsc{verl} AgentLoop: the rendered system and user messages plus tool schemas are sent to vLLM~\citep{kwon2023vllm}; assistant text is parsed for XML-form tool calls; the six- or ten-tool mock registry is instantiated per participant and executed (several calls may run in one turn); returned text and images are serialized back as tool messages, re-encoded, and fed to the next generation step, with tool tokens masked out of the policy gradient and images re-passed through the vision encoder. The loop ends when the policy emits no further call or a cap is reached (6{,}144 continuation tokens, 10 assistant turns, 2{,}048 tokens per tool response); a batch then supplies eight trajectories per question to group-relative optimization. Evaluation reuses the same interface contract through a different implementation: an OpenAI-compatible vLLM server whose server-side parser extracts tool calls, wrapped in a two-node LangGraph loop over the real tool implementations (vista3d and medpanel read the segmentation and rendering directories; reports read the cache). The harnesses therefore differ in two confined ways (server-side versus in-loop call parsing and real versus mock tool execution) while sharing tool names, prompts, and caches, so recorded trajectories feed RL unchanged and evaluation inputs and outputs remain commensurate with training.

\paragraph{Token classes of a rollout.}
\Cref{fig:trajectory_tokens} summarizes which parts of a rollout carry gradient. Gray context tokens (system prompt and question) are fixed and never trained; blue model-generated tokens (tool choice, arguments, the reasoning phases, and the final synthesis) are the only spans RL optimizes; orange tool-returned text enters the context as evidence only; and pink image placeholders mark returned images, which the vision encoder re-embeds at each generation step.

\begin{figure}[htbp]
\centering
\includegraphics[width=\linewidth]{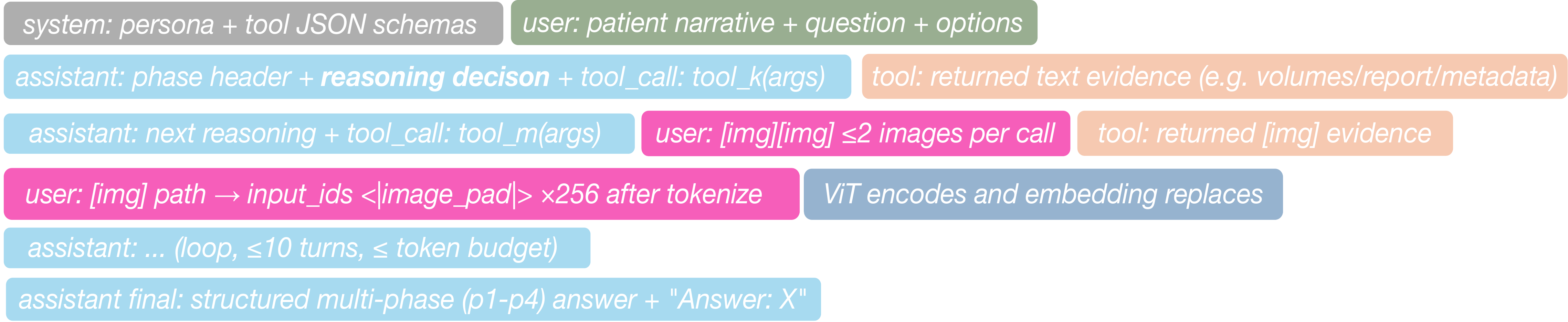}
\caption{\textbf{Token composition of a generic trajectory.} Gray, fixed context (system prompt and question); blue, model-generated tokens, the only spans RL trains (tool choice, arguments, reasoning, synthesis); orange, tool-returned text (evidence, input only); pink, image placeholders for returned images, re-passed through the vision encoder.}
\label{fig:trajectory_tokens}
\end{figure}

\subsection{Example trajectory}
\label{app:example_trajectory}

To make the interaction contract of \cref{sec:interaction} concrete, \cref{tab:example_trajectory} renders one complete Consultant rollout for participant 2306330 in the recorded trajectory format and in full: the system prompt, all five reasoning phases, all nine tool calls with their arguments and returned observations, and the per-condition verification are shown without elision. Tool rows read \texttt{name(arguments)} $\to$ observation, with the run's tool-call identifier and the participant's cached image paths in small type.

{\small
\setlength{\tabcolsep}{4pt}
\begin{longtable}{@{}p{0.14\textwidth}p{0.82\textwidth}@{}}
\caption{A complete Consultant Agent trajectory on a multi-select hepatic-diagnosis question for participant 2306330, rendered in the recorded trajectory format. The agent recalls a liver skill, cross-checks fat-sensitive and high-resolution imaging, longitudinal volumetry, and draft reports, verifies or rejects each candidate condition in P3, and commits to steatosis plus a benign lesion rather than malignancy.}
\label{tab:example_trajectory}\\
\toprule
\textbf{Step} & \textbf{Content} \\
\midrule
\endfirsthead
\multicolumn{2}{@{}l}{\emph{\cref{tab:example_trajectory} continued}}\\
\toprule
\textbf{Step} & \textbf{Content} \\
\midrule
\endhead
\midrule
\multicolumn{2}{r@{}}{\emph{continued on next page}}\\
\endfoot
\bottomrule
\endlastfoot
Question & \emph{Based on longitudinal medical imaging and clinical records, which hepatic conditions are supported by the available evidence? Select all that apply.} Options: (A)~hepatic steatosis (fatty liver disease); (B)~benign hepatic lesion; (C)~primary liver malignancy; (D)~liver cirrhosis; (E)~none of the above. \\
System & You are a clinical consultant agent, the final arbitration over heterogeneous clinical evidence rather than any single source. You command the full ten-tool Consultant registry of \cref{tab:appendix_registry}. To inspect imaging you may list available acquisitions (\texttt{list\_available\_images}), view a selected slice or montage (\texttt{view\_image}), and compare matched views across visits (\texttt{compare\_visits}); for quantitative structure you may segment an organ (\texttt{vista3d\_segment}) and compare volumes across stations (\texttt{vista3d\_compare\_stations}) or across visits (\texttt{vista3d\_compare\_visits}); for textual claims you may list the generated MedGemma reports (\texttt{list\_available\_reports}) and read one via \texttt{get\_report(organ, report\_type)} with \texttt{report\_type} in findings, impression, or followup for liver, pancreas, or wholebody, treating these draft reports as claims to verify because they may hallucinate; and for accumulated experience you may recall reusable skills (\texttt{recall\_skills}), recording new ones (\texttt{write\_skill}) only in the later review stage.\newline Required workflow: (1)~call \texttt{recall\_skills} for the implicated organ(s) first; (2)~review evidence from multiple sources (images, volumes, reports), using each for what it is best at; (3)~actively doubt---mark key claims [\textsc{verified}], [\textsc{suspicious}], or [\textsc{contradicted}] and verify every decision-relevant suspicious claim with a direct tool call; (4)~reflect---build the causal chain and explicitly consider what could mislead you here.\newline When done calling tools, the final message (no tool calls) must contain all five headers in order, under 2000 characters: Phase~1 Skill Recall \& Review Plan; Phase~2 Multi-source Review; Phase~3 Doubt \& Verification; Phase~4 Reflective Insight (the causal chain from verified evidence toward the answer, plus what could have misled and why it was rejected); Phase~5 Final Answer.\newline Rules: maximum 12 tool calls (hard-enforced), and once about 9 are used stop calling tools and report with the evidence at hand; prefer high-value calls (\texttt{recall\_skills} first, \texttt{vista3d\_compare\_visits} for change questions, \texttt{get\_report} for report evidence, \texttt{view\_image} only for targeted verification) and skip sources irrelevant to this question; never copy a report conclusion without verification when a quantitative tool can check it; the \texttt{Answer:} line in Phase~5 is mandatory and the answer is extracted only from it; do not call \texttt{write\_skill} at this stage (skill recording happens in a later review stage). \\
User (patient info) & \emph{Patient ID: 2306330.}  This female participant was born in 1955 and was 64 years of age at the time of the most recent assessment (Instance~3). Longitudinal follow-up data were available from two assessment visits: baseline assessment (Instance~2) was conducted on 2017-08-16 at age~61 years, and follow-up assessment (Instance~3) on 2019-12-16 at age~64 years, a follow-up interval of 2.3 years. Regarding housing tenure, the participant transitioned from own with a mortgage (Instance~0) to own outright (Instance~3), indicating a change in housing stability, and resided in a house or bungalow. Household composition at Instance~3 comprised 2 resident(s); the household reported ownership of three vehicle(s), a change from two at Instance~2. The reported average total household income before tax at Instance~3 was less than 18,000, compared to 52,000--100,000 at Instance~2, suggesting a potential decrease in household economic resources. Regarding occupational characteristics, the participant reported commuting distance of nan units from home to workplace (changed from 13 units at Instance~2) and employment duration of nan units in their primary occupation. This comprehensive sociodemographic profile provides essential context for interpreting clinical and imaging findings, as socioeconomic factors are well-established determinants of health outcomes and disease risk in population-based cohort studies. \\
P1 Plan & The case contains liver abnormalities including a reported focal lesion. I begin by recalling liver-specific skills, then select the sources best suited to each sub-question. The main clinical challenge is distinguishing an incidental or benign hepatic lesion from primary liver malignancy. I will not directly convert an imaging abnormality into a cancer diagnosis. Instead, I will characterize lesion morphology, evaluate longitudinal stability, quantify liver structural changes, and verify whether clinical evidence supports malignancy. \\
Tool call 1 & \texttt{recall\_skills(organ=liver; keywords=hepatic lesion characterization, hepatocellular carcinoma, hepatic steatosis, malignancy verification)} $\to$ Liver lesion assessment requires correlation between morphology, interval change, and clinical confirmation. Stable lesions without invasive characteristics are less suggestive of malignancy. Imaging suspicion alone should not be interpreted as confirmed cancer. {\scriptsize\texttt{call\_6b57153c6e84494a94af5fa8}} \\
Tool call 2 & \texttt{list\_available\_images()} $\to$ Available imaging (Visit~1/Visit~2 pairs): liver---IDEAL multiecho and IDEAL temporal; whole body---Dixon BH~17s~W (stations 4/8/12/20/24); pancreas---gradient-echo multiecho. {\scriptsize\texttt{call\_26553c9a231a453ea5b4c90}} \\
P2 Review & Because the question involves both diffuse liver disease and focal lesion characterization, I will first select liver-specific MRI sequences. Fat-sensitive sequences are required to evaluate steatotic change, while station-level liver slices (Dixon 8\_W) are required to assess lesion morphology and continuity. Longitudinal comparison is necessary to determine whether the lesion demonstrates progressive behavior. \\
Tool call 3 & \texttt{compare\_visits(modality=IDEAL\_multiecho; visit\_comparison=V1\_vs\_V2)} $\to$ Longitudinal IDEAL multiecho comparison (V1 vs V2 montages) shows diffuse fat-related signal increase and a persistent focal hepatic lesion without substantial interval enlargement or aggressive transformation. {\scriptsize\texttt{call\_3428f8072e794336b512cff6}} {\scriptsize\path{2306330/IDEAL__IDEAL_multiecho__montage_v1.png}, \path{2306330/IDEAL__IDEAL_multiecho__montage_v2.png}}\newline\begin{minipage}[t]{0.48\linewidth}\centering\includegraphics[width=\linewidth]{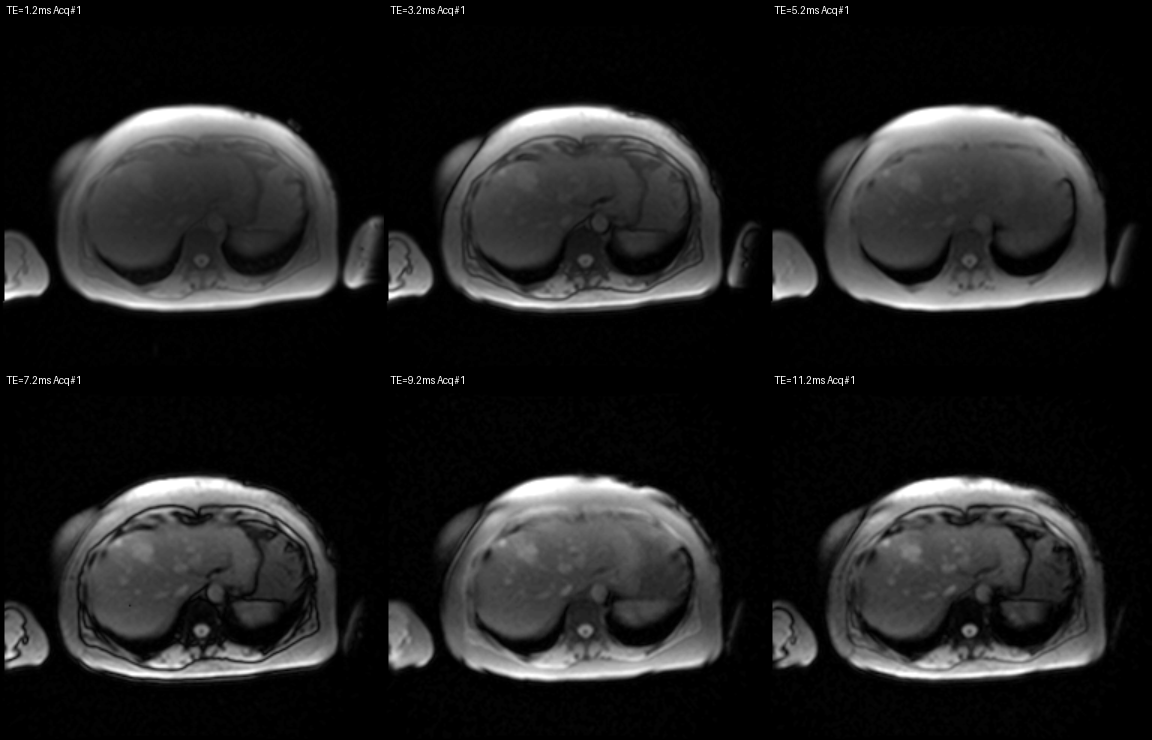}\\{\scriptsize\path{IDEAL__IDEAL_multiecho__montage_v1.png}}\end{minipage}\hspace{8pt}\begin{minipage}[t]{0.48\linewidth}\centering\includegraphics[width=\linewidth]{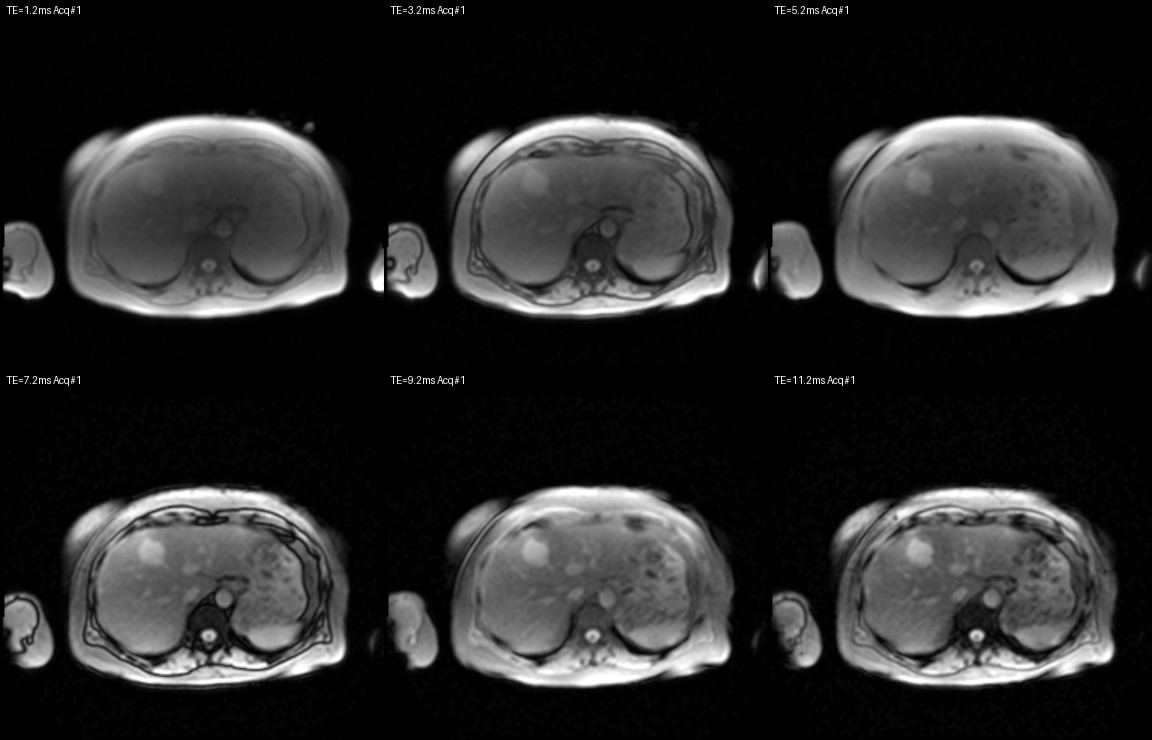}\\{\scriptsize\path{IDEAL__IDEAL_multiecho__montage_v2.png}}\end{minipage} \\
Tool call 4 & \texttt{view\_image(modality=dixon\_8\_W; visit=V1; slice\_idx=[10,22])} $\to$ Multi-slice review of Dixon 8\_W (V1) at slices~10 and~22 shows a well-circumscribed focal lesion at slice~22 with smooth contour; slice~10 shows no additional suspicious lesion, and there is no vascular invasion or infiltrative growth pattern. {\scriptsize\texttt{call\_c5ca7b264e914be793daf890}} {\scriptsize\path{2306330/dixon_8_W_v1_slice10.png}, \path{2306330/dixon_8_W_v1_slice22.png}}\newline\begin{minipage}[t]{0.30\linewidth}\centering\includegraphics[width=\linewidth]{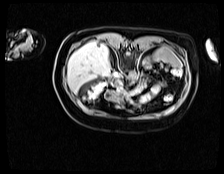}\\{\scriptsize\path{dixon_8_W_v1_slice10.png}}\end{minipage}\hspace{8pt}\begin{minipage}[t]{0.30\linewidth}\centering\includegraphics[width=\linewidth]{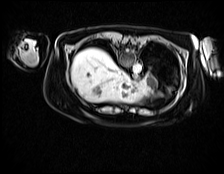}\\{\scriptsize\path{dixon_8_W_v1_slice22.png}}\end{minipage} \\
Tool call 5 & \texttt{compare\_visits(modality=dixon\_8\_W; slice\_idx=22; visit\_comparison=V1\_vs\_V2)} $\to$ Side-by-side Dixon 8\_W slice~22 (V1 vs V2) shows the focal lesion stable in size and morphology, with increased diffuse hepatic signal alteration consistent with steatosis progression rather than malignant growth. {\scriptsize\texttt{call\_8fa2d91c92b44aa6b7f0e31d}} {\scriptsize\path{2306330/dixon_8_W_compare_slice22.png}}\newline\includegraphics[width=0.55\linewidth]{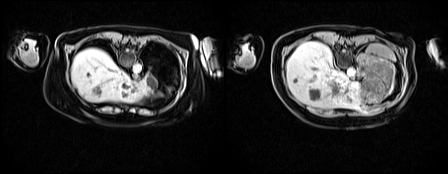}\newline{\scriptsize\path{dixon_8_W_compare_slice22.png}} \\
Tool call 6 & \texttt{vista3d\_compare\_visits(organ=liver; slice\_idx=22; visit\_comparison=V1\_vs\_V2))} $\to$ Longitudinal liver volumetry on Dixon 8\_W demonstrates an increase from $1150.8$\,ml (V1) to $1263.4$\,ml (V2), $+112.6$\,ml ($+9.8\%$); the V1/V2 segmentation overlays confirm hepatomegaly without independently indicating malignant progression. {\scriptsize\texttt{call\_95e653587c3e47938842c17a}} {\scriptsize\path{2306330/v1_liver_dixon_8_W_overlay.png}, \path{2306330/v2_liver_dixon_8_W_overlay.png}}\newline\begin{minipage}[t]{0.30\linewidth}\centering\includegraphics[width=\linewidth]{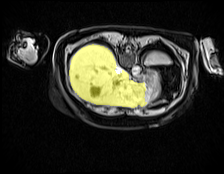}\\{\scriptsize\path{v1_liver_dixon_8_W_overlay.png}}\end{minipage}\hspace{8pt}\begin{minipage}[t]{0.30\linewidth}\centering\includegraphics[width=\linewidth]{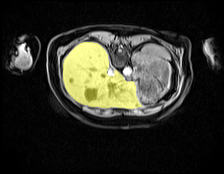}\\{\scriptsize\path{v2_liver_dixon_8_W_overlay.png}}\end{minipage} \\
Tool call 7 & \texttt{get\_report(organ=liver; report\_type=impression)} $\to$ [MedGemma impression $\mid$ liver $\mid$ IDEAL] Liver MRI impression reports diffuse hepatic steatotic change, hepatomegaly, and a focal low-density hepatic lesion requiring further characterization. The report does not provide definitive confirmation of primary liver malignancy. {\scriptsize\texttt{call\_d55d1c92060147e6b1cff24f}} \\
Tool call 8 & \texttt{get\_report(organ=wholebody; report\_type=findings)} $\to$ [MedGemma findings $\mid$ whole\_body $\mid$ Dixon] Clinical record review identifies no confirmed liver cancer diagnosis, oncological treatment history, or documented metastatic disease; the whole-body Dixon station-2 montage shows no additional suspicious lesion or metastatic distribution. {\scriptsize\texttt{call\_e4ce774d91074eeab48cc229}} {\scriptsize\path{2306330/Dixon__Dixon_BH_17s_W_station2__montage_7.png}}\newline\includegraphics[width=0.36\linewidth]{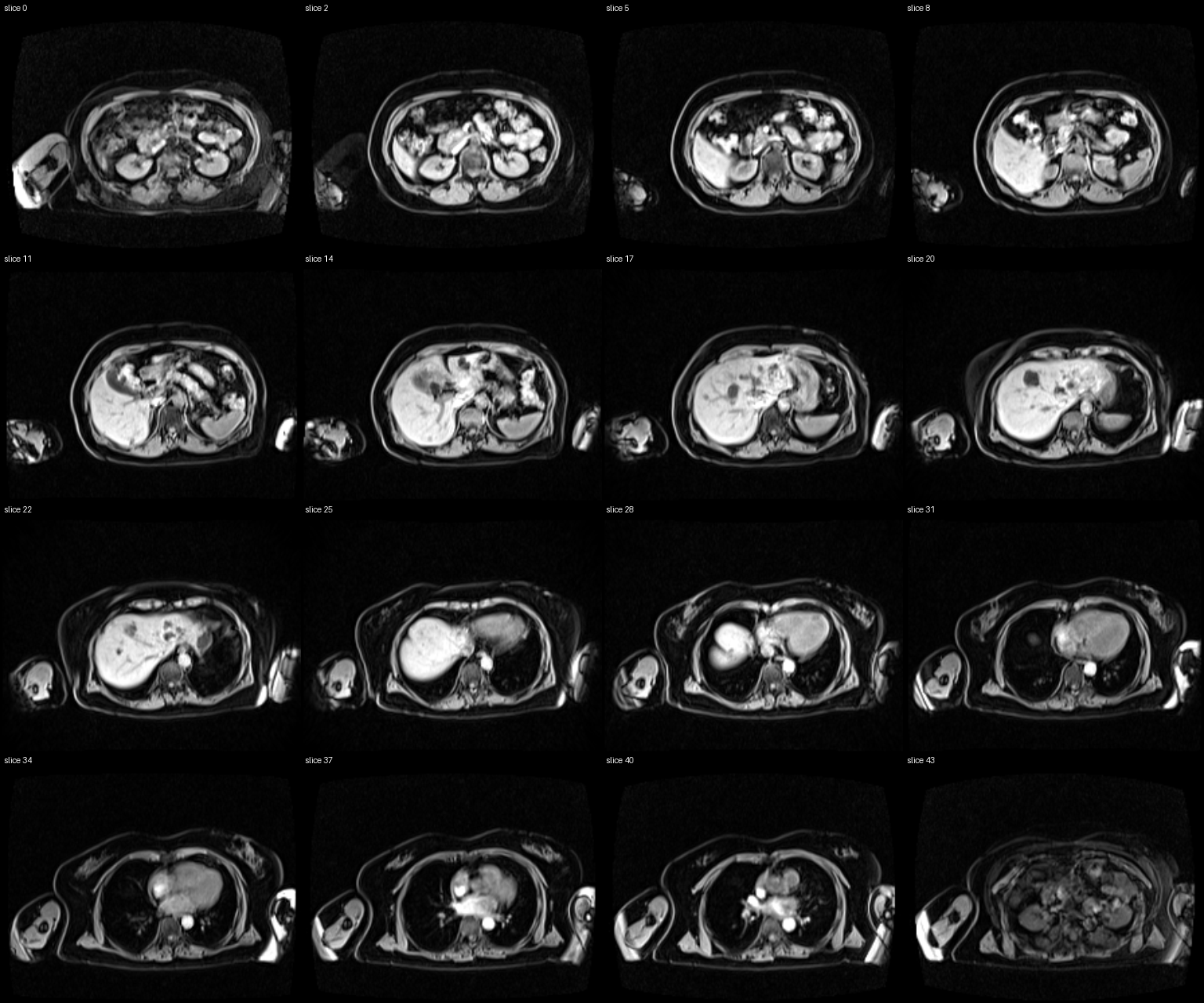}\newline{\scriptsize\path{Dixon__Dixon_BH_17s_W_station2__montage_7.png}} \\
P3 Verify & Hepatic steatosis: \textsc{verified}---diffuse hepatic signal abnormality on fat-sensitive MRI together with liver impression findings supports steatotic change.
\newline Benign hepatic lesion: \textsc{supported}---stable appearance across visits, smooth morphology, absence of invasive imaging characteristics.
\newline Primary liver malignancy: \textsc{not confirmed}---requires further characterization, but no documented cancer diagnosis, pathological confirmation, vascular invasion, or aggressive interval progression.
\newline Liver cirrhosis: \textsc{contradicted}---imaging does not demonstrate cirrhotic morphology or supporting clinical evidence.
\newline None of the above: \textsc{rejected}---multiple hepatic abnormalities are supported by independent evidence sources. \\
P4 Insight & The critical decision is not whether a liver lesion exists, but whether it represents malignant disease. Multi-source evidence shows two distinct hepatic findings: diffuse hepatic steatosis and a persistent focal lesion. The lesion remains morphologically stable across longitudinal imaging, without vascular invasion, infiltrative growth, or confirmed oncological history. Therefore, the evidence supports a benign hepatic lesion requiring surveillance rather than primary liver malignancy. My final diagnosis rests on evidence consistency across imaging, quantitative assessment, and clinical records rather than a single suspicious finding. \\
P5 Answer & \emph{Reasoning:} Based on the evidence verification and clinical insight in Phase~4, I conclude that the available evidence supports hepatic steatosis and a benign hepatic lesion. The focal hepatic abnormality lacks sufficient evidence for primary liver malignancy; therefore, cancer is not selected.\newline \emph{Answer:} A. Hepatic steatosis (fatty liver disease); B. Benign hepatic lesion. \\
Visual context & Each vision call's returned rendering is displayed directly beneath its row with its full filename; every image enters the context as placeholder-aligned visual tokens. \\
Reflection (review) & \emph{Summary:} the decision turned on separating a stable, well-circumscribed focal lesion without invasive features from diffuse steatosis; quantitative volumetry ($+9.8\%$) supported hepatomegaly but not malignant progression, and the whole-body record showed no confirmed cancer diagnosis. \emph{Reflection:} the draft liver impression's ``focal lesion requiring further characterization'' was the main overcall risk; treating hedged report language as suspicion rather than diagnosis, and confirming lesion stability across visits and slices, avoided upgrading to malignancy. This lesson is recorded as a reusable skill in the post-hoc review stage. \\
Tool call 9 (review) & \texttt{write\_skill}(\texttt{patient\_type}=hepatic lesion with steatosis; \texttt{guidance}=verify lesion stability across visits and invasive features before assigning malignancy, and treat hedged report language as suspicion not diagnosis; \texttt{pitfalls}=do not equate ``requiring further characterization'' with confirmed cancer, and do not let steatosis-related heterogeneity mimic malignancy; \texttt{organ}=liver; \texttt{keywords}=hepatic lesion, steatosis, malignancy verification) $\to$ Skill recorded for patient type ``hepatic lesion with steatosis''; it will be available via \texttt{recall\_skills} for similar future patients. {\scriptsize\texttt{call\_7a9c2e4f6b8d0a1c3e5f7b9d}} \\
\end{longtable}
}

The reference answer for this question is $\{A,B\}$ (hepatic steatosis and benign hepatic lesion); the P5 output above matches it. This reference is supplied for the reader and is never visible to the policy during the rollout.

\section{Training Configuration and Data Materialization}
\label{app:training}

\paragraph{Recorded demonstrations.}
Qwen3.8-Max~\citep{qwen2026qwen38max} is invoked through the role-specific LangGraph recorder with the question, patient narrative, options, tool schema, and participant identifier, but without the reference answer. Each assistant tool request is executed before the next model call; textual results and base64-encoded images enter the running context, while saved JSONL records replace image payloads with resolvable paths. The raw record preserves the system/user messages, every assistant tool call, tool responses, image references, and the terminal synthesis. A dense companion record stores parsed phases, predicted and reference answers, correctness, tool-image metadata, and errors. The pre-selection inventory contains 46,632 Anatomist, 50,401 Radiologist, and 50,401 Consultant episodes, for 147,434 recordings in total. Training membership is determined by participant exclusion, trajectory availability, and the screening rule in \cref{eq:rsft}; raw recording counts are not retained-training counts. For Consultant, any ground-truth verdict, post-answer review reflection, and skill-writing calls are appended only after the first-pass answer and are not part of the answer-blind acquisition path.

\paragraph{Trajectory split, test benchmark, and compute.}
The three agents answer the \emph{same} diagnosis-derived questions and differ only in trajectory and tool registry, so the inventory above holds one episode per question for the Radiologist and Consultant and fewer for the Anatomist. A participant-disjoint manifest bucket, keyed by an MD5 hash of the participant identifier, reserves the held-out test split ($\rho{=}0.1$) and assigns the remaining episodes to RSFT and agentic RL, a $40/50/10$ split in which $40\%$ of episodes train RSFT, $50\%$ form the agentic-RL pool from which $4$--$6$ questions per participant are sampled as RL prompts, and the final $10\%$ is held out for testing. The test set is thus $4{,}850$ questions from $455$ participants; answering it under all three agents yields $14{,}135$ agent--question evaluations, with the Radiologist and Consultant each covering all $4{,}850$ questions ($455$ participants) and the Anatomist $4{,}435$ questions ($418$ participants). The Anatomist shortfall arises because the remaining participants lack whole-body MRI at one or both visits, so Vista3D cannot segment their Dixon sequence; this is the coverage asymmetry shown in panels~1 and~5 of \cref{fig:benchmark_overview}. RSFT and agentic RL together consume roughly $2{,}000$ H100-80GB GPU-hours on a single $8\times$H100 node.

\paragraph{Offline evidence-linkage filtering.}
The rejection-sampling recipe scans existing records once. It requires at least one assistant tool call and a parseable Phase~4 insight in the final first-pass response. Kimi K2.6~\citep{moonshot2026kimik26} scores the written chain from evidence through insight to the Phase-5 answer on the fixed four-criterion rubric of \cref{tab:appendix_rubric}, which weights evidence identification (0.3), logical reasoning (0.3), answer support (0.2), and completeness (0.2), and retains trajectories whose aggregate score is at least 0.4. It produces a retained participant/question manifest; it does not regenerate rejected trajectories or impose an independent answer-correctness condition. These requirements select complete trajectories and do not define an SFT token mask. The supplied training code consumes such manifests, but the complete filter script, accepted manifests, and raw responses are not included in this snapshot. Candidate, retained, failure-fallback, and per-role counts therefore remain required run artifacts rather than quantities inferred from cache sizes.

\paragraph{Checkpoint handoff.}
The provided chain initializes a mixed-role SFT continuation from checkpoint 1515, fits one additional epoch using the extension manifests, and launches agentic RL from the final continuation model. If the final save is unavailable, the wrapper can select the latest intermediate checkpoint; this fallback must be recorded as a different initialization. All three configurations are mixed within training. The supplied chain does not train three separate policies in sequence.

\begin{table}[htbp]
\centering\small
\begin{tabularx}{\textwidth}{@{}p{.30\textwidth}XX@{}}
\toprule
Setting & RSFT  & Agentic RL defaults \\
\midrule
Training data and epochs & Mixed three-role; one epoch & Mixed three-role; one epoch \\
Learning rate & $3\times10^{-6}$ & $2\times10^{-6}$ \\
Batch & $8$ devices $\times2$ samples $\times2$ accumulation & $16$ prompts $\times8$ rollouts \\
Sequence budget & 8,192 tokens & 4,096 prompt; 6,144 continuation; 12,288 total \\
Assistant-turn budget & Recorded conversation & 10 \\
Image handling & Resize parameter 512 & Backend-specific cached images \\
Actor updates & ViT, visual merger, language decoder & ViT, visual merger, language decoder; FSDP \\
RL schedule / clipping & --- & Cosine; 5\% warmup; grad norm 1.0; clip radius $\eta{=}0.2$ (verl default); weight decay 0 \\
Rollout sampling & --- & Temperature 0.7; top-$p$ 0.9 \\
OPSD & --- & Coefficient 2.0; top-128 + residual bucket; chunk 4 \\
Other reward weights & --- & Format 0.01; judge 0.6; linkage $+0.5/-0.3$ \\
Group-relative linkage & --- & Top 2 of 8; score range at least 0.02 \\
Frozen-reference penalty & --- & Coefficient 0.05; \artifact{low_var_kl}; entropy coefficient 0 \\
vLLM engine / memory & --- & \artifact{max_num_seqs} 96; GPU utilization 0.55; 11,264 dynamic tokens per GPU; tool responses $\le$2,048; parameter + optimizer offload; save every 25 steps \\
\bottomrule
\end{tabularx}
\caption{RSFT and Agentic RL launcher defaults. Run-specific overrides and initialization checkpoints are recorded separately.}
\label{tab:appendix_training}
\end{table}

\begin{figure}[htbp]
\centering
\includegraphics[width=\linewidth]{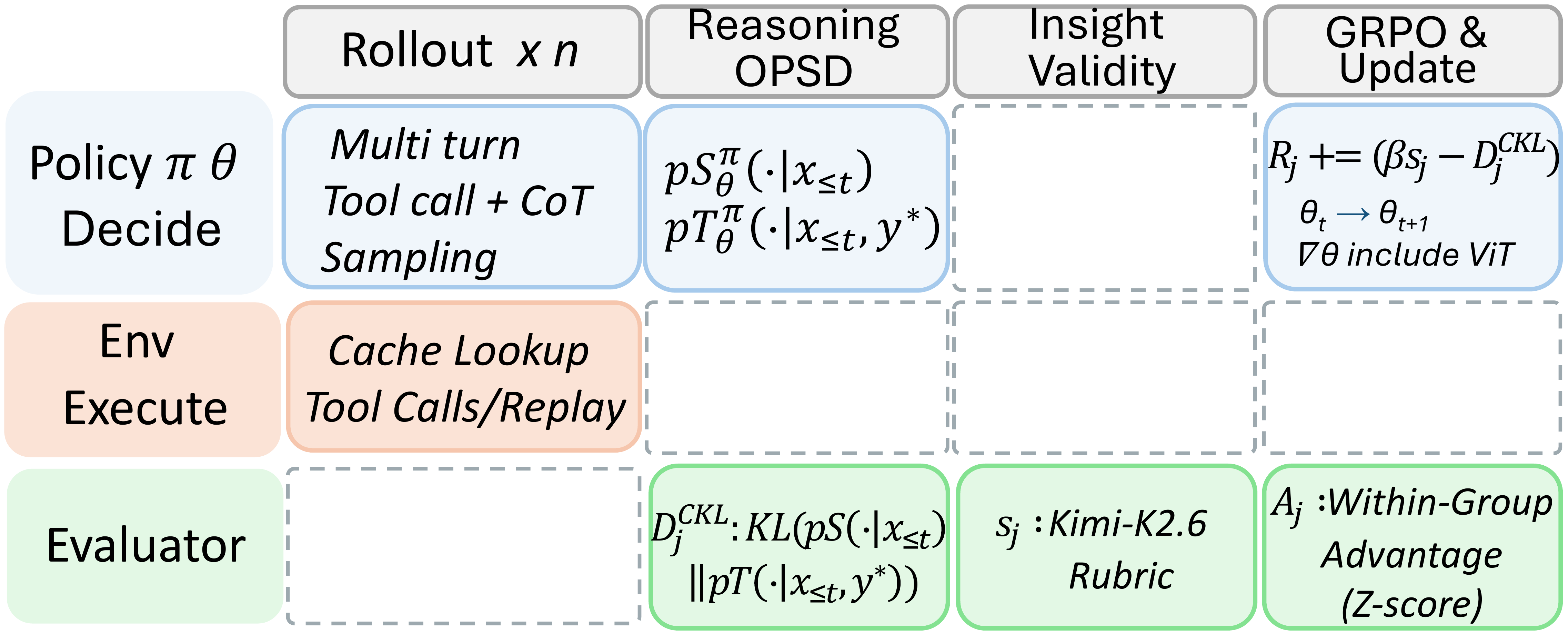}
\caption{\textbf{Weight lifecycle within one RL training step.} The weights $\vtheta_t$ produced by the previous step are the step's only active policy. vLLM samples rollouts from $\vtheta_t$; the two OPSD scoring forwards (ordinary and answer-privileged) run on the same $\vtheta_t$; the optimizer computes $\vtheta_{t+1}$ from the collected batch and syncs it back to vLLM for the next step. The retired weights $\vtheta_{t-1}$ leave their stored sampling log-probabilities (\texttt{old\_log\_probs}) as the denominator of the clipped ratio, while the reference policy $\pi_{\mathrm{ref}}=\vtheta_0$ stays frozen and serves only the $\beta\,\mathcal K_{\mathrm{ref}}$ penalty. Hyperparameters in \cref{tab:appendix_training}.}
\label{fig:rl_step}
\end{figure}

\paragraph{Weight lifecycle within one RL step.}
\Cref{fig:rl_step} traces the three weight copies through one step. The incoming weights $\vtheta_t$ are the step's only active policy, so rollouts sample from them, the paired OPSD scoring forwards read them, and the optimizer produces $\vtheta_{t+1}$ from the collected batch and syncs it back to vLLM before the next step begins. The retired weights $\vtheta_{t-1}$ survive only as stored sampling log-probabilities (\artifact{old_log_probs}), which form the denominator of the clipped importance ratio. The reference policy $\pi_{\mathrm{ref}}=\vtheta_0$ simply denotes the RSFT checkpoint at which RL starts. That checkpoint is never updated and never samples; it is kept in memory for one purpose, to score the $\beta\,\mathcal K_{\mathrm{ref}}$ penalty that keeps the moving policy from drifting too far from its initialization (\cref{tab:appendix_training}).

SFT materializes recorded image references and uses assistant spans as prediction targets. Tool observations and system/user text provide context; padding is masked. Neither stage uses parameter-efficient adapters or freezes the ViT in the full-model setting. The SFT loader constructs the complete vision--language model, and the RL optimizer receives all actor parameters; the framework's \artifact{freeze_vision_tower} flag, which would otherwise freeze the visual encoder (the ViT ``vision tower''), stays at its default \texttt{False}. Observation masking removes local prediction losses but preserves the gradient path from later assistant tokens to visual features. All variants in \cref{tab:learning_ablation} retain full-parameter optimization.

The current mixed serializer chooses either the three-tool Anatomist or the three-tool Radiologist schema from observed calls; it has no explicit ten-tool Consultant branch. Loading three trajectory sources therefore does not establish identical schema conditioning across SFT and RL. The manifest interface admits explicit participant/question pairs; an acceptance manifest and a cached training sequence are distinct artifacts. Cache counts must not be reinterpreted as quality-screen retention counts.

\paragraph{SFT and RL materialization details.}
After the participant-disjoint split, the RSFT stage retains 16{,}916 / 17{,}921 / 17{,}921 questions for the Anatomist, Radiologist, and Consultant roles. SFT renders each retained trajectory through the model chat template with images at their message positions and supervises assistant spans only (the tool-call segment and the terminal five-phase answer), masking system, user, and tool-returned tokens; span positions are recovered by incremental re-rendering and length differencing. Images are resized to $512\times512$ and stored by path, then pretokenized by 64 parallel workers into per-sample tensors (inputs, labels, mask, image); this pass yields 48{,}464 samples trained for one epoch at learning rate $
3\times10^{-6}$, producing the checkpoint that initializes RL. RL materializes one parquet row per training question (30{,}604 rows) carrying the data source, the rendered system and user messages, the reference answer, the agent name, and an extra-info bundle with the answer and options for scoring, the question text, the recorded teacher insight (used only as privilege, never shown to the policy), and a tool-parameter pack supplying each tool instance's participant identifier. Rollout prompts contain no images: every view and measurement is acquired on-policy through tool calls. Mock tools then replay, deterministically and without GPU work, the cached result the real tool would have returned for that participant, tool, and argument set.

\paragraph{Answer supervision and leakage controls.}
Neither stage trains on ground-truth answers. RSFT distills recorded teacher trajectories whose terminal answers are the teacher's own predictions (the reference is withheld during recording), so they act as \textbf{pseudo-labels} for action format and synthesis, not as ground truth. Agentic RL sees \textbf{no ground-truth answer in context}; the reference enters only the outcome reward after the first-pass answer is locked, so the policy learns from \textbf{reward signals alone}. OPSD's hint conditions only a detached scoring forward of a frozen snapshot through which gradients do not flow and which the updated policy never sees, so privileged scoring \textbf{cannot leak the answer into acquisition behavior}.

\paragraph{Why the Instruct checkpoint.}
We initialize from Qwen3-VL-8B-Instruct rather than a reasoning/thinking variant or a larger successor for three reasons. Training-start comparability: the Instruct variant avoids inheriting a checkpoint specialized for explicit chain-of-thought output. The Qwen3-VL technical report describes reinforcement-learning post-training for the model family, so this choice does not isolate CASE's effects from all vendor-side RL~\citep{bai2025qwen3vl}. Efficiency: at 8B parameters the full model (vision encoder, merger, decoder) plus optimizer state, rollouts, and the paired OPSD scoring forwards fit the single $8\times80$GB node, whereas a larger backbone would force sharding or a separate scoring cluster. Stack stability: Qwen3.5 uses a gated-delta-network architecture~\citep{qwen2026qwen35}; in our distributed rollout stack, its Triton/FlashInfer kernels compiled just in time at launch and intermittently deadlocked, whereas the Qwen3-VL kernel path compiled ahead of time and trained reliably.
\begin{figure}[htbp]
\centering
\includegraphics[width=0.92\linewidth]{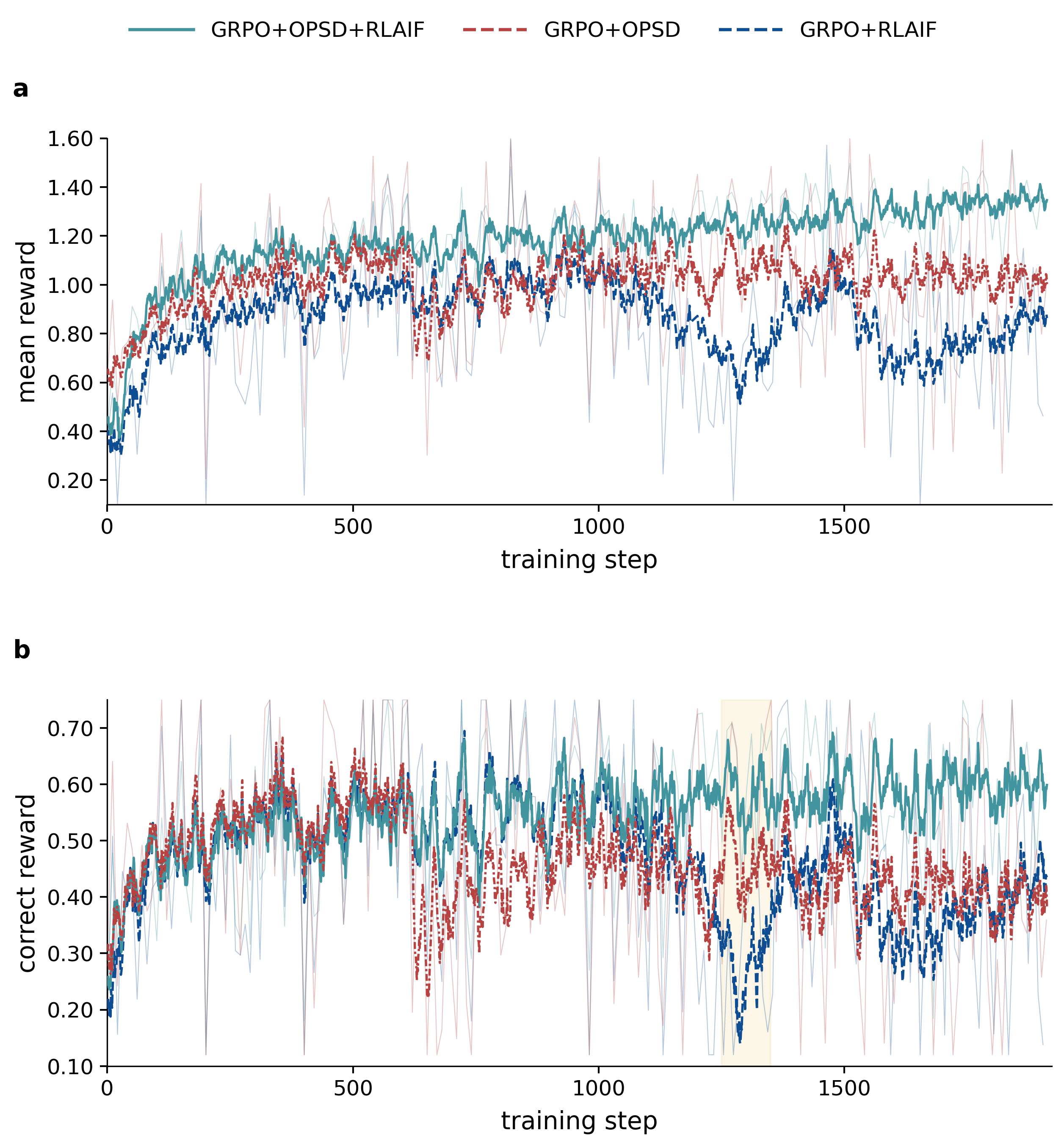}
\caption{Full-epoch training dynamics (steps 1--1{,}912) for the three RL variants with the line coding of \cref{fig:ablation_zoom}: (a)~mean return; (b)~correctness reward. The highlighted band marks the GRPO+RLAIF collapse; CASE-8B (full) is the steadiest and converges within one epoch, while GRPO+OPSD avoids collapse but, without the insight term, is prone to reward hacking and ends at lower accuracy.}
\label{fig:ablation_curves}
\end{figure}

\paragraph{Training dynamics and reward collapse.}
\Cref{fig:ablation_curves} extends the main-text zoom to the full epoch (steps 1--1{,}912) for the three RL variants, showing (a)~the mean return and (b)~the correctness reward. CASE-8B (full, teal) rises steadily with the smallest step-to-step variance and converges within the single epoch. GRPO+OPSD (red) does not collapse (the privileged brake suppresses the repetitive mode) but, lacking the rubric insight term, it is exposed to reward hacking: surface plausibility can be inflated without strengthening the evidence-to-answer logic, which is reflected in its lower endpoint accuracy. GRPO+RLAIF (blue) tracks the others early, then collapses in the highlighted band exactly where repetitive restatement overtakes commitment. The severity of this late degeneration also reflects the benchmark's difficulty: each episode is a long horizon in which the policy itself chooses which studies, modality, sequence, and frame to inspect, so an unresolved evidence path leaves no reliable basis on which to conclude, and restatement is the cheapest continuation.

\paragraph{A worked repetition case.}
To further verify that OPSD acts as a brake on repetition rather than as a length penalty, and to explain why the rubric-only variant collapses late in training, we inspect one repetition loop on participant 1588571, whose synthesis has degenerated into \emph{``\ldots these findings are not diagnostic of diabetes, these findings are not diagnostic of diabetes, \ldots''}. Conditioned on this same student prefix $\hat y_{<t}$, the ordinary distribution $p_S(\cdot\mid x,\hat y_{<t})$ and the answer-privileged distribution $p_T(\cdot\mid x,y^\star,\hat y_{<t})$ point in opposite directions:
\begin{equation}
p_S=\begin{cases}\texttt{these}&0.31\\\texttt{findings}&0.22\\\texttt{not}&0.16\\\texttt{diagnostic}&0.12\\\texttt{however}&0.05\\\texttt{therefore}&0.03\end{cases}
\qquad
p_T=\begin{cases}\texttt{therefore}&0.27\\\texttt{however}&0.18\\\texttt{given}&0.16\\\texttt{no}&0.13\\\texttt{select}&0.09\\\texttt{these}&0.02\end{cases}
\label{eq:repetition_case}
\end{equation}
The decisive property is not any individual token but the orientation of the two distributions. $p_S$ places high mass on previously repeated tokens, whereas $p_T$ places high mass on tokens that advance the clinical decision, so aligning the ordinary view with the privileged one transfers mass from loop-continuation toward commitment. This is precisely the brake the rubric-only objective lacks: a text-only rubric scores the written insight and cannot observe that the next tokens merely restate the prefix, and because a compact policy learns the surface form of an explanation before it reliably resolves the answer, under residual uncertainty the cheapest continuation is restatement; once correctness stalls, restatement compounds into the late-epoch collapse of \cref{fig:ablation_curves}. \Cref{tab:repetition_case} contrasts the two ablation variants on the same participant: without OPSD the synthesis restates verbatim and never commits, whereas the full model weighs the supporting context against the draft reports and derives the answer from its insight.

\begin{table}[htbp]
\centering\footnotesize
\setlength{\tabcolsep}{4pt}
\renewcommand{\arraystretch}{1.05}
\begin{tabular}{p{0.20\linewidth}p{0.75\linewidth}}
\toprule
Variant & Phase-4 synthesis for participant 1588571 \\
\midrule
GRPO + RLAIF (no OPSD) & The pancreatic imaging findings offer only indirect evidence relevant to diabetes, including fat deposition, atrophy, or secondary complications. \textbf{However, these findings are not diagnostic of diabetes, these findings are not diagnostic of diabetes, these findings are not diagnostic of diabetes\ldots} \\
\addlinespace[3pt]
CASE-8B (full) & The current MRI findings provide no direct basis for diagnosing diabetes, although pancreatic fat, atrophy, or associated complications may offer supporting context. In this case, the reports describe preserved anatomy and improving pancreatic fat content, without diabetes-related complications or a documented diagnosis in the clinical history. Therefore, the available evidence does not support selecting the diabetes option. \\
\bottomrule
\end{tabular}
\caption{Repetition behavior on participant 1588571 with and without OPSD. Without the privileged brake the synthesis degenerates into verbatim restatement and never commits; the full model states the supporting context, weighs the draft reports, and derives the answer strictly from its insight.}
\label{tab:repetition_case}
\end{table}

\section{Positioning within On-Policy Distillation and Agentic RL}
\label{app:distillation_landscape}

On-policy sampling, the source of teacher information, and the way feedback enters optimization are separate design choices. Student-generated rollouts do not themselves imply self-distillation: GKD, RLAD, TGPO, and OPDVR use external teachers in their studied formulations. For privileged self-distillation, shared weights do not make all update rules equivalent. \Suppref{tab:distillation_landscape} separates direct student supervision, local advantage modulation, and trajectory-level reward construction.

\begin{table}[htbp]
\centering\footnotesize
\setlength{\tabcolsep}{4pt}
\renewcommand{\arraystretch}{1.12}
\begin{tabularx}{\textwidth}{@{}p{.13\textwidth}p{.25\textwidth}Y@{}}
\toprule
Method & Teacher / privileged information & How feedback enters learning \\
\midrule
GKD~\citep{agarwal2024gkd} & External teacher on student-generated sequences & Distribution matching on learner-visited prefixes; compatible with joint RL fine-tuning. \\
RLAD~\citep{zhang2026rlad} & External teacher and old-student mixture & Mixture-anchored policy ratios make imitation selective according to rollout advantage. \\
TGPO~\citep{liu2026tgpo} & External teacher at student prefixes & Cross entropy on teacher-preferred next tokens supplements trajectory-level verifiable-reward optimization. \\
OPDVR~\citep{lin2026opdvr} & External teacher and outcome verifier & Outcome correctness gates the direction of sampled-token distillation rewards within a policy-gradient objective. \\
TRACE~\citep{wang2026trace} & Synchronized self-teacher with coarse diagnostic labels & Differentiable student distillation is routed to annotated reasoning spans; other positions retain GRPO. The default emphasizes forward KL on key spans of correct rollouts. \\
SSOPD~\citep{tan2026ssopd} & Self-generated correct completion & A forward-KL auxiliary loss supervises early prefixes of an unsuccessful completion; GRPO remains the outcome objective. \\
OCSD~\citep{yang2026ocsd} & Future observations in matched full/ablated replay & Observation-calibrated sampled-token scores modulate advantages at selected steps while preserving the trajectory-advantage sign. \\
SSPO~\citep{wu2026sspo} & Evidence anchors and incorrect-answer feedback & Step log-likelihood disagreement gives positive advantage multipliers on unsuccessful trajectories; correct trajectories retain GRPO. \\
Our formulation & Current policy with training-only answer / optional reference insight & Mean insight-region reverse KL is detached and subtracted from the joint correctness--rubric reward before group normalization; all assistant actions share the resulting trajectory advantage. \\
\bottomrule
\end{tabularx}
\caption{Mechanism comparison, not an empirical ranking. The formulations correspond to the versions cited. Direct distillation differentiates through student probabilities; privileged replay scores can instead be held fixed while optimizing the rollout policy.}
\label{tab:distillation_landscape}
\end{table}

\paragraph{Scope of selective supervision.}
TRACE selects annotated critical or error spans and supplies coarse diagnostic labels to its teacher; the locations gate its loss rather than being copied into the privileged prompt. Our insight region instead follows a structured response phase with an approximate implementation mask (\Cref{sec:opsd}). We neither identify causally decisive tokens nor directly differentiate an insight distillation loss. The comparison is between a routed process loss and a trajectory-ranking signal, not between selective and indiscriminate distillation.

\paragraph{Reward construction versus advantage modulation.}
\citet{yang2026ocsd} and \citet{wu2026sspo} preserve the sign of the inherited trajectory advantage in OCSD and SSPO while changing selected token or step weights. In our formulation, $R_{n,j}=\sg[B_j-\lambda_{\mathrm{CKL}}D^{\mathrm{CKL}}_{n,j}]$ is formed before group normalization (\Cref{eq:reward}), so the synthesis penalty can change the ordering of whole trajectories and hence their relative-advantage signs. All assistant actions within a trajectory inherit that same scalar. This can train the acquisition policy, but does not isolate which preceding retrieval caused a good or poor synthesis. \citet{tian2026pbsd} and \citet{zhang2026adrs} provide further relevant controls through turn-level PBSD modulation and token-level ADRS shaping.

\paragraph{Evaluating repetition control.}
Training-only privilege is intended to reduce unresolved synthesis and redundant restatement, not to reward short responses as such. A lower distributional penalty need not imply less repetition or a correct inference: a model could also ignore the privilege or rationalize a known answer. Controlled comparisons should therefore report repetition and loop frequency together with accuracy, evidence sensitivity, and token use. A full-system accuracy gain does not establish an OPSD-specific effect.

\section{Rubric-Based Feedback and Group-Relative Selection}
\label{app:judge}

RL uses batch rubric scoring as task-specific direct AI feedback~\citep{lee2024rlaif}, followed by within-group top-2 selection. Each scoring item contains the question, options, reference answer, predicted answer, and Phase-4 insight; images and the full tool log are excluded. The prompt truncates question and option text to 300 characters each and insight text to 800. The configured judge, \artifact{kimi-k2.6}, returns one aggregate score in $[0,1]$ and a brief justification. The four criteria in \suppref{tab:appendix_rubric} define its allocation, not separately measured subscores.

\begin{table}[htbp]
\centering\small
\begin{tabularx}{\textwidth}{@{}p{.23\textwidth}cX@{}}
\toprule
Criterion & Allocation & Assessment \\
\midrule
Evidence identification & $[0,0.3]$ & Specific findings rather than generic abnormality claims. \\
Logical reasoning & $[0,0.3]$ & A coherent connection between findings, interpretation, and answer. \\
Answer support & $[0,0.2]$ & Evidence that supports the selected conclusion. \\
Completeness & $[0,0.2]$ & The intermediate steps needed to justify the conclusion. \\
\bottomrule
\end{tabularx}
\caption{Fixed rubric for batch insight scoring. Group-relative selection changes the linkage adjustment, not these scoring criteria.}
\label{tab:appendix_rubric}
\end{table}

\paragraph{Outcome-linked top-2 selection.}
A batch of 16 prompts produces $16\times8=128$ rollouts. Linkage selection is performed separately within each question's eight-rollout group, not across the entire batch. The rule selects two rollouts per eligible group rather than a fixed percentile; it does not define an absolute standard of reasoning quality.

The batch wrapper scores insights longer than 20 characters and assigns zero to missing or shorter spans. For question group $x$, let $\mathcal U_x$ contain the indices of the two highest-scoring eligible insights. If fewer than two are eligible or their score range is below $0.02$, set $\mathcal U_x=\varnothing$. Ties follow the stable input order. With $c_j$ denoting exact answer correctness, selection determines the linkage adjustment as
\begin{equation}
 b_j=\ind[j\in\mathcal U_x],\qquad
 L_j=\begin{cases}
 +0.5, & b_j=1\ \text{and}\ c_j=1,\\
 -0.3, & b_j=1\ \text{and}\ c_j=0,\\
 0, & b_j=0.
 \end{cases}
 \label{eq:judge_gate}
\end{equation}
This adjustment rewards highly rated synthesis paired with a correct answer and penalizes highly rated explanations of an incorrect answer. Every rollout retains its continuous $0.6\,s_j$ term, including those outside the top two. All eight receive paired OPSD scoring and enter GRPO; selection changes rewards, not training membership. The API batch size of four (up to 40 concurrent scoring workers) controls request packing independently of these rollout groups.

\paragraph{Judge failures.}
The batch wrapper catches an outer judge exception by assigning $0.5$ to eligible insights and falling back to $b_j=\ind[s_j>0.4]$. Individual API or parsing failures can also return $0.5$ without activating this fallback, allowing those entries to participate in ranking. These cases differ from a valid group with nearly uniform scores, for which linkage is disabled. Failure flags and fallback frequencies are needed to distinguish these reward paths. The reference-insight lexical diagnostic is not included in the returned reward.

\paragraph{Judge-source robustness of the rubric signal.}
The rubric term is supplied by an external judge API and is therefore a black-box feedback signal; a natural concern is that observed gains merely track one scorer's idiosyncrasies, as discussed for LLM-as-a-judge evaluation~\citep{zheng2023judge}. \Cref{fig:judge_sources} re-runs the training-time diagnostics under three judge models (Kimi K2.6~\citep{moonshot2026kimik26}, blue solid; Qwen3.7-Plus~\citep{qwen2026qwen37plus}, red dashed; DeepSeek V4-Flash~\citep{deepseekai2026v4}, teal dashed). In (a)~the causal-KL between ordinary and answer-privileged synthesis decreases steadily over steps 0--1{,}912 (1  epoch) for every judge, indicating that synthesis becomes progressively less dependent on the privileged hint (the policy resolves the answer from its own gathered evidence), independently of which scorer ranks the insights. In (b)~the judged insight score rises throughout training for all three judges even though their absolute levels differ, reflecting genuine differences in scoring strictness and openness rather than divergent trends. Together these show a consistent upward trend across the three tested judges; agreement among model-based scorers does not by itself establish independently verified clinical grounding.

\begin{figure}[htbp]
\centering
\includegraphics[width=0.78\linewidth]{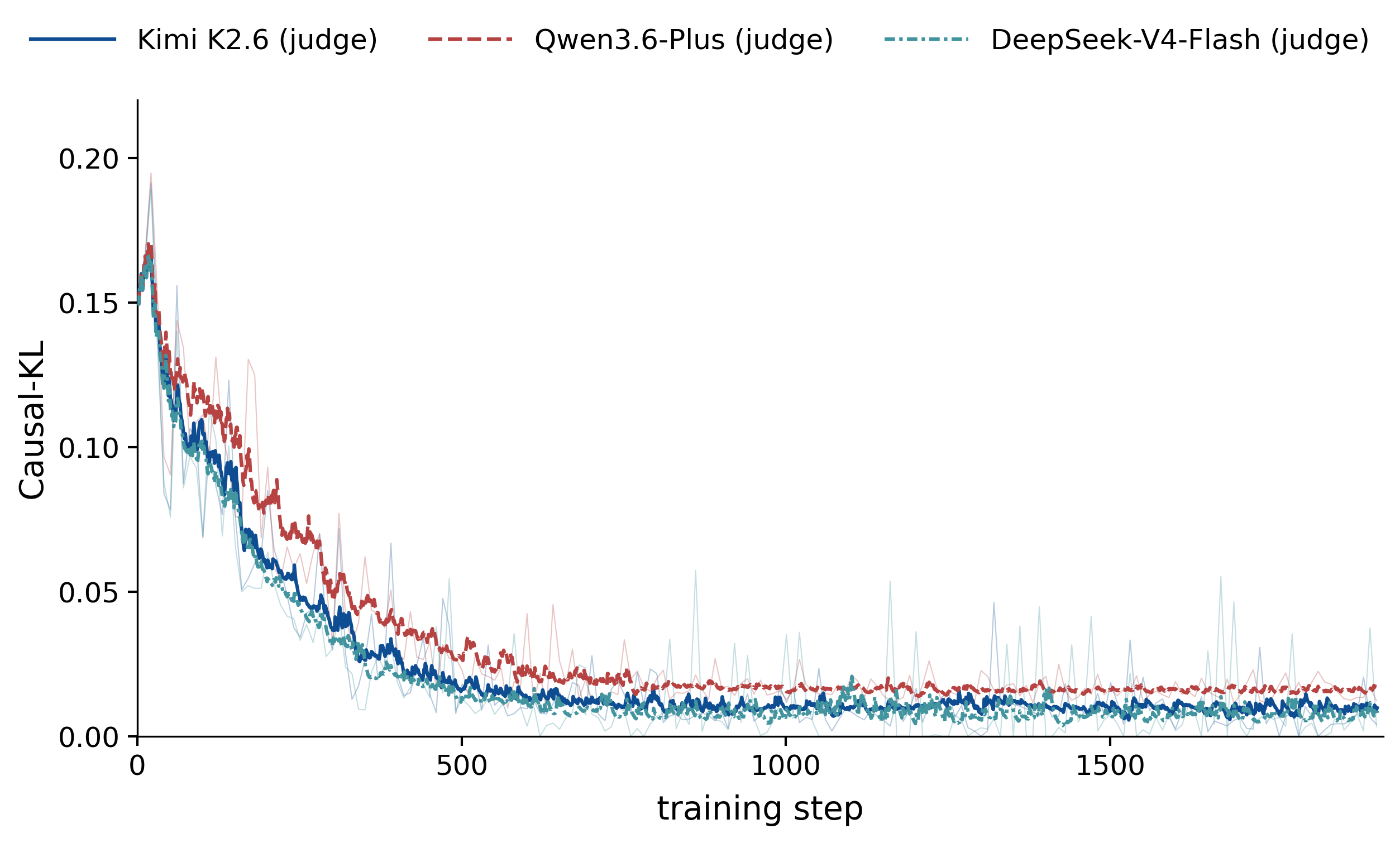}\\[2pt]
{\small (a) Causal-KL between ordinary and answer-privileged synthesis, steps 0--1{,}912.}\\[8pt]
\includegraphics[width=0.78\linewidth]{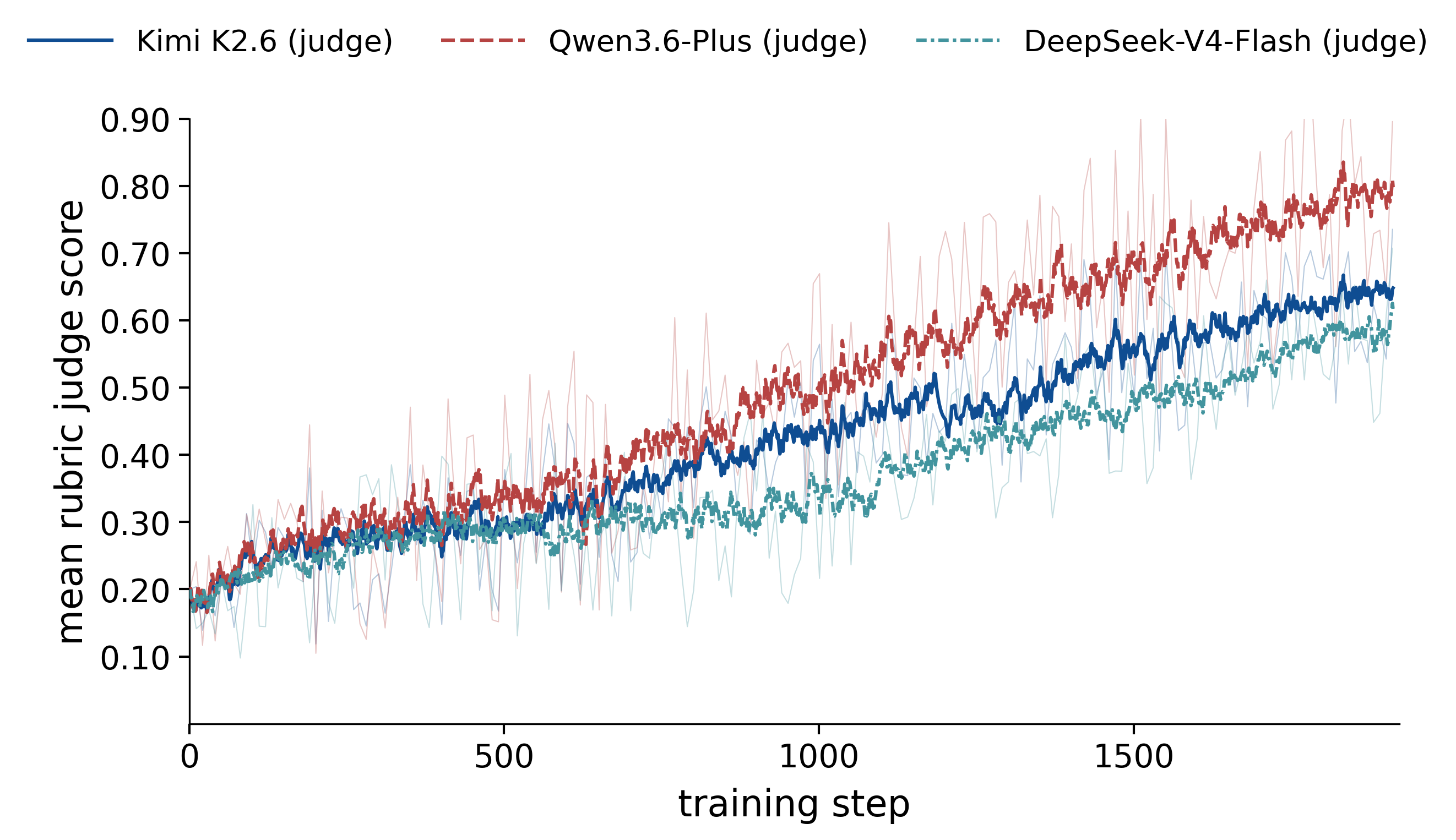}\\[2pt]
{\small (b) Judged insight score over the same steps.}
\caption{Judge-source robustness of the training-time rubric signal under three judge APIs (Kimi K2.6, blue solid; Qwen3.7-Plus, red dashed; DeepSeek V4-Flash, teal dashed). (a)~Causal-KL decreases steadily for every judge; (b)~the judged insight score rises for all three despite different absolute strictness, showing a consistent upward judged-score trend across the three tested sources.}
\label{fig:judge_sources}
\end{figure}

\paragraph{Evaluation scope.}
Top-2 selection is confined to the training reward. Held-out insight evaluation uses fixed rubric scores without group-relative selection, as specified in \appref{app:judge_evaluation}.

\section{Controlled Evaluation and Measurement Definitions}
\label{app:evaluation}
\label{app:judge_evaluation}

\paragraph{Fixed comparison unit.}
A versioned role-specific manifest fixes participant/question identifiers, task--format strata, options, labels, and evidence availability before inference. All methods within a panel share prompts, tools, stopping rules, and parsing. The supplied evaluator uses temperature zero, 6,144 tokens per response, and a 16,384-token local context limit; the Anatomist permits 4 executed tool calls, the Radiologist 6, and the Consultant 12. Resolved model/API versions, inference-effort settings, image limits, and denominators accompany each result. Matched interfaces do not imply identical provider-side compute.

\paragraph{Test-time information boundary.}
Score only the first answer, before reference feedback or the post-answer review reflection. Fix the skill-memory sources with test participants and their feedback excluded; the store is live within a run but frozen across runs (\appref{app:environment}), and an explicitly empty store defines the cold condition. Feedback-adapted evaluation rounds remain separate from held-out results.

\paragraph{Insight-score fairness.}
The complementary insight score is produced by a single judge whose API version and prompt are pinned and reported, so scores stay comparable over time even as provider models are updated. This judge is held out from the RL reward model, so the policy is never trained toward the evaluator, and all methods in a panel are scored in one batch under the same rubric, temperature, and option ordering. To confirm that conclusions do not hinge on one judge, we re-score a subset with at least one additional judge from a different provider and report rank correlation and any order changes.

\paragraph{Exact matching and aggregation.}
Let $\mathcal D$ contain $N$ scheduled questions, $\mathcal D_t$ contain $N_t$ questions in stratum $t$, and $\mathcal D_p$ contain $N_p$ questions from participant $p$, with $P$ participants. For predicted and reference option sets $\widehat Y_i,Y_i^\star$, define $a_i=V_i\ind[\widehat Y_i=Y_i^\star]$, where $V_i=1$ only for a completed first pass with a valid, unambiguous answer. Then
\begin{align}
 A_t&=\frac{1}{N_t}\sum_{i\in\mathcal D_t}a_i,
 & A_{\mathrm{micro}}&=\frac{1}{N}\sum_{i\in\mathcal D}a_i
 =\sum_{t=1}^{10}\frac{N_t}{N}A_t,\nonumber\\
 A_{\mathrm{macro}}&=\frac{1}{10}\sum_{t=1}^{10}A_t,
 & A_{\mathrm{macro,P}}&=\frac{1}{P}\sum_{p=1}^{P}
 \frac{1}{N_p}\sum_{i\in\mathcal D_p}a_i.
 \label{eq:appendix_strict_accuracy}
\end{align}
All accuracies are percentages. Micro and Macro are the reported overall metrics; participant-macro is an optional sensitivity analysis. For domain strata $\mathcal T$, pool counts as $\sum_{t\in\mathcal T}N_tA_t/\sum_{t\in\mathcal T}N_t$. The mapping in \suppref{tab:appendix_types} places hypertension within cardiovascular medicine; Macro still averages ten strata, not six domains.

Canonicalization ignores option order and duplicates but rejects out-of-set letters and multiple selections on single-select items. Multi-select requires complete set equality. Missing records, terminal failures, and invalid answers score zero; recoverable tool errors alone do not invalidate a final answer. Archived extractors recognize A--I and may default to A after failure, so final scoring must reparse the original response against the actual option set, including J where applicable.

Secondary analyses report binary class counts, a training-majority baseline, and balanced accuracy (mean class recall), together with multi-select accuracy by reference-set cardinality.

\paragraph{Uncertainty and relative gains.}
Use 10,000 paired participant-cluster bootstrap replicates, retaining each sampled participant's questions and identical resampling indices across methods; the 2.5th and 97.5th percentiles give 95\% intervals. Every reported quantity is recomputed within each replicate rather than combined from the two endpoint intervals. These intervals measure test-sample uncertainty conditional on a checkpoint; variability across independent training seeds is reported separately.

\paragraph{Per-role clinical accuracy.}
\Cref{tab:perrole_accuracy} gives the full per-role breakdown of the role-averaged \cref{tab:main_accuracy}, with the same ten strata, model grouping, and matched per-panel conditions.

\begin{table}[htbp]
\centering
\footnotesize
\setlength{\tabcolsep}{3pt}
\renewcommand{\arraystretch}{1.1}
\resizebox{\linewidth}{!}{%
\begin{tabular}{l*{12}{c}}
\toprule
\raisebox{0.5\baselineskip}{Domain} & \multicolumn{1}{c}{\raisebox{0.5\baselineskip}{Oncology}} & \multicolumn{1}{c}{\shortstack{Endocrine/\\metabolic}} & \multicolumn{3}{c}{\raisebox{0.5\baselineskip}{Cardiovascular}} & \multicolumn{2}{c}{\shortstack{Respiratory/\\allergy}} & \multicolumn{2}{c}{\raisebox{0.5\baselineskip}{Multisystem}} & \multicolumn{1}{c}{\shortstack{Medical\\history}} & \multicolumn{2}{c}{\raisebox{0.5\baselineskip}{Overall (\%) $\uparrow$}} \\
\cmidrule(lr){2-2}\cmidrule(lr){3-3}\cmidrule(lr){4-6}\cmidrule(lr){7-8}\cmidrule(lr){9-10}\cmidrule(lr){11-11}\cmidrule(lr){12-13}
Method / Question & \shortstack{Cancer\\(B)} & \shortstack{Diabetes\\(B)} & \shortstack{Condition\\(S)} & \shortstack{Condition\\set (M)} & \shortstack{Hyper-\\tension\\(B)} & \shortstack{Condition\\(S)} & \shortstack{Condition\\set (M)} & \shortstack{ICD-10\\(S)} & \shortstack{Disease\\systems\\(M)} & \shortstack{History\\(B)} & Micro & Macro \\
\midrule
\multicolumn{13}{@{}l}{\textit{(a) Anatomist Agent}} \\
Qwen3-VL-8B-Instruct & 67.6 & 27.6 & 78.7 & 69.9 & 36.5 & 43.4 & 38.5 & 8.8 & 5.9 & 45.3 & 39.5 & 42.2 \\
\addlinespace[2pt]
Lingshu-32B & \textbf{89.4} & 82.9 & 98.5 & \textbf{97.8} & 72.4 & 65.5 & 68.7 & 17.7 & 0.6 & 28.4 & 59.2 & 62.2 \\
Lingshu-7B & 88.0 & \textbf{95.2} & 54.3 & 35.2 & 78.0 & 54.5 & 35.4 & 13.6 & 2.8 & 44.6 & 47.1 & 50.2 \\
MedGemma-27B & 41.2 & 25.5 & 19.2 & 46.9 & 23.7 & 13.0 & 14.6 & 7.9 & 2.5 & 27.9 & 21.0 & 22.2 \\
MedGemma-1.5-4B & 19.2 & 21.6 & 20.9 & 38.0 & 18.1 & 9.1 & 9.1 & 9.2 & 0.6 & 27.0 & 16.6 & 17.3 \\
\addlinespace[2pt]
Kimi K2.6 & 76.0 & 64.7 & 68.1 & 85.4 & 59.6 & 77.5 & 63.4 & 50.2 & 3.7 & 60.2 & 61.1 & 60.9 \\
Kimi K3 & 86.1 & 89.4 & 87.8 & 92.8 & 67.1 & 87.7 & 65.1 & 48.6 & 3.4 & 59.3 & 68.3 & 68.7 \\
Qwen3.7-Plus & 86.4 & \textbf{95.2} & 91.0 & 95.6 & 78.6 & 67.5 & 68.4 & 21.0 & 1.6 & 42.3 & 62.1 & 64.8 \\
Qwen3.8-Max & 88.0 & 74.0 & 96.1 & 96.9 & 68.2 & 77.7 & 69.4 & 45.0 & 1.6 & 46.8 & 65.8 & 66.4 \\
Gemini-3.5-Flash & 69.3 & 73.8 & 66.2 & 72.7 & 54.6 & 75.0 & 51.0 & 33.7 & 1.6 & 47.6 & 53.7 & 54.5 \\
GPT-5.4-0305 & 71.9 & 86.5 & 74.9 & 82.1 & \textbf{79.7} & 80.9 & 61.2 & 40.6 & 2.8 & \textbf{68.8} & 63.6 & 64.9 \\
Claude Opus 4.8 & 74.3 & 79.6 & 81.5 & 85.2 & 61.3 & 84.5 & 67.7 & 64.0 & 4.4 & 57.7 & 67.4 & 66.0 \\
\addlinespace[2pt]
\textbf{CASE (8B, ours)} & \textbf{89.4} & \textbf{95.2} & \textbf{99.5} & \textbf{97.8} & \textbf{79.7} & \textbf{94.5} & \textbf{73.0} & \textbf{67.0} & \textbf{10.3} & 62.4 & \textbf{77.5} & \textbf{76.9} \\
\bottomrule
\end{tabular}}
\vspace{4pt}
\resizebox{\linewidth}{!}{%
\begin{tabular}{l*{12}{c}}
\toprule
\raisebox{0.5\baselineskip}{Domain} & \multicolumn{1}{c}{\raisebox{0.5\baselineskip}{Oncology}} & \multicolumn{1}{c}{\shortstack{Endocrine/\\metabolic}} & \multicolumn{3}{c}{\raisebox{0.5\baselineskip}{Cardiovascular}} & \multicolumn{2}{c}{\shortstack{Respiratory/\\allergy}} & \multicolumn{2}{c}{\raisebox{0.5\baselineskip}{Multisystem}} & \multicolumn{1}{c}{\shortstack{Medical\\history}} & \multicolumn{2}{c}{\raisebox{0.5\baselineskip}{Overall (\%) $\uparrow$}} \\
\cmidrule(lr){2-2}\cmidrule(lr){3-3}\cmidrule(lr){4-6}\cmidrule(lr){7-8}\cmidrule(lr){9-10}\cmidrule(lr){11-11}\cmidrule(lr){12-13}
Method / Question & \shortstack{Cancer\\(B)} & \shortstack{Diabetes\\(B)} & \shortstack{Condition\\(S)} & \shortstack{Condition\\set (M)} & \shortstack{Hyper-\\tension\\(B)} & \shortstack{Condition\\(S)} & \shortstack{Condition\\set (M)} & \shortstack{ICD-10\\(S)} & \shortstack{Disease\\systems\\(M)} & \shortstack{History\\(B)} & Micro & Macro \\
\midrule
\multicolumn{13}{@{}l}{\textit{(b) Radiologist Agent}} \\
Qwen3-VL-8B-Instruct & 60.5 & 38.4 & 70.1 & 65.0 & 47.6 & 56.0 & 49.6 & 33.4 & 3.8 & 39.9 & 46.0 & 46.4 \\
\addlinespace[2pt]
Lingshu-32B & 85.4 & 72.4 & 25.4 & 42.6 & 43.2 & 64.9 & 58.7 & 42.4 & 2.8 & 51.7 & 49.4 & 49.0 \\
Lingshu-7B & 20.1 & 66.4 & 4.0 & 10.5 & 41.2 & 21.4 & 24.8 & 34.7 & 2.0 & 44.2 & 27.9 & 26.9 \\
MedGemma-27B & 11.0 & 7.9 & 16.3 & 4.2 & 28.1 & 23.7 & 27.5 & 39.5 & 4.3 & 61.1 & 24.0 & 22.4 \\
MedGemma-1.5-4B & 65.8 & 56.7 & 52.2 & 74.5 & 58.8 & 58.4 & 60.0 & 34.7 & 4.8 & 49.4 & 50.7 & 51.5 \\
\addlinespace[2pt]
Kimi K2.6 & 42.4 & 35.5 & 33.5 & 32.5 & 38.6 & 34.5 & 37.1 & 42.6 & 5.1 & 58.8 & 37.1 & 36.1 \\
Kimi K3 & 85.2 & 91.6 & 84.6 & 87.7 & 73.1 & 90.2 & 67.5 & 60.0 & 6.3 & 44.2 & 69.8 & 69.0 \\
Qwen3.7-Plus & 87.0 & 74.8 & 80.4 & 84.4 & 64.7 & 86.7 & 65.7 & 55.5 & 8.0 & 52.2 & 66.3 & 65.9 \\
Qwen3.8-Max & 51.2 & 31.1 & 21.0 & 41.3 & 22.0 & 44.1 & 38.0 & 57.5 & 6.3 & \textbf{62.4} & 40.2 & 37.5 \\
Gemini-3.5-Flash & 70.9 & 77.9 & 65.8 & 71.0 & 61.4 & 79.6 & 54.7 & 42.9 & 2.8 & 36.8 & 56.4 & 56.4 \\
GPT-5.4-0305 & 72.0 & 89.8 & 73.2 & 78.5 & \textbf{87.7} & 84.0 & 64.2 & 50.7 & 5.4 & 51.9 & 65.4 & 65.7 \\
Claude Opus 4.8 & 72.0 & 79.9 & 76.8 & 78.9 & 65.2 & 85.2 & 68.6 & \textbf{77.4} & 8.0 & 42.2 & 68.3 & 65.4 \\
\addlinespace[2pt]
\textbf{CASE (8B, ours)} & \textbf{88.1} & \textbf{93.2} & \textbf{89.5} & \textbf{90.1} & 78.8 & \textbf{94.0} & \textbf{71.0} & 68.3 & \textbf{11.4} & 53.7 & \textbf{74.8} & \textbf{73.8} \\
\bottomrule
\end{tabular}}
\vspace{4pt}
\resizebox{\linewidth}{!}{%
\begin{tabular}{l*{12}{c}}
\toprule
\raisebox{0.5\baselineskip}{Domain} & \multicolumn{1}{c}{\raisebox{0.5\baselineskip}{Oncology}} & \multicolumn{1}{c}{\shortstack{Endocrine/\\metabolic}} & \multicolumn{3}{c}{\raisebox{0.5\baselineskip}{Cardiovascular}} & \multicolumn{2}{c}{\shortstack{Respiratory/\\allergy}} & \multicolumn{2}{c}{\raisebox{0.5\baselineskip}{Multisystem}} & \multicolumn{1}{c}{\shortstack{Medical\\history}} & \multicolumn{2}{c}{\raisebox{0.5\baselineskip}{Overall (\%) $\uparrow$}} \\
\cmidrule(lr){2-2}\cmidrule(lr){3-3}\cmidrule(lr){4-6}\cmidrule(lr){7-8}\cmidrule(lr){9-10}\cmidrule(lr){11-11}\cmidrule(lr){12-13}
Method / Question & \shortstack{Cancer\\(B)} & \shortstack{Diabetes\\(B)} & \shortstack{Condition\\(S)} & \shortstack{Condition\\set (M)} & \shortstack{Hyper-\\tension\\(B)} & \shortstack{Condition\\(S)} & \shortstack{Condition\\set (M)} & \shortstack{ICD-10\\(S)} & \shortstack{Disease\\systems\\(M)} & \shortstack{History\\(B)} & Micro & Macro \\
\midrule
\multicolumn{13}{@{}l}{\textit{(c) Consultant Agent}} \\
Qwen3-VL-8B-Instruct & 75.0 & 72.4 & 70.0 & 75.1 & 31.1 & 58.2 & 59.9 & 38.5 & 7.6 & 38.2 & 52.5 & 52.6 \\
\addlinespace[2pt]
Lingshu-32B & 80.1 & 55.8 & 41.1 & 23.5 & 30.9 & 46.2 & 27.5 & 45.6 & 5.4 & 41.4 & 41.2 & 39.8 \\
Lingshu-7B & 88.1 & 87.6 & 87.5 & 75.2 & \textbf{78.3} & 76.3 & 55.6 & 37.1 & 5.1 & 39.6 & 61.6 & 63.0 \\
MedGemma-27B & 74.4 & 70.2 & 18.8 & 1.8 & 39.6 & 40.5 & 30.5 & 39.8 & 4.8 & 57.8 & 38.5 & 37.8 \\
MedGemma-1.5-4B & 80.1 & 77.0 & 39.3 & 82.9 & 68.3 & 54.5 & 62.9 & 36.7 & 4.8 & 46.3 & 54.3 & 55.3 \\
\addlinespace[2pt]
Kimi K2.6 & 26.9 & 24.3 & 23.9 & 21.5 & 27.4 & 31.4 & 28.1 & 39.8 & 5.1 & 61.4 & 30.3 & 29.0 \\
Kimi K3 & \textbf{89.2} & \textbf{94.7} & 83.9 & 93.4 & 52.7 & 95.0 & 68.1 & 69.2 & 9.4 & 50.4 & 72.4 & 70.6 \\
Qwen3.7-Plus & 87.6 & 92.5 & 97.3 & \textbf{96.5} & 74.7 & 89.4 & 67.5 & 58.8 & 8.8 & 50.1 & 72.6 & 72.3 \\
Qwen3.8-Max & 88.3 & 94.0 & \textbf{97.5} & 95.8 & 25.1 & \textbf{95.4} & 71.0 & 71.0 & 8.8 & \textbf{62.9} & 73.1 & 71.0 \\
Gemini-3.5-Flash & 74.8 & 81.5 & 65.8 & 76.3 & 45.0 & 84.6 & 55.6 & 49.9 & 4.6 & 42.2 & 58.8 & 58.0 \\
GPT-5.4-0305 & 75.1 & 92.1 & 72.1 & 83.1 & 62.9 & 88.1 & 64.4 & 58.2 & 8.0 & 58.8 & 66.9 & 66.3 \\
Claude Opus 4.8 & 74.0 & 81.0 & 74.8 & 82.4 & 46.0 & 87.9 & 67.9 & \textbf{87.6} & \textbf{11.7} & 47.1 & 70.1 & 66.0 \\
\addlinespace[2pt]
\textbf{CASE (8B, ours)} & 85.9 & \textbf{94.7} & 92.9 & 89.7 & 68.5 & 94.4 & \textbf{71.4} & 73.4 & 11.1 & \textbf{62.9} & \textbf{76.0} & \textbf{74.5} \\
\bottomrule
\end{tabular}}
\caption{Per-role clinical accuracy: the full breakdown of the role-averaged \cref{tab:main_accuracy}. Within each panel all methods share prompts, patient evidence, execution limits, and parsing rules; models follow the base, medical, and general grouping of \cref{tab:main_accuracy}, followed by CASE. Columns are the ten task--format strata (B, S, M) under six clinical domains, with Micro and Macro as in \cref{eq:metrics}. Missing or failed episodes count as incorrect.}
\label{tab:perrole_accuracy}
\end{table}

\paragraph{Paired significance of the CASE advantage.}
\Cref{tab:main_accuracy} reports CASE ahead of every baseline on a single fixed test sample. To verify that this margin is a robust effect rather than an artifact of which participants were drawn, we recompute each gap with a paired participant-cluster bootstrap (10{,}000 replicates per role, resampling the same participants and questions for the full policy and the comparator) and report, per role, the point estimate with its 95\% percentile interval in \cref{tab:paired_comparisons}. Its three columns are computed within a role from first-pass exact answer-set accuracy over that role's scheduled questions (\cref{eq:metrics}): \emph{Micro} is a policy's own such accuracy, and writing $A_f$ for CASE-full's Micro and $A_b$ for a comparator's, \emph{Full $-$ method} is the gap $A_f-A_b$ in percentage points while \emph{Relative gain} is that gap as a percentage of the comparator, $100\,(A_f/A_b-1)$ percent. An interval that excludes zero indicates the advantage survives test-sample resampling at the given checkpoint.

\begin{table}[htbp]
\centering\footnotesize
\setlength{\tabcolsep}{4pt}
\renewcommand{\arraystretch}{1.0}
\begin{tabular}{@{}lccc@{}}
\toprule
Method & \shortstack{Micro (\%)\\$\pm$95\% CI} & \shortstack{Full $-$ method (pp)\\$\pm$95\% CI} & \shortstack{Relative gain (\%)\\$\pm$95\% CI} \\
\midrule
\multicolumn{4}{@{}l}{\textit{(a) Anatomist Agent}} \\
\textbf{CASE (full)} & 77.5 $\pm$ 4.3 & --- & --- \\
CASE (RSFT only) & 64.9 $\pm$ 3.9 & 12.6 $\pm$ 1.8 & 19.4 $\pm$ 3.8 \\
Qwen3-VL-8B-Instruct & 39.5 $\pm$ 3.4 & 38.0 $\pm$ 2.3 & 96.2 $\pm$ 13.5 \\
\addlinespace[2pt]
Qwen3.8-Max & 65.8 $\pm$ 4.0 & 11.7 $\pm$ 2.1 & 17.8 $\pm$ 3.8 \\
GPT-5.4-0305 & 63.6 $\pm$ 4.0 & 13.9 $\pm$ 2.0 & 21.9 $\pm$ 3.7 \\
Claude Opus 4.8 & 67.4 $\pm$ 3.7 & 10.1 $\pm$ 1.5 & 15.0 $\pm$ 2.9 \\
\addlinespace[2pt]
\multicolumn{4}{@{}l}{\textit{(b) Radiologist Agent}} \\
\textbf{CASE (full)} & 74.8 $\pm$ 3.2 & --- & --- \\
CASE (RSFT only) & 67.9 $\pm$ 4.3 & 6.9 $\pm$ 2.5 & 10.2 $\pm$ 2.0 \\
Qwen3-VL-8B-Instruct & 46.0 $\pm$ 4.1 & 28.8 $\pm$ 1.9 & 62.6 $\pm$ 7.2 \\
\addlinespace[2pt]
Qwen3.8-Max & 40.2 $\pm$ 3.3 & 34.6 $\pm$ 1.5 & 86.1 $\pm$ 11.7 \\
GPT-5.4-0305 & 65.4 $\pm$ 3.0 & 9.4 $\pm$ 1.6 & 14.4 $\pm$ 3.0 \\
Claude Opus 4.8 & 68.3 $\pm$ 4.4 & 6.5 $\pm$ 1.6 & 9.5 $\pm$ 2.1 \\
\addlinespace[2pt]
\multicolumn{4}{@{}l}{\textit{(c) Consultant Agent}} \\
\textbf{CASE (full)} & 76.0 $\pm$ 4.1 & --- & --- \\
CASE (RSFT only) & 64.4 $\pm$ 3.2 & 11.6 $\pm$ 1.6 & 18.0 $\pm$ 3.4 \\
Qwen3-VL-8B-Instruct & 52.5 $\pm$ 3.6 & 23.5 $\pm$ 2.4 & 44.8 $\pm$ 5.1 \\
\addlinespace[2pt]
Qwen3.8-Max & 73.1 $\pm$ 3.5 & 2.9 $\pm$ 1.5 & 4.0 $\pm$ 1.4 \\
GPT-5.4-0305 & 66.9 $\pm$ 3.5 & 9.1 $\pm$ 1.7 & 13.6 $\pm$ 2.6 \\
Claude Opus 4.8 & 70.1 $\pm$ 3.3 & 5.9 $\pm$ 2.1 & 8.4 $\pm$ 1.9 \\
\bottomrule
\end{tabular}
\caption{Role-specific paired gains of \textbf{CASE (full)} over each comparator: Micro accuracy, the difference in percentage points (pp), and the relative gain, each computed from the corresponding role manifest. Every cell is the point estimate $\pm$ the half-width of the 95\% paired participant-cluster bootstrap percentile interval (differences and relative gains recomputed within replicates from the same panel manifest); ``---'' marks the reference policy, and the pp and relative-gain columns are CASE-full minus the comparator, so positive values favor CASE.}
\label{tab:paired_comparisons}
\end{table}

\paragraph{Training ablations.}
The four RL configurations in \cref{tab:learning_ablation} enable neither addition, rubric feedback only, OPSD only, or both; RSFT provides the pre-RL reference. Every RL row starts from the same RSFT checkpoint and retains correctness, format reward, and the frozen-reference regularizer. The rubric condition adds both the continuous judge score and top-2 correctness linkage. Let $A_{uv}$ be the accuracy of the $2\times2$ cell that enables rubric feedback for $u\in\{0,1\}$ and OPSD for $v\in\{0,1\}$,(e.g., $A_{00}$, $A_{10}$, $A_{01}$, and $A_{11}$ denote the accuracies of GRPO, GRPO+RLAIF, GRPO+OPSD, and GRPO+RLAIF+OPSD (CASE-8B), respectively.) so $A_{00}$ adds neither component, $A_{10}$ and $A_{01}$ add one each, and $A_{11}$ adds both. The interaction contrast $A_{11}-A_{10}-A_{01}+A_{00}=(A_{11}-A_{01})-(A_{10}-A_{00})$ compares the accuracy gain from rubric feedback when OPSD is already present against the gain when OPSD is absent. It vanishes if the two components contribute additively and is nonzero only when they interact, so it measures their synergy in accuracy units without presuming that their benefits simply add. Training participants, rollout count, update schedule, full-parameter optimization, and tool budgets remain matched; trajectories remain policy-selected. Three independent seeds are the target, with the completed count and compute reported explicitly.

For a reported metric $M$, the main-table average is
$\overline M_{\mathrm{role}}=\frac13(M_{\mathrm{anat}}+M_{\mathrm{rad}}+M_{\mathrm{cons}})$.
Each $M_r$ is computed first within role $r$; trajectories are never pooled across roles. A variant missing any role is not assigned a three-role mean. Accuracy uses every scheduled question. Insight quality uses a common judge-scored question subset within each role, as defined in \appref{app:judge_evaluation}; \suppref{tab:role_ablation} reports both individual and shared judge coverage. Coverage is a diagnostic, not a quality score: it falls when episodes are unparseable or the judge cannot score them. We therefore expect the rubric-only variant (GRPO+RLAIF, without OPSD) to show the lowest coverage, because the rubric rewards elaboration while nothing brakes the mid-to-late-RL repetitive restatement that pushes Phase-4 synthesis past the length limit, truncating the insight and sometimes the answer; OPSD keeps synthesis concise, so GRPO+OPSD and CASE-8B retain high coverage. Scoring insight quality on the shared subset removes the survivorship bias that unequal coverage would otherwise introduce.

\begin{table}[htbp]
\centering\footnotesize
\setlength{\tabcolsep}{4pt}
\renewcommand{\arraystretch}{1.0}
\begin{tabular}{@{}lccc@{}}
\toprule
Training variant & \shortstack{Micro (\%)\\$\uparrow$} & \shortstack{Insight\\quality $\uparrow$} & \shortstack{Judge coverage\\(\%)} \\
\midrule
\multicolumn{4}{@{}l}{\textit{(a) Anatomist Agent}\quad Shared judge coverage: 84.1\%} \\
RSFT only & 64.9 & 0.45 & 94.2 \\
GRPO (correctness + format) & 73.0 & 0.48 & 90.6 \\
GRPO + RLAIF & 74.3 & 0.65 & 85.3 \\
GRPO + OPSD & 77.1 & 0.59 & 93.1 \\
CASE-8B & 77.5 & 0.70 & 95.0 \\
\addlinespace[2pt]
\multicolumn{4}{@{}l}{\textit{(b) Radiologist Agent}\quad Shared judge coverage: 83.5\%} \\
RSFT only & 67.9 & 0.43 & 93.5 \\
GRPO (correctness + format) & 69.0 & 0.47 & 89.8 \\
GRPO + RLAIF & 70.3 & 0.63 & 84.6 \\
GRPO + OPSD & 74.8 & 0.57 & 92.4 \\
CASE-8B & 74.8 & 0.68 & 94.6 \\
\addlinespace[2pt]
\multicolumn{4}{@{}l}{\textit{(c) Consultant Agent}\quad Shared judge coverage: 79.2\%} \\
RSFT only & 64.4 & 0.41 & 92.8 \\
GRPO (correctness + format) & 65.9 & 0.46 & 88.1 \\
GRPO + RLAIF & 67.2 & 0.61 & 80.4 \\
GRPO + OPSD & 69.4 & 0.55 & 91.7 \\
CASE-8B & 76.0 & 0.66 & 94.2 \\
\addlinespace[2pt]
CASE-8B + skill memory & 77.9 & 0.71 & 94.6 \\
\bottomrule
\end{tabular}
\caption{Role-specific ablations underlying \cref{tab:learning_ablation}. Micro uses the complete role-specific manifest. Insight quality is the fixed judge's score in $[0,1]$ of the logical chain from the Phase-4 key insight to the Phase-5 answer, averaged over the questions judged for all variants (the shared subset). Judge coverage is each variant's judge-scored fraction and shared coverage is their intersection. The last Consultant row adds inference-time skill-memory retrieval (\appref{app:environment}), a Consultant-only augmentation rather than a training variant.}
\label{tab:role_ablation}
\end{table}

Further controls can compare an SDPO-style gradient-bearing distillation loss~\citep{hubotter2026sdpo} with the detached reward used here, and turn- or token-level feedback from PBSD and ADRS~\citep{tian2026pbsd,zhang2026adrs}. Any multimodal adaptation must preserve the defining update rule. These are supplementary comparisons, not results implied by the blank tables.

Teacher-conditioning diagnostics remove the reference rationale, use a length-matched neutral hint, or shuffle the answer for scoring only. A teacher-evolution control freezes only the privileged branch while the ordinary branch tracks the actor; freezing both would confound teacher evolution with stale student scores. Exact assistant-span alignment and full-vocabulary scoring on a fixed subset test the archived span estimator and top-$128$ approximation.

\paragraph{Evidence controls.}
\Suppref{tab:evidence_controls} separates answer priors, supplied context, and autonomous acquisition. Question-only retains the question and options but no patient evidence. Clinical-context-only adds the initial permitted narrative without tools. A fixed packet is selected by a prespecified rule without test labels or model-specific retrieval and respects the same context ceiling. Adaptive acquisition uses the full environment. Removing a source also removes its derivatives: the earlier-visit-only condition, for example, excludes follow-up findings from reports, measurements, and narratives as well as images. The same test manifest is retained, with accuracy drops interpreted as dependence on that information, not causal identification of a particular finding.

\begin{table}[htbp]
\centering\footnotesize
\setlength{\tabcolsep}{4pt}
\renewcommand{\arraystretch}{1.0}
\begin{tabular}{@{}lccccc@{}}
\toprule
Evidence condition & \shortstack{RSFT\\micro} & \shortstack{Full\\micro} & \shortstack{Full\\macro} & \shortstack{Full drop\\(pp)} & \shortstack{Full\\calls} \\
\midrule
Question only & 35.8 & 38.6 & 36.9 & 37.5 & 0 \\
Initial clinical context & 48.4 & 52.7 & 50.9 & 23.4 & 0 \\
Fixed evidence packet & 62.2 & 66.0 & 64.4 & 10.1 & 0 \\
Adaptive acquisition, all sources & 65.7 & 76.1 & 75.1 & --- & 5.3 \\
Adaptive, without MRI views & 60.9 & 65.9 & 65.1 & 10.2 & 4.7 \\
Adaptive, earlier visit only & 59.8 & 64.6 & 62.9 & 11.5 & 3.9 \\
Adaptive, without reports & 62.6 & 71.7 & 70.4 & 4.4 & 3.7 \\
\bottomrule
\end{tabular}
\caption{Evidence sensitivity under a fixed test manifest. Accuracies are percentages; drop is full-policy micro accuracy with all sources minus that in the stated condition, in percentage points. Zero calls is an imposed no-tool condition, not a measured result. Packet construction is accounted for separately from online calls.}
\label{tab:evidence_controls}
\end{table}

MRI-view removal preserves nonvisual measurements and reports, testing direct visual access rather than removal of all imaging-derived information. Report removal deletes report-derived content from the initial prompt and other adapters as well as disabling the reporting tool. For visually answerable subsets, report an independently specified clinical answerability criterion; a history-only target need not be sensitive to images. Synthetic evidence substitutions require revalidated targets rather than automatic reuse of the original diagnosis. Participant-disjoint performance and source-removal tests support in-cohort generalization and evidence dependence, not unrestricted clinical transfer.

\paragraph{Controlled insight evaluation.}
Fix the external judge's model version, evaluation prompt, rubric weights, decoding parameters, extraction rule, and input limits before comparing policies. Each anonymized first-pass insight is scored independently against the same question, options, and reference answer, together with that policy's prediction. Randomize evaluation order and retain both correct and incorrect answers. The four rubric criteria in \suppref{tab:appendix_rubric} remain fixed; neither the training top-2 gate nor candidate-relative ranking enters evaluation. The external judge is distinct from the evolving policy used for OPSD scoring.

For variant $m$ and role $r$, let $\mathcal D_r$ be the scheduled question set and $\mathcal V_{m,r}\subseteq\mathcal D_r$ those with a parseable first-pass answer and a successfully judged insight. Define the shared subset $\mathcal E_r=\bigcap_m\mathcal V_{m,r}$. With individual judge scores $s^{\mathrm{eval}}_{m,r,i}\in[0,1]$, report
\begin{equation}
 S_{m,r}=\frac{1}{|\mathcal E_r|}\sum_{i\in\mathcal E_r}s^{\mathrm{eval}}_{m,r,i},
 \qquad
 \overline S_m=\frac13\sum_{r\in\{\mathrm{anat},\mathrm{rad},\mathrm{cons}\}}S_{m,r}.
 \label{eq:judge_evaluation}
\end{equation}
Thus the main-table insight score is an arithmetic mean within each role followed by an equal mean across roles, evaluated on matched questions. Individual coverage is $|\mathcal V_{m,r}|/|\mathcal D_r|$ and shared coverage is $|\mathcal E_r|/|\mathcal D_r|$; an empty shared subset has no reported mean. This conditional quality score does not replace accuracy on the complete manifest. Report truncation and judge-failure rates, apply the same retry policy to every variant, and do not substitute a neutral score for a failed request. Participant-cluster bootstrap intervals use the shared subset and paired resampling.

A fixed judge controls the evaluation procedure, not all evaluator bias. Reusing the training judge also leaves possible adaptation to that evaluator. Higher scores indicate rubric-assessed argument quality, not independently verified image grounding. Blinded expert review or a judge not used in training should check score calibration, unsupported certainty, and explanations of incorrect answers on a prespecified subset.

\paragraph{Distributional and repetition diagnostics.}
Causal-KL is the mean of \cref{eq:bucket} over successfully scored trajectories with valid insight spans. Each checkpoint supplies both distributions, so this diagnostic has no fixed external reference and is not a cross-model clinical-quality metric. A decrease can also reflect reduced sensitivity to the hint rather than improved reasoning. Scoring failures and missing spans are reported as coverage, never recoded as zero divergence. Hold the span rule, scoring temperature, and privilege fields fixed, and add privilege only for offline scoring after the first-pass answer is locked. In addition to policy-generated histories, score a fixed held-out history bank with identical tokens, images, and privileges across checkpoints to separate changes in scoring behavior from changes in acquired evidence. Swapping RSFT- and RL-collected evidence prefixes before generating a fresh synthesis further distinguishes acquisition quality from synthesis quality. These analyses do not replace end-to-end accuracy.

For repetition, we use an 8-gram instance of the repeated-$n$-gram degeneration diagnostic of \citet{welleck-etal-2020-consistency}: tokenize Phase-4 text by lowercased whitespace-separated words. With $L_i\geq8$ words and $U_{8,i}$ distinct contiguous eight-grams, define
\begin{equation}
 \operatorname{Rep}_{8,i}=1-\frac{U_{8,i}}{L_i-7}.
 \label{eq:repetition_metric}
\end{equation}
Report its trajectory mean as a percentage with eligible coverage and median insight length; shorter or missing insights are not assigned zero repetition. Only assistant Phase-4 prose is scored; metadata and tool observations are excluded. Lexical repetition is a diagnostic, so gains require preserved accuracy and valid-answer coverage, not merely fewer generated words. A reduction in discourse-marker repetition remains a mechanistic hypothesis unless supported by phrase-level analysis.

\paragraph{Interaction cost and failure accounting.}
Count tool requests, executed calls, tool errors, duplicate calls, and budget-triggered stops separately. A duplicate repeats the tool name and canonical arguments within one trajectory. \Suppref{tab:interaction_cost} reports per-scheduled-question resource means, including work incurred by failed runs, and median/95th-percentile end-to-end latency. A complete execution ledger supplies the resource denominator; missing telemetry requires an explicit coverage rate, not zero imputation. Image blocks count returned images, not slices displayed within a montage; visual positions count the processor's actual merged embeddings. Repeatedly encoding an image is charged each time it is processed. Generated assistant tokens cover the entire first-pass interaction, not only the final answer. Vendor-inaccessible token or visual counts are marked not reported rather than estimated from image count.

\begin{table}[htbp]
\centering\footnotesize
\setlength{\tabcolsep}{4pt}
\renewcommand{\arraystretch}{1.0}
\begin{tabular}{@{}lcccccc@{}}
\toprule
Policy & Calls & \shortstack{Image\\blocks} & \shortstack{Visual\\positions} & \shortstack{Assistant\\tokens} & \shortstack{Latency (s)\\median / p95} & \shortstack{Failure\\(\%)} \\
\midrule
\textbf{CASE-8B (full)} & 5.1 & 8.4 & 3612 & 1452 & 34.2 / 68.5 & 2.6 \\
CASE-8B (RSFT only) & 4.6 & 7.2 & 3087 & 1248 & 31.5 / 61.2 & 6.8 \\
Qwen3-VL-8B-Instruct & 2.9 & 3.6 & 1514 & 782 & 22.4 / 49.1 & 21.4 \\
\addlinespace[2pt]
Qwen3.8-Max & 5.3 & 8.6 & 3704 & 1603 & 43.6 / 86.2 & 4.9 \\
GPT-5.4-0305 & 4.3 & 7.0 & 3012 & 1324 & 46.1 / 92.4 & 5.7 \\
Claude Opus 4.8 & 4.1 & 6.8 & 2896 & 1281 & 44.3 / 89.0 & 5.2 \\
\bottomrule
\end{tabular}
\caption{First-pass interaction cost per scheduled question on the same test bank as \cref{tab:main_accuracy}. Failure is $100(1-N^{-1}\sum_iV_i)$ (unparseable or tool-error terminations), not the clinical error rate; resource means include work from failed runs and distinguish unavailable instrumentation from zero use. Read with \cref{tab:main_accuracy}: the full policy reaches the best accuracy at the lowest failure rate and lower latency than the frontier APIs, so autonomous acquisition is paid for in tool calls and images rather than brute-force length.}
\label{tab:interaction_cost}
\end{table}

\section{Historical Pilot Measurement}
\label{app:historical}

The archived Consultant pilot reports 3,752 correct answers among 4,850 questions (77.36\%). Of 220 detected extraction failures, 49 were counted correct through fallback to option A. Treating all detected failures as incorrect yields $3703/4850=76.35\%$. These are alternative scoring conventions for the same predictions. The pilot reports initially cold skill memory, no first-pass retrieval matches, and no reuse of newly written skills. Raw predictions, recall logs, and the resolved configuration are unavailable for independent verification; this aggregate is therefore excluded from controlled comparisons and component-effect claims. Test-feedback adaptation rounds remain outside the held-out protocol.

\section{Implementation Details and Reproducibility}
\label{app:implementation}

\paragraph{Visual and protocol alignment.}
The current tool loop reads the singular \artifact{image} field and inserts one placeholder per returned image. Report adapters supply encoded strings, whereas image and anatomy adapters supply PIL objects. End-to-end decoding and placeholder--feature alignment therefore require validation across both formats. Cache-key granularity, role-specific budgets, prompt/schema differences, and reference-insight coverage are documented in \appref{app:environment} and \appref{app:training}; changes to these settings define a new protocol version rather than retroactively correcting an earlier evaluation.

\paragraph{Scoring masks and inserted privilege.}
The \artifact{find_insight_positions} function locates the first \artifact{Phase 4} marker and the following \artifact{Phase 5}, then maps their relative character offsets to estimated nonpadding token indices. If Phase 4 is absent, it selects the middle 30\%--70\% of the response. This approximate scoring mask is not intersected with the assistant-action mask and can include positions outside the intended insight. Exact token-span alignment would change the scoring implementation; scoring and policy-gradient masks are therefore distinct reproducibility artifacts.

The privileged input inserts the answer hint before the complete continuation. Optional reference insights are truncated to 500 characters. Option lookup uses the complete answer string, so a composite multi-select label may yield no option text. The helper rebuilds position IDs from the cumulative attention mask and broadcasts them over multiple position axes; correct image-aware alignment requires runtime validation in addition to matching response-token rows.

\paragraph{Policy objective.}
Over assistant tokens $t\in\mathcal A_j$, with likelihood ratio $\rho_{n,j,t}(\vtheta)=\pi_{\vtheta}(u_{j,t}\mid g_{j,t})/\pi_{\bar{\vtheta}_n}(u_{j,t}\mid g_{j,t})$, clipping radius $\eta=0.2$ (the launcher leaves the verl default), and advantage $A_j$ from \cref{eq:advantage}, the trainer minimizes the token-averaged PPO-style clipped surrogate~\citep{schulman2017ppo} $\min\{\rho_{n,j,t}A_j,\clip(\rho_{n,j,t},1\pm\eta)A_j\}$ plus a $0.05\,\mathcal K_{\mathrm{ref}}$ reference term. This term is a low-variance KL penalty to the frozen SFT policy $\pi_{\mathrm{ref}}=\vtheta_0$ on unprivileged contexts (\cref{fig:rl_step}), distinct from OPSD, whose self-teacher is refreshed and whose score enters the reward rather than the loss.

\paragraph{Reward injection and failure handling.}
The reverse-KL penalty uses top-$K$ truncation ($K{=}128$) with coefficient $2.0$. The worker scores ordinary and privileged inputs in matched chunks with identical multimodal inputs, averages the divergence over the estimated insight mask, and returns one CPU scalar per trajectory. The trainer places $2D_j$ in the final slot of an otherwise zero penalty tensor and subtracts it from token-level rewards before advantage estimation. It neither broadcasts the penalty across Phase-4 tokens nor differentiates through scoring logits. Causal-KL exceptions skip the penalty for the update; image mismatches or out-of-memory failures within a scoring chunk return zero for affected trajectories. Failure counts are needed to distinguish these fallbacks from successful low-divergence scores.

\paragraph{Why $K=128$ buckets suffice.}
Let $p_{\vtheta}(\cdot\mid c_t)$ be the next-token distribution over $\mathcal V$ ($|\mathcal V|=151{,}936$), sorted $p_{(1)}\ge p_{(2)}\ge\cdots$, with top-$k$ coverage $C(k)=\sum_{i\le k}p_{(i)}$. Averaged over $N=768$ positions ($12$ prompts $\times$ $64$ tokens sampled from the base Qwen3-VL-8B-Instruct, recording normalized top-$2000$ log-probabilities offline so $\bar C(k)$ is exact for $k\le2000$), coverage saturates early: $\bar C(64)=0.969$, $\bar C(128)=0.979$, $\bar C(256)=0.986$, $\bar C(2000)=0.996$ (\cref{fig:bucket128}). The top-$128$ set thus holds $\approx98\%$ of the mass, a $16\times$ larger set adds only $\approx1.7$ points, and the remaining ${\approx}151.8$k entries contribute $<0.4\%$ in total, so truncating the divergence there is numerically near-lossless. The saving is in memory at the \emph{divergence} step, not in the forward. Each step scores $16$ prompts $\times$ $G{=}8$ rollouts $=128$ trajectories; for each, the ordinary view $[x;u_j]$, comprising realized tool calls, observations, and chain-of-thought (CoT), and the answer-privileged view $[x;z_j^\star;u_j]$ pass through the same frozen snapshot in a forward-only scoring pass with no sampling. Both forwards necessarily emit full-vocabulary logits, but the causal-KL is then formed only on the Phase-4 span, token by token, over the top-$128$ logits plus a lumped tail (\cref{eq:bucket}). Restricting the comparison to $129$ entries per position means the divergence never retains full-vocabulary log-probability tensors for the two views (only the top-$128$ logits and a per-position log-normalizer are kept), so this dominant KL intermediate shrinks by ${\sim}10^3\times$ (from $|\mathcal V|$ to $129$ per position per view) across the $128$ trajectories. That is what keeps the paired OPSD scoring inside the same $8\times80$GB budget as the policy, optimizer state, and rollouts, with no separate scoring cluster.

\begin{figure}[htbp]
\centering
\includegraphics[width=0.92\linewidth]{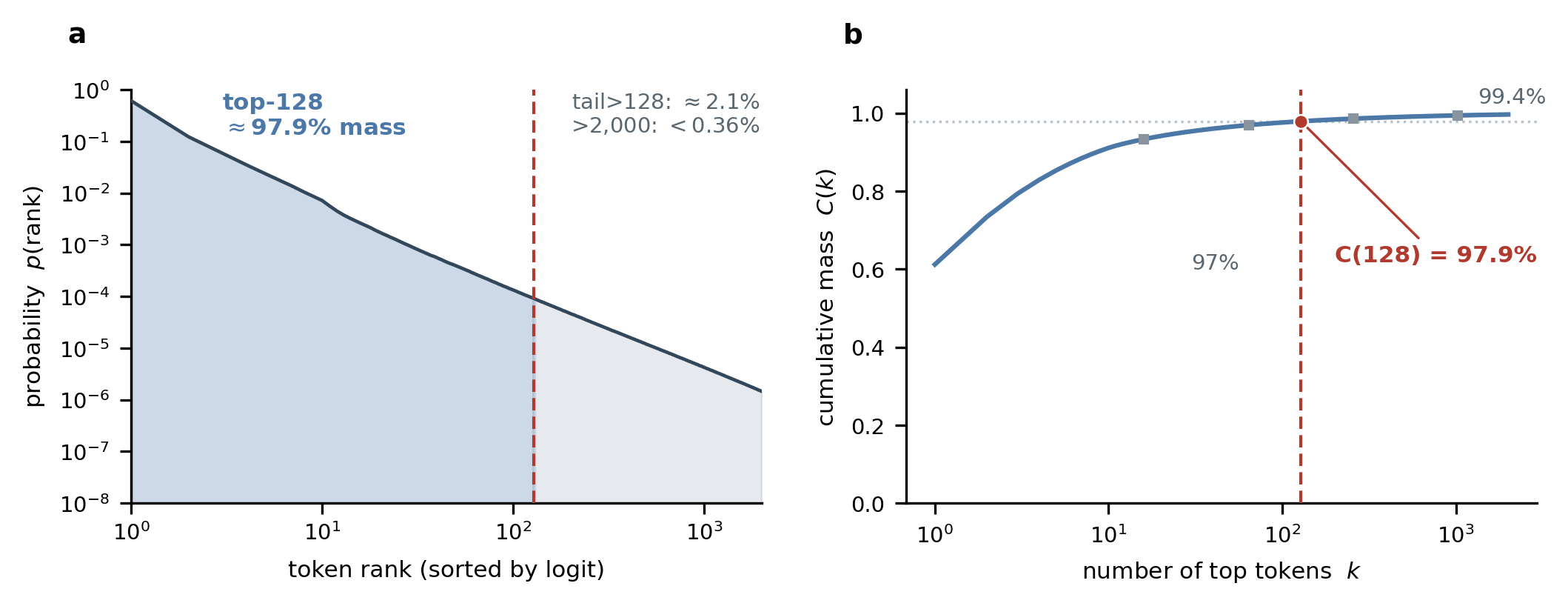}
\caption{\textbf{The top-$128$ logits carry ${\approx}98\%$ of next-token mass in Qwen3-VL-8B-Instruct.} (a) Mean sorted next-token distribution over $768$ rollout positions ($12$ prompts $\times$ $64$ tokens), log--log: ranks $\le128$ (blue) hold $97.9\%$, ranks $129$--$2000$ add $2.1\%$, and ranks $>2000$ add $<0.4\%$. (b) Cumulative coverage $C(k)$ saturates by $k{=}128$ ($C(128){=}97.9\%$; $C(64){=}96.9\%$, $C(1024){=}99.4\%$), so $K{=}128$ buckets suffice for the causal-KL term.}
\label{fig:bucket128}
\end{figure}


\section{Mathematical Properties and Scope of Privileged Self-Consistency}
\label{app:math}

We characterize vocabulary coarsening, detached reward optimization, and conditional bounds on repetitive continuations. At update $n$, both scoring distributions use the fixed snapshot $\bar{\vtheta}_n$:
$q_{n,j,t}(\cdot)=\pi_{\bar{\vtheta}_n}(\cdot\mid g_{j,t})$ and
$p_{n,j,t}(\cdot)=\pi_{\bar{\vtheta}_n}(\cdot\mid g^+_{j,t})$.
The history $g_{j,t}$ is the serialized prefix preceding $u_{j,t}$, including previously acquired multimodal observations; $g^+_{j,t}$ inserts the training privilege after the initial prompt, as in \cref{sec:opsd}. Below, $z_j$ abbreviates $z_j^\star$, $D_j=D^{\mathrm{CKL}}_{n,j}$, $\lambda=\lambda_{\mathrm{CKL}}=2$, and $R_j=R_{n,j}$; we suppress $n$ where unambiguous.
Unless stated otherwise, logarithms are natural and probabilities are evaluated before numerical flooring. None of the results below assumes that a low divergence establishes clinical correctness.

\subsection{Scoring contexts and update sequence}
\label{app:scoring}

Write the tokenization of prompt $x$ as $c_j$, the serialized policy continuation as $u_j$, and the tokenized privilege as $z_j$. The two inputs are $[c_j;u_j]$ and $[c_j;z_j;u_j]$, with identical tool calls, tool-result text, retrieved image inputs, and reasoning. Let $\bm Z^{\mathrm S}_j$ and $\bm Z^{\mathrm T}_j$ be their full-vocabulary logit matrices. Suppressing padding, define prompt length $P_j=|c_j|$ and hint length $H_j=|z_j|$. Using one-based indexing, the logit vectors predicting continuation token $u_{j,t}$ are
\begin{equation}
 \bm\ell^{\mathrm S}_{j,t}=\bm Z^{\mathrm S}_j[P_j+t-1,:],\qquad
 \bm\ell^{\mathrm T}_{j,t}=\bm Z^{\mathrm T}_j[P_j+H_j+t-1,:].
 \label{eq:app-alignment}
\end{equation}
The offset aligns prediction targets, not absolute input columns. Causal attention restricts each row to its preceding history even when the complete sequence is passed to the model at once. The teacher hint is training-only earlier context; future response tokens and future tool images are not legitimate earlier context. Padding, image order, and multimodal position metadata must be handled consistently in the actual model call.

\begin{table}[htbp]
\centering\small
\begin{tabularx}{\textwidth}{@{}lXXX@{}}
\toprule
Forward & Context & Parameters & Role in learning \\
\midrule
Ordinary scoring & Student history & Current snapshot, fixed & Produces $q$; no scoring gradient. \\
Privileged scoring & Same history plus $z_j$ & Same snapshot, fixed & Produces $p$; no scoring gradient. \\
Actor & Ordinary history & Trainable policy & Optimizes sampled assistant actions. \\
Reference & Ordinary history & Frozen SFT policy & Supplies the separate reference penalty. \\
\bottomrule
\end{tabularx}
\caption{Distinct forward roles in the specified algorithm. The external text judge is neither the OPSD teacher nor the reference policy.}
\label{tab:app-forwards}
\end{table}

\begin{algorithm}[ht]
\caption{Agentic RL with batch top-2 feedback and privileged synthesis scoring}
\label{alg:app-update}
\begin{algorithmic}[1]
\STATE Fix scoring and rollout snapshot $\bar{\vtheta}_n$; draw training prompts.
\FOR{each prompt $x$ and its training privilege $z$}
\STATE Sample $G$ interactions without $z$, choosing tools from accumulated evidence.
\STATE Retain each serialized continuation, image inputs, and assistant-action mask.
\STATE Score all eligible insights; form the group-relative gates $b_{1:G}$ using \suppref{eq:judge_gate}.
\FOR{every rollout $j=1,\ldots,G$, irrespective of correctness or judge rank}
\STATE Compute answer, format, judge, and gated linkage rewards $B_j$.
\STATE Score both ordinary and privileged inputs without gradients; align next-token rows.
\STATE Compute $D_j=D^{\mathrm{CKL}}_{n,j}$ on the approximate insight mask using student top-$128$ plus tail.
\STATE Subtract $\lambda_{\mathrm{CKL}}D_j$ in the final reward slot; set $R_j=\sg[B_j-\lambda_{\mathrm{CKL}}D_j]$.
\ENDFOR
\STATE Normalize all $G$ returns jointly to obtain the group-relative advantages.
\ENDFOR
\STATE Update vision encoder, merger, and decoder using the clipped objective and frozen-SFT reference penalty in \cref{sec:reward}.
\STATE Refresh the rollout/scoring snapshot from the updated policy for the next iteration.
\end{algorithmic}
\end{algorithm}

\subsection{Full normalization and the exact coarsening residual}

Fix one scored prefix and suppress $n,j,t$. For effective logit vectors $\bm\ell^{\mathrm S},\bm\ell^{\mathrm T}$, write $\ell^{\mathrm S}(v),\ell^{\mathrm T}(v)$ for the scalar vocabulary entries, including any scoring-temperature scaling. The full-vocabulary distributions are
\begin{equation}
q(v)=\frac{e^{\ell^{\mathrm S}(v)}}{\sum_{w\in\mathcal V}e^{\ell^{\mathrm S}(w)}},
\qquad
p(v)=\frac{e^{\ell^{\mathrm T}(v)}}{\sum_{w\in\mathcal V}e^{\ell^{\mathrm T}(w)}}.
\label{eq:app-full-normalization}
\end{equation}
Let $S=\TopK(q,K)$ with $K<|\mathcal V|$, $U=\mathcal V\setminus S$, and
$q_U=\sum_{v\in U}q(v)$, $p_U=\sum_{v\in U}p(v)$.
The partition $\mathcal P=\{\{v\}:v\in S\}\cup\{U\}$ induces
$\widetilde q=((q(v))_{v\in S},q_U)$ and
$\widetilde p=((p(v))_{v\in S},p_U)$.
Thus
\begin{equation}
d_K(q,p)=\KL(\widetilde q\Vert\widetilde p)
=\sum_{v\in S}q(v)\log\frac{q(v)}{p(v)}
 +q_U\log\frac{q_U}{p_U}.
\label{eq:app-bucket}
\end{equation}
The selected probabilities retain the full normalizers in
\suppref{eq:app-full-normalization}; they are not renormalized over $S$.

\begin{proposition}[Exact residual under vocabulary coarsening]
\label{prop:coarsening}
For strictly positive $q,p$ on the finite vocabulary,
\begin{equation}
\KL(q\Vert p)-d_K(q,p)
=q_U\KL\!\left(q(\cdot\mid U)\Vert p(\cdot\mid U)\right)\geq0.
\label{eq:app-residual}
\end{equation}
Equality holds precisely when the two conditional tail distributions agree.
\end{proposition}
\begin{proof}
For $v\in U$, substitute $q(v)=q_Uq(v\mid U)$ and
$p(v)=p_Up(v\mid U)$ into the tail contribution to $\KL(q\Vert p)$:
\begin{equation*}
\sum_{v\in U}q(v)\log\frac{q(v)}{p(v)}
=q_U\log\frac{q_U}{p_U}
 +q_U\sum_{v\in U}q(v\mid U)
       \log\frac{q(v\mid U)}{p(v\mid U)}.
\end{equation*}
Adding the unchanged terms on $S$ proves the identity; nonnegativity and the equality condition follow from those of KL.
\end{proof}

Let $m^{\mathrm I}_{j,t}=\ind[t\in\mathcal I_j]$ mark the estimated insight region, $M_j=\max\{1,\sum_t m^{\mathrm I}_{j,t}\}$, and
$D_j^{\mathrm{full}}=M_j^{-1}\sum_t m^{\mathrm I}_{j,t}\KL(q_{n,j,t}\Vert p_{n,j,t})$.
Applying \suppref{prop:coarsening} at each prefix gives
\begin{equation}
D_j^{\mathrm{full}}-D_j
=\frac{1}{M_j}\sum_t m^{\mathrm I}_{j,t}q_{U,j,t}
 \KL\!\left(q_{n,j,t}(\cdot\mid U_{j,t})
 \Vert p_{n,j,t}(\cdot\mid U_{j,t})\right).
\label{eq:app-trajectory-residual}
\end{equation}
Consequently, $0\leq D_j\leq D_j^{\mathrm{full}}$, but a small $D_j$ does not imply a small full divergence. Small tail mass alone is insufficient without control of the conditional tail divergence. The implementation floors tail masses at $10^{-8}$; the exact identity concerns the unfloored distributions. Coarsening reduces the explicit comparison to $K+1$ events, while both vocabulary normalizers and model projections remain full-vocabulary operations.

\subsection{Detached rewards and the finite-batch policy gradient}
\label{app:gradient}

For a group of $G$ sampled interactions, let
$R_j=\sg[B_j-\lambda D_j]$,
$\bar R=G^{-1}\sum_jR_j$, and
$A_j=(R_j-\bar R)/(s_R+\epsilon)$, where $s_R$ is the group reward standard deviation. All rewards, advantages, scoring parameters, and sampled histories are held fixed during the actor update. Centering gives the exact identity
\begin{equation}
A_j=\frac{(B_j-\bar B)-\lambda(D_j-\bar D)}{s_R+\epsilon},
\qquad
\bar B=\frac1G\sum_iB_i,\quad \bar D=\frac1G\sum_iD_i.
\label{eq:app-centered-advantage}
\end{equation}
The denominator normalizes the combined reward; the expression is not a sum of independently normalized advantages.

Let $m^{\mathrm A}_{j,t}=\ind[t\in\mathcal A_j]$ mark assistant actions, including tool-call tokens, and let
$w_{j,t}\geq0$ be the fixed weights implementing the actor's reduction, with
$\sum_{j,t}w_{j,t}m^{\mathrm A}_{j,t}=1$.
Define $\rho_{j,t}(\vtheta)=
\pi_{\vtheta}(u_{j,t}\mid g_{j,t})/
\pi_{\bar{\vtheta}_n}(u_{j,t}\mid g_{j,t})$ on these actions.
Writing the reference-policy regularizer as $\mathcal K_{\mathrm{ref}}$ with coefficient $\beta_{\mathrm{ref}}=0.05$, the clipped finite-batch surrogate is
\begin{align}
\widehat{\mathcal J}_n(\vtheta)
={}&\sum_{j,t}w_{j,t}m^{\mathrm A}_{j,t}
\min\!\left\{\rho_{j,t}(\vtheta)A_j,
\clip(\rho_{j,t}(\vtheta),1-\eta,1+\eta)A_j\right\}
\nonumber\\
&-\beta_{\mathrm{ref}}\mathcal K_{\mathrm{ref}}(\vtheta).
\label{eq:app-surrogate}
\end{align}
At $\vtheta=\bar{\vtheta}_n$, the likelihood ratio is one and clipping is locally inactive. Therefore
\begin{equation}
\left.\nabla_{\vtheta}\widehat{\mathcal J}_n\right|_{\bar{\vtheta}_n}
=\sum_{j,t}w_{j,t}m^{\mathrm A}_{j,t}A_j
 \nabla_{\vtheta}\log\pi_{\vtheta}(u_{j,t}\mid g_{j,t})\big|_{\bar{\vtheta}_n}
-\beta_{\mathrm{ref}}\nabla_{\vtheta}\mathcal K_{\mathrm{ref}}(\bar{\vtheta}_n).
\label{eq:app-surrogate-gradient}
\end{equation}
Thus an insight-region score changes the coefficient of every assistant-action gradient in its trajectory, including earlier evidence-acquisition actions. Tool observations and image placeholders are not sampled actions; their representations can nevertheless receive gradients through subsequent assistant predictions. A shared trajectory advantage does not identify which individual tool call caused the reward.

For comparison, directly differentiating reverse KL at a fixed prefix against a fixed teacher $p$ would yield
\begin{equation}
\nabla_{\vtheta}\KL(q_{\vtheta}\Vert p)
=\sum_{v\in\mathcal V}q_{\vtheta}(v)
 \log\frac{q_{\vtheta}(v)}{p(v)}\,
 \nabla_{\vtheta}\log q_{\vtheta}(v).
\label{eq:app-direct-gradient}
\end{equation}
The additional constant from differentiating $q\log q$ cancels because
$\sum_v\nabla_{\vtheta} q_{\vtheta}(v)=0$.
No term of the form \suppref{eq:app-direct-gradient} is introduced by the detached OPSD reward in \suppref{eq:app-surrogate-gradient}. The two procedures differ in their first-order objectives.

If all $D_j$ in a group coincide, subtracting $\lambda D_j$ translates every reward equally: it leaves both $s_R$ and all advantages unchanged. More generally,
$R_i-R_j=(B_i-B_j)-\lambda(D_i-D_j)$.
For equal base rewards, the smaller divergence receives the larger reward, but this ranking is not a guarantee that subsequent conditional divergences decrease.
\Suppref{eq:app-surrogate-gradient} is a derivative of the stated finite-batch surrogate, not a claim of an unbiased trajectory-policy-gradient estimator when rollout uses nucleus truncation. Detaching and periodically refreshing the score does not establish convergence to a fixed objective.

\subsection{A conditional interpretation of repetition}
\label{app:repetition}

At a fixed prefix, let $E\subseteq\mathcal V$ contain next tokens that extend a specified repetition pattern. This event may depend on the prefix. Write
$\operatorname{kl}(a\Vert b)=a\log(a/b)+(1-a)\log((1-a)/(1-b))$ for binary KL.

\begin{proposition}[Control of a specified continuation event]
\label{prop:repetition-event}
For any $E$,
\begin{equation}
\operatorname{kl}(q(E)\Vert p(E))\leq\KL(q\Vert p),
\qquad
q(E)\leq p(E)+\sqrt{\KL(q\Vert p)/2}.
\label{eq:app-repeat-full}
\end{equation}
If $E$ is a union of cells of $\mathcal P$, both occurrences of the full KL may instead be replaced by $d_K(q,p)$.
\end{proposition}
\begin{proof}
Applying the log-sum inequality separately to $E$ and its complement gives the first bound. Binary Pinsker gives
$|q(E)-p(E)|\leq\sqrt{\operatorname{kl}(q(E)\Vert p(E))/2}$.
If $E$ is a union of partition cells, the same argument applies to
$\widetilde q,\widetilde p$, whose masses on $E$ equal those of $q,p$.
\end{proof}

The result is conditional: if privilege reduces the teacher's repetition-event mass to at most $\alpha$, and the relevant divergence is at most $d$, then the student's mass is at most $\alpha+\sqrt{d/2}$. Neither premise follows merely from supplying an answer. If $E$ includes only part of the tail, bucket KL need not control it; probability can move within the tail without changing the score. Likewise, an average over the approximate mask $m^{\mathrm I}$ does not imply a uniform bound at every synthesis position. These statements do not guarantee fewer repeated sequences, lower entropy, or improved accuracy after an RL update. The proposed explanation---that answer privilege can resolve uncertainty underlying repetitive continuation---therefore requires matched empirical tests.

\begin{remark}[Mechanism of the repetition contrast]
\label{rem:repetition-mechanism}
Consider a Phase-4 prefix that has already stated the decisive finding and the answer, then continues to restate it. At the next position, the ordinary view $q$ retains mass on connectives that prolong restatement, whereas the answer-privileged view $p$---conditioned on a resolved conclusion---shifts mass toward the answer label and end-of-sequence. The event $E$ of \cref{prop:repetition-event} is then precisely this repetitive-continuation mass, with $q(E)\gg p(E)$, so the reverse divergence $\KL(q\Vert p)$ is large exactly where the policy spins without committing. Because the detached penalty of \cref{eq:reward} makes a lower $D^{\mathrm{CKL}}_{n,j}$ a higher return, the group-relative gradient favors rollouts whose ordinary view already resembles the privileged one, reducing $q(E)$ over training without prescribing the retained content. Privilege does not supply the conclusion; it makes an already-reachable conclusion the locally preferred continuation, which is the premise $p(E)\le\alpha$ that \cref{prop:repetition-event} requires. We verify this behavior through the Rep-8 diagnostic of \cref{eq:repetition_metric} and the late-training completion behavior of the ablation arms in \cref{tab:learning_ablation}.
\end{remark}

\subsection{Support, degeneracy, and the evolving self-teacher}

Reverse KL does not permit unrestricted escape from teacher support: if
$q(v)>0$ and $p(v)=0$, the full $\KL(q\Vert p)$ is infinite. With finite softmax logits all probabilities are positive, but assigning student mass to a low-probability teacher event still incurs a cost. Coarsening constrains selected events and total tail mass, while allowing discrepancies inside the tail. Its flexibility concerns alternative supported continuations, not guaranteed superiority to the teacher.

Low divergence can also be uninformative. A model that ignores the privilege can satisfy $q=p$ while producing an incorrect or unsupported explanation; bucket equality additionally allows unequal full distributions within the tail. Answer correctness and rubric feedback supply complementary signals, but their combination does not eliminate all such failures or certify image grounding. Refreshing both conditional distributions with $\bar{\vtheta}_n$ makes the self-teacher evolve with the policy. It neither establishes that privilege makes that teacher more accurate at every update nor implies monotonic improvement, a capacity increase, or eventual performance beyond the demonstration policy. Moreover, the two prompted conditionals are not assumed to be compatible marginals of one Bayesian model, so the score is not identified with conditional mutual information. Broader conditioning and generalization pitfalls in on-policy distillation are discussed by \citet{zhu2026manyfaces}.

\end{document}